%% file: neurips_2026.tex
\documentclass{article}

\PassOptionsToPackage{numbers, compress}{natbib}

\usepackage[preprint]{neurips_2026}
\usepackage{amsmath}
\usepackage{amsthm}
\usepackage{tikz}
\usepackage[ruled,vlined]{algorithm2e}
\usepackage{wrapfig}
\usepackage{thmtools}
\usepackage{booktabs}
\usepackage[table]{xcolor}
\usepackage{array}
\usepackage{graphicx}

\definecolor{rowgray}{RGB}{245,245,245}

\newtheorem{theorem}{Theorem}
\newtheorem*{theorem*}{Theorem}

\newtheorem*{proposition*}{Proposition}

\newtheorem*{corollary*}{Corollary}

\newtheorem{definition}[theorem]{Definition}
\newtheorem{remark}[theorem]{Remark}

\newcommand{\R}{\mathbb{R}}

\newcommand{\lam}{\lambda}
\newcommand{\Fhat}{\hat{\mathcal F}}
\newcommand{\ind}{\mathbf{1}}
\newcommand{\res}[1]{\ifcsname res@#1\endcsname\csname res@#1\endcsname\else\textbf{??}\PackageWarning{res}{Undefined result #1}\fi}
\InputIfFileExists{tables/numbers.tex}{}{}
\InputIfFileExists{tables/numbers_analytic.tex}{}{}
\InputIfFileExists{tables/numbers_baseline.tex}{}{}
\InputIfFileExists{tables/numbers_matched_cost.tex}{}{}

\usepackage[utf8]{inputenc} 
\usepackage[T1]{fontenc}    
\usepackage{hyperref}       
\usepackage{cleveref}
\crefname{section}{Sec.}{Secs.}
\Crefname{section}{Sec.}{Secs.}
\crefname{subsection}{Sec.}{Secs.}
\Crefname{subsection}{Sec.}{Secs.}
\crefname{figure}{Fig.}{Figs.}
\Crefname{figure}{Fig.}{Figs.}
\crefname{table}{Tab.}{Tabs.}
\Crefname{table}{Tab.}{Tabs.}
\crefname{algorithm}{Alg.}{Algs.}
\Crefname{algorithm}{Alg.}{Algs.}
\crefname{theorem}{Thm.}{Thms.}
\Crefname{theorem}{Thm.}{Thms.}
\crefname{definition}{Def.}{Defs.}
\Crefname{definition}{Def.}{Defs.}
\crefname{lemma}{Lem.}{Lems.}
\Crefname{lemma}{Lem.}{Lems.}
\crefname{equation}{Eq.}{Eqs.}
\Crefname{equation}{Eq.}{Eqs.}
\crefname{appendix}{App.}{Apps.}
\Crefname{appendix}{App.}{Apps.}
\usepackage{url}            
\usepackage{booktabs}       
\usepackage{multirow}       
\usepackage{makecell}       
\usepackage{graphicx}       
\usepackage{amsfonts}       
\usepackage{nicefrac}       
\usepackage{microtype}      
\usepackage{xcolor}         

\title{OperatorSHAP: Fast and Accurate Shapley Value Estimation for Neural Operators}

\author{%
  Joshua Stiller \\
  LMU Munich, MCML\\
  \texttt{joshua.stiller@lmu.de} \\
\And
  Santo M. A. R. Thies \\
  LMU Munich, MCML\\
  \texttt{santo.thies@ifi.lmu.de} \\
\AND
  Felix Czaja \\
  LMU Munich\\
  \texttt{felix.czaja@ifi.lmu.de} \\
\And 
  Eyke Hüllermeier \\
  LMU Munich, MCML, DFKI\\
  \texttt{eyke@lmu.de} \\
}
\begin{document}

\maketitle
\begin{abstract}
Understanding model predictions is essential for physical applications, where outputs often inform safety-critical decisions, such as structural load assessment, weather warnings, and clinical diagnosis. Shapley values satisfy many desirable properties as an attribution method, but their computational cost during inference hinders their practical use. Current amortized explainers, such as FastSHAP, are limited to homogeneous inputs, which is problematic for physical applications where data often comes from irregular grids and geometries. We introduce OperatorSHAP, a grid-agnostic attribution method that allows us to train FastSHAP-like explainers for neural operators. We establish a theoretical framework for attributions in function space, connecting to Aumann–Shapley values. We further show that OperatorSHAP's explanations are consistent with state-of-the-art discrete Shapley values across resolutions and transfer across grid sizes without retraining.
\end{abstract}
\section{Introduction}
\label{sec:intro}
Neural operators~\cite{li_fourier_2021,kovachkiNeuralOperatorLearning2023} learn maps between function spaces and are grid-agnostic by design. The same trained model accepts a function sampled on 128 points or on 1,024 points. This flexibility is essential for physical applications, where datasets often contain samples with heterogeneous grids, coming from irregular geometries or sensor placements (see \cref{fig:table_and_mgn_cylinder}).

Fast and reliable explanations of their predictions are important for widespread industrial adoption, where users need to assess the credibility of high-stakes outputs such as structural load estimations or weather warnings. High explanation speeds are particularly relevant in commercial settings, where models are continuously monitored or a large number of users should be supplied with feature attributions. Amortized Shapley explainers, such as FastSHAP~\cite{jethani_fastshap_2022}, are among the most established and efficient explanation tools in the field of explainable AI. However, they are restricted to homogeneous inputs, which limits their applicability. In addition, it is not obvious how to define amortized Shapley values on heterogeneous meshes, as nodes on test meshes may not be present in any input sample.

Aumann--Shapley values~\cite{aumannValuesNonAtomicGames2015} provide a theoretical extension of Shapley values to games with infinite players. In this setting, values are defined similarly to densities over the input domain, where the attribution of a coalition can be calculated by integrating this density over the spatial region corresponding to that coalition. Showing that a game has  a well-defined Aumann--Shapley value is a non-trivial task, especially in the function space setting of neural operators. In this paper, we investigate these theoretical foundations and show how to design grid-agnostic explainers for neural operators. Our contributions are as follows:
\begin{enumerate}
    \item \textbf{Theoretical foundations (\cref{sec:theory}).}
    We define function-valued \emph{operator games} suitable for the Aumann--Shapley framework and show under which conditions the game and its values are well-defined for current neural operator architectures.

    \item \textbf{Amortized grid-agnostic explainer
    (\cref{sec:method}).}
    We provide a FastSHAP-like training procedure for Aumann--Shapley values, allowing us to train  amortized explainers even when the samples are provided on irregular grids. 

    \item \textbf{Empirical validation (\cref{sec:experiments}).}
    We validate our approach on a variety of partial differential equations (PDEs) and show that the resulting explainer produces attributions that are consistent with state-of-the-art discrete Shapley values and transfer across grid sizes without retraining.
\end{enumerate}

\section{Related work}
\label{sec:related}
\begin{figure}[t]
  \centering
    \begin{minipage}[c]{0.44\textwidth}
    \centering
    \begin{tikzpicture}
      \node[inner sep=0pt] (table) {
        \footnotesize
        \renewcommand{\arraystretch}{1.25}
        \setlength{\tabcolsep}{4pt}
        \begin{tabular}{@{}lcc@{}}
          \toprule
          \rowcolor{rowgray}
          \textbf{Method} & \makecell{\textbf{Inference}\\\textbf{speed}} & \makecell{\textbf{Grid}\\\textbf{flexibility}} \\
          \midrule
          \makecell[l]{Classical SHAP\\(e.g., KernelSHAP)} & Slow & Any grid \\
          FastSHAP & Fast & \makecell{Homogeneous\\grid only} \\
          OperatorSHAP & Fast & Any grid \\
          \bottomrule
        \end{tabular}
      };
    \end{tikzpicture}
  \end{minipage}\hfill
  \begin{minipage}[c]{0.54\textwidth}
    \centering
    \includegraphics[width=\linewidth]{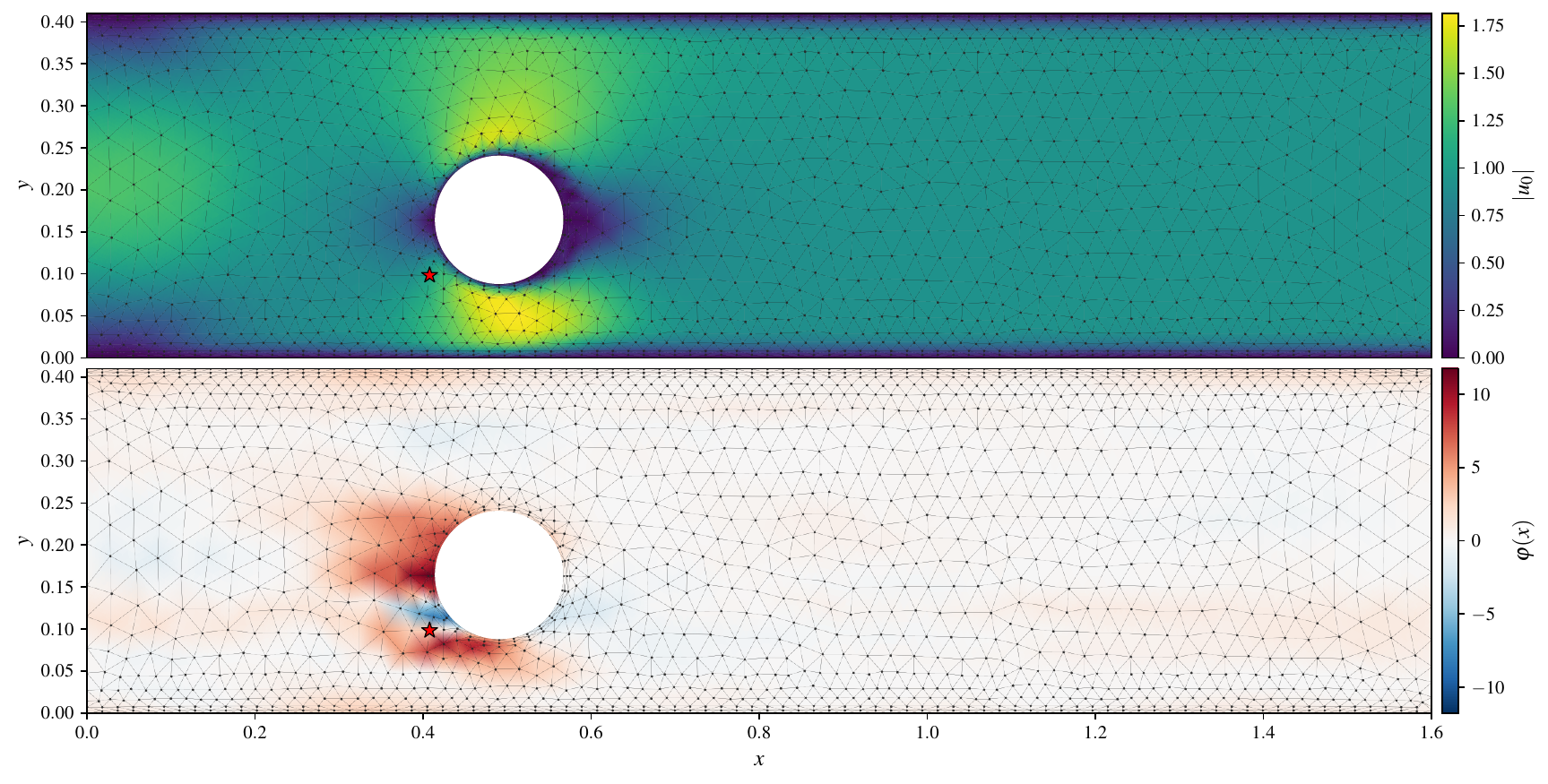}
  \end{minipage}
  \caption{\emph{Left:} Table with an overview of current methods. \emph{Right:} A sample from the MeshGraphNets cylinder-flow dataset~\cite{pfaff_learning_2021}. Top shows a heatmap of airflow velocity strengths, with a grid of sensor points, and the cylinder object obstructing the airflow. The bottom shows the attribution map from OperatorSHAP for the same sample. The explained spatial point is marked with a red star. The varying cylinder placement in this dataset yields heterogeneous meshes, but OperatorSHAP is able to give correct explanations. Here, the area upstream of the cylinder and along its lower side is positively relevant, as the air flows from left to right.}
  \label{fig:table_and_mgn_cylinder}
\end{figure}
\paragraph{Classical Shapley estimation.}
SHAP~\citep{lundberg_unified_2017} introduced Shapley-value explanations to machine learning, while \citet{castro_polynomial_2009} proposed one of the earliest permutation-based approximation methods. Subsequent work has developed model-agnostic estimators based on Monte Carlo sampling \citep{castro_polynomial_2009,wang_data_2023,kolpaczkiSVARM2024}, weighted least-squares objectives \citep{fumagalli_polyshap_2026,fumagalli_odd_2026,muscoLeverageSHAP2025}, and surrogate modeling of the value function \citep{witter_regression-adjusted_2025,butler_proxyspex_2025}. Complementary model-specific approaches exploit structure in trees \citep{lundbergFromLocalToGlobalUnderstanding2020,yu_linear_2022,zernInterventionalSHAP2023}, kernel methods \citep{mohammadi_amortized_2026}, and $k$-nearest-neighbor algorithms \citep{wang_efficient_2024}. Related work extends these ideas from single-feature attributions to interaction~\citep{fujimoto_axiomatic_2006} or group attributions, including Monte Carlo, regression-based, tree-specific, and graph-neural-network-specific estimators \citep{fumagalli_shap-iq_2023,kolpaczki_svarm-iq_2024,fumagalli_kernelshap-iq_2024,muschalikTreeSHAPIQ2024,muschalik_exact_2025}.
OperatorSHAP belongs to this model-agnostic approximation literature, but adapts Shapley-value estimation to explanation problems in dynamical systems. We build on Shapley values, as they are widely used and satisfy desirable axioms, but other methods could give similar explanations (see \citet{fumagalli_unifying_2025} for a recent overview).

\paragraph{Amortized SHAP.} FastSHAP~\cite{jethani_fastshap_2022} introduced the idea of training a single explainer to produce Shapley values in one pass for finite-dimensional inputs. We extend the same loss to function-valued inputs by using a neural operator explainer. This parallels recent work~\citep{InstaSHAP,covert_learning_2023}, which extends FastSHAP's amortized framework to Generalized Additive Models and Vision Transformers, but all of these previous approaches operate on finite-dimensional feature spaces and do not generalize across mesh resolutions or grid structures.

\paragraph{Integrated gradients and Aumann--Shapley.} Integrated gradients (IG)~\cite{sundararajan_axiomatic_2017} and its variants are differential explanations for finite-dimensional models. Their connection to Aumann--Shapley values is well understood in the game-theory literature~\cite{aumannValuesNonAtomicGames2015,neymanChapter56Values2002}. \citet{liuNewBaselineAssumption2023} explored the connection of IG to Aumann--Shapley values in the machine learning context and \citet{delicado_functional_2025} proposed asymptotic Aumann--Shapley values for functional data. IG requires a differentiable model and moves all inputs jointly along a single straight path, so it assigns zero attribution to piecewise-constant outputs such as thresholded predictions~\cite{merrill_generalized_2019}. None of these works investigates the connection in the context of neural operators, grid heterogeneity, or amortized explainers.

\paragraph{Neural Operators.} Neural operators~\cite{kovachkiNeuralOperatorLearning2023} are a class of models that learn solution operators of PDEs. Fourier neural operators (FNOs)~\cite{li_fourier_2021} use Fourier transforms to learn the solution in frequency space, and GINO~\cite{li_geometry-informed_2023-1} is a graph-based architecture that can handle irregular meshes. Further variants include multi-wavelet neural operators~\cite{gupta_multiwavelet-based_2021}, which use wavelet transforms instead of Fourier transforms or the Local FNO (LFNO)~\cite{liu-schiaffini_neural_2024}, which uses Fourier transforms on local patches.

\section{Preliminaries}
\label{sec:prelims}
\subsection{Attributions for finite players}
\paragraph{Cooperative games.}
A \emph{cooperative game} on a finite player set $N=\{1,\dots,n\}$ is a set-function $v:2^N\to\R$ with $v(\emptyset)=0$, where $v(S)$ is the worth that coalition $S\subseteq \{1,\dots,n\}$ produces when its members act jointly~\cite{shapleyValuePersonGames1988}. In the context of machine learning, players are input features to some model, and $v(S)$ measures how much of a model's output is recovered when the features in $S$ are revealed and the rest of the features are set to some baseline (e.g., zero if applicable; various other baselines are used; see \cref{app:baselines}). How to allocate the total worth $v(N)$ back to individual players has a unique answer under mild axioms.

\paragraph{Shapley values.}
That answer, due to Shapley~\cite{shapleyValuePersonGames1988}, averages the marginal
contribution $v(S\cup\{i\})-v(S)$ of player $i$ over all coalitions, weighted
by the number of orderings in which each coalition appears:
\begin{equation}
\label{eq:shapley_def}
\psi_i(v)
\;=\;
\sum_{S\subseteq N\setminus\{i\}}
\frac{|S|!\,(n-|S|-1)!}{n!}\bigl(v(S\cup\{i\})-v(S)\bigr).
\end{equation}
This $\psi_i$ is the unique allocation satisfying \emph{efficiency} ($\sum_i\psi_i(v)=v(N)$), \emph{symmetry} ($v(S\cup\{i\})=v(S\cup\{j\})\;\forall S\subseteq N\setminus\{i,j\}\Rightarrow\psi_i(v)=\psi_j(v)$), the \emph{null-player} property ($v(S\cup\{i\})=v(S)\;\forall S\subseteq N \setminus \{i\}\Rightarrow\psi_i(v)=0$), and \emph{additivity} ($\psi(v+w)=\psi(v)+\psi(w)$).
Due to the computational complexity, Shapley values are usually approximated by methods such as KernelSHAP~\cite{lundberg_unified_2017} or, more recently, regression-adjusted Monte Carlo estimators (RegressionMSR)~\cite{witter_regression-adjusted_2025}.

\paragraph{FastSHAP.}
Whereas Shapley values are calculated for each sample separately, FastSHAP~\cite{jethani_fastshap_2022} amortizes the computational cost of SHAP across inputs by training a single \emph{explainer} $\Phi_\theta:\R^n\to\R^n$ that returns the entire Shapley vector for a given input $x$ in one forward pass, using the weighted MSE objective
\[
\mathcal L(\theta)
\;=\;
\mathbb E_{x,\,S\sim p_w}\!\bigl[\bigl(v(S)-\langle\Phi_\theta(x),\ind_S\rangle\bigr)^2\bigr],
\]
with efficiency $\langle\Phi_\theta(x),\ind_N\rangle=v(N)$ imposed by hard
normalization (Algorithm~1 of~\cite{jethani_fastshap_2022}). Here, $v$ is the value of a game for input $x$, and $p_w$ is a suitable distribution over coalitions.
For finite-dimensional inputs, $\Phi_\theta$ is typically an MLP, with more recent variants replacing it by a generalized additive model~\cite{InstaSHAP} or a vision transformer~\cite{covert_learning_2023}. As our datasets often use irregular grids, none of these architectures can be used as explainers off the shelf.

\subsection{The operator setting}

\paragraph{Spatial PDEs and operator surrogates.}
We seek to explain surrogate models for PDEs (see \cref{app:pde_primer} for a short introduction).
We consider PDE problems on a bounded spatial domain
$\Omega\subset\R^d$ with prescribed boundary and initial conditions, where we denote
$u:\Omega\times[0,\infty)\to\R$ as the state (not necessarily time-dependent).
We consider well-posed problems that define a \emph{solution operator}
$\mathcal F:u_0\mapsto u(\cdot,t^\star)$ that maps an input function (e.g., an
initial condition, a coefficient field, or a forcing) to a solution function
on $\Omega$.
Neural operators~\cite{kovachkiNeuralOperatorLearning2023}, and Fourier neural operators
(FNOs)~\cite{li_fourier_2021} in particular, learn $\mathcal F$ from data on a discrete grid of $\Omega$, but can accept inputs on any discretization of $\Omega$ without retraining. We write $\Fhat$ for such a learned surrogate of $\mathcal F$. We review the FNO architecture in App.~\ref{app:fno_architecture}.

Classical Shapley attribution does not share this property: choosing a grid
fixes the player set, halving the spacing doubles the number of players. Even with a fixed number of grid points, when the grid structure (i.e., the position of each grid point) changes between samples, their attributions are not comparable and an amortized explainer cannot easily learn them, even with a grid-agnostic model. We therefore need a different concept of spatial attributions that is independent of the underlying grid.

\paragraph{The Aumann--Shapley value.}
In PDE problems, we can view every spatial point (e.g., on $[0, 1]$) as a player contributing to the overall outcome of the PDE evolution. Aumann and Shapley derive an infinite-player analog of the Shapley value by working with a non-atomic measure space $(\Omega,\mathcal C,\lam)$, where w.l.o.g.\ $\Omega=[0,1]$ represents the player space, $\mathcal C$ the coalition space, and games are defined as maps $v:\mathcal C\to\R$ with $v(\emptyset)=0$. A coalition $S\in\mathcal C$ is a measurable, possibly disconnected, spatial region and takes the role of a subset of players.\footnote{For readability, we omit some of the technical details, but review them in \cref{app:aumann_shapley}. All of our proofs still use the full mathematical framework.} They define a continuous analog of the value of a game as a linear function that takes a game $v$ as input and outputs a function $\varphi(v): \mathcal C \rightarrow \R$ that satisfies $\varphi(v)(\Omega)=v(\Omega)$, $\varphi(v)(\emptyset) = 0$, and is finitely additive, i.e., $\varphi(v)(S\cup T) = \varphi(v)(S) + \varphi(v)(T)$ for disjoint $S,T\in\mathcal C$.

In contrast to the finite-player setting, they show that only certain classes of games admit a well-defined Aumann--Shapley value. The two classes we will use in this paper are \textbf{pNA} and \textbf{ASYMP} games. If a game $v$ has a finite-dimensional representation $v=f\circ(\mu_1,\dots,\mu_m)$ with $f\in C^1([0,1]^m)$, $f(0)=0$, and $\mu_i$ non-atomic probability measures on $(\Omega,\mathcal C)$, then $v\in$ pNA~\cite[after Def.~1]{neymanChapter56Values2002}.
The value function $\varphi v$ is then given analytically by the \emph{diagonal formula}~\cite{neymanChapter56Values2002}
\begin{equation}\label{eq:diagonal_formula}
(\varphi v)(S)
\;=\;
\int_0^1 \sum_{i=1}^m \mu_i(S)\,\partial_i f(t,\dots,t)\,dt
\qquad(S\in\mathcal C).
\end{equation}
A game $v$ belongs to the broader class \textbf{ASYMP} when its discrete Shapley values along \emph{finite} sub-fields of $\mathcal C$ converge to a common limit, the \emph{asymptotic value} (\cref{app:aumann_shapley}). The inclusion $\mathrm{pNA}\subset \mathrm{ASYMP}$ lets us approximate the continuum value of our operator games by finite-grid Shapley values in \cref{sec:method}. \cref{fig:heat2d_triptych} shows the reverse of this. OperatorSHAP was trained to predict Aumann--Shapley values, and can thus be discretized to predict the Shapley values over finite games of various granularity without retraining. 

Our goal will be to define a game that falls into the pNA class, such that we know it has a value we can learn. We especially want to ensure that common neural operator architectures fall into this class.
\section{Theoretical foundations}
\label{sec:theory}

\begin{figure}[t]
  \centering
  \includegraphics[width=\textwidth]{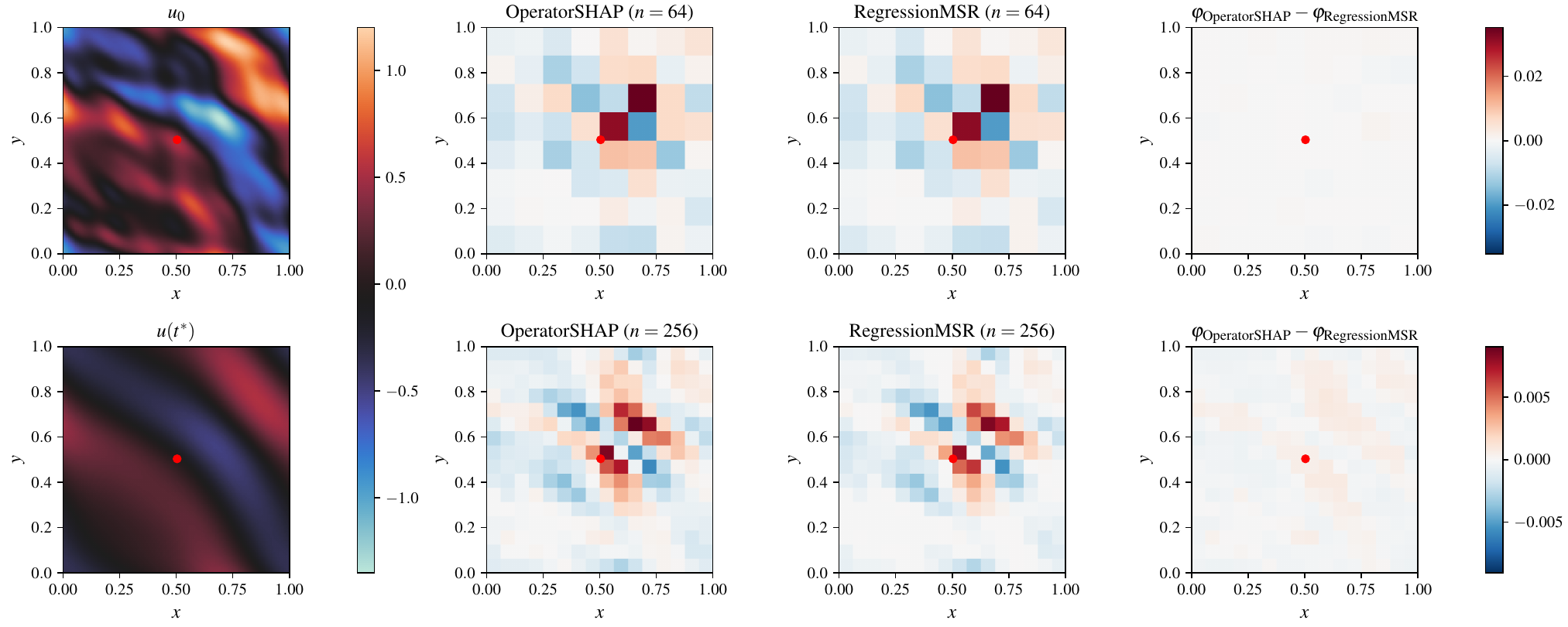}
  \caption{Explanations and input data of the 2D heat equation. The left column shows input (\emph{top}) and label (\emph{bottom}) of the backbone model. The right three columns show the explanation of OperatorSHAP, RegressionMSR, and their difference, respectively. The top row uses an $8\times 8$ grid and the bottom row a $16\times 16$ grid. The same OperatorSHAP model produces both rows, at comparable results to RegressionMSR.}
  \label{fig:heat2d_triptych}
\end{figure}

We begin with a short intuition of the functional-analytic setting for readers unfamiliar with the topic. We work on function spaces, i.e., vector spaces, where each element is a function. Operators are maps between function spaces, i.e., functions that take functions as input and output other functions. We can define derivatives of these operators. The strictest one, the Fréchet derivative, is what is commonly understood as a ``derivative'', namely a linear map that approximates change in a neighborhood of a point. However, Fréchet differentiability is quite restrictive in function spaces and highly dependent on the input space we consider. We will therefore consider a well-behaved input space (the Sobolev space we define). Otherwise, even smooth activation functions, like the GELU, would not be differentiable as an operator in function space. Moreover, common function spaces do not distinguish between two functions that differ on a finite set of points, which makes point evaluation not well-defined. This is problematic for our purposes, as we want to attribute the output at a specific spatial point to the model input. However, in the Sobolev space we consider, every function has a continuous representative which we can use to define point evaluation. We state the results on the flat torus to avoid technicalities regarding boundary conditions. Only \Cref{prop:fno_frechet} depends on this choice, because the Fourier layers of the FNO
act on periodic functions. \Cref{app:general_domains} extends it to bounded smooth domains, and all other results hold there unchanged.

\subsection{Pointwise region attributions}
\label{sec:masked_attr}

\paragraph{Notation.}
Let $\Omega=\mathbb T^d$ be the flat torus and
\[
X := H^s(\Omega;\R^{c_{\mathrm{in}}}),
\qquad
Y := H^s(\Omega;\R^{c_{\mathrm{out}}}),
\qquad
c_{\mathrm{in}},c_{\mathrm{out}}\in\mathbb N, s>d/2, 
\]
where $H^s = W^{s,2}$ is the Sobolev space of functions with $s$ weak derivatives in $L^2$.
We consider a surrogate $\Fhat:X\to Y$, a baseline function $\bar u\in X$, an input function $u_0\in X$, and a query point $x\in\Omega$ (write $h:=u_0-\bar u$). We want to attribute the scalar change $\Fhat(u_0)(x)-\Fhat(\bar u)(x)$ to spatial regions of the input. We will write $D\Fhat(u)[\cdot]: X\to Y$ for the Fréchet derivative of $\Fhat$ (which itself is an operator).

\paragraph{Scalarizing the output.}
Because $s>d/2$, the Sobolev embedding theorem (\cite{runstSobolevSpacesFractional1996}, \S2.2.5) yields
a continuous embedding
\(
H^s(\Omega;\R^{c_{\mathrm{out}}}) \hookrightarrow C^0(\Omega;\R^{c_{\mathrm{out}}})
\) into the space of continuous functions.
In particular, for each $x\in\Omega$, point evaluation
$\operatorname{ev}_x:C^0(\Omega;\R^{c_{\mathrm{out}}})\to\R^{c_{\mathrm{out}}}$, $\operatorname{ev}_x(y)=y(x)$, is a
bounded linear map, and it is natural to study the functional $\Fhat_x := \operatorname{ev}_x
\circ \Fhat$, which evaluates the output of our operator at a point $x$ (we assume a single output channel,
for multiple output channels see \cref{app:multi_channel}). 

We will now define the game for which we want Shapley attributions, as the difference in prediction value of our neural operator $\Fhat$ when it receives a baseline input or an input that was activated in a region $S$ (i.e., that is not baseline in $S$). As the input to $\Fhat$ needs to be smooth, we activate the input in a smooth way via helper functions $\eta_{i}$.

\begin{definition}[Smooth operator games]
\label{def:smooth_op_game}
For $i\in \{1, \dots, m\}$, define the smooth masks $\eta_i:\Omega \rightarrow [0,1]$, such that $\sum_i\eta_i\equiv 1$, $\int_\Omega \eta_i d\lambda > 0$ for all $i \in \{1, \dots, m\}$. Let $\Fhat_x: X\rightarrow \R$ be a continuously Fr\'echet differentiable operator evaluated at a point $x$. Then we define the smooth operator game $v: \mathcal{C}\rightarrow \R$ by
\[
v(S) \,:=\,
\Fhat_x\!\left(\bar u+\sum_{i=1}^m \mu_i(S)\,\eta_i h\right)-\Fhat_x(\bar u), \quad \mathrm{where} \; \mu_i(S) := \frac{\int_S \eta_i\,d\lam}{\int_\Omega \eta_i\,d\lam}\,.
\]
The game depends on $x$, $\bar u$, and $u_0$. We suppress this dependence in the notation.
\end{definition}
We may think of the $\eta_i$ as smooth bump functions equally distributed on $\Omega$ (see \cref{fig:pou_main}, with an exemplary construction in \cref{app:pou_construction}). For a coalition $S\in \mathcal{C}$ the game measures the change in the output at $x$, when each mask $\eta_i$ is scaled by the fraction of its mass that lies in $S$.

\begin{figure}[t]
  \centering
  \includegraphics[width=\textwidth, trim=0 0 0 177, clip]{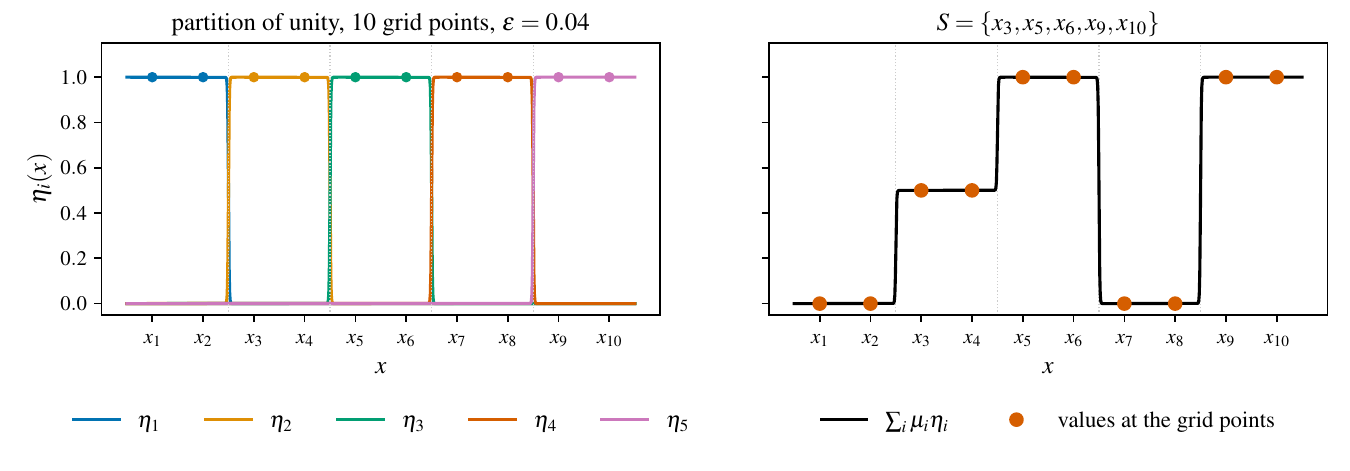}
  \caption{Partition of unity with five masks on a grid with two points per mask (construction in \cref{app:pou_construction}). \emph{Left:} the
  masks $\eta_i$, each equal to one at both grid points of its cell. \emph{Right:} the smooth mask $\sum_i\mu_i\eta_i$ of the coalition
  $S=\{x_3,x_5,x_6,x_9,x_{10}\}$, which holds one of the two grid points of $\eta_2$, so $\mu=(0,1/2,1,0,1)$.}
  \label{fig:pou_main}
\end{figure}

\begin{restatable}[The smooth operator game is pNA]{theorem}{SmoothGamePNA}
\label{thm:smooth_game_pna}
Under the assumptions of \cref{def:smooth_op_game}, the smooth operator game $v$ lies in $pNA$, and admits a unique Aumann--Shapley value $\varphi v$ given by the diagonal formula~\cite[Thm.~2]{neymanChapter56Values2002}.
\end{restatable}
\begin{proof}[Proof sketch]
$v$ factorizes as $v = f_x\circ(\mu_1,\dots,\mu_m)$ with $f_x(a):=\Fhat_x\!\bigl(\bar u+\sum_i a_i\eta_i h\bigr)-\Fhat_x(\bar u)$, a $C^1$ map on $[0,1]^m$ with $f_x(0)=0$, and the $\mu_i$ non-atomic probability measures on $(\Omega,\mathcal C)$. This is the canonical $pNA$ representation~\cite[after Def.~1]{neymanChapter56Values2002}, and the value claim follows from~\cite[Thm.~2]{neymanChapter56Values2002}. The full proof is in \cref{prf:smooth_game_pna}.
\end{proof}

\begin{remark}
In the implementation (\cref{alg:operatorshap}) we will use hard binary masks on the grid points, 
instead of the smooth masks $\eta_i$. This is not a deviation from the theory, but rather an 
effect of the discrete nature of neural operators. In practice, neural operators 
observe $u_0$ only through its values on a grid $\{x_1,\dots,x_p\}$.
Choosing a partition of unity in which the support of each $\eta_i$ contains only the grid point 
$x_i$ (\cref{app:pou_construction}) gives $\eta_i(x_j)=\delta_{ij}$, which \textbf{appears} as a hard binary 
mask, but \textbf{still satisfies the theoretical requirements}. The coalition weights $\mu_i(S)$ then differ from $0$ or $1$ by an amount set by the 
transition width of the $\eta_i$, which can be chosen below machine precision.
\end{remark}

\begin{remark}
The smoothness assumption concerns the surrogate $\Fhat$, not the underlying PDE. Most neural operators smooth out shocks or singularities of the true solution. Thus, while the true solution may not be accurately represented by the model, the attribution of the model's output is still well-defined and can be used to understand the model's behavior. On the other hand, our game does not need to satisfy boundary conditions or similar, as it explains only the surrogate model, not an underlying PDE. 
\end{remark}
\begin{remark}
The point evaluation $\operatorname{ev}_x$ can be replaced by any differentiable functional of the output field, e.g., a spatial average or
an integral over a patch, by the chain rule. For inputs with several channels, a coalition masks all channels of a cell together and the density sums their contributions. Channels can also be treated as separate player sets and masked independently.
\end{remark}

\subsection{Common neural operator architectures satisfy this}
The Fr\'echet $C^1$ assumption underlying \cref{def:smooth_op_game} and \cref{thm:smooth_game_pna} is met by the standard Fourier neural operator architecture~\cite{li_fourier_2021}, provided we use a smooth activation as shown in the following proposition.

\begin{restatable}[Smooth FNOs are Fr\'echet $C^1$]{proposition}{SmoothFNOFrechet}
\label{prop:fno_frechet}
For $s>d/2$ and $\sigma\in C^\infty(\R)$ with $\sigma(0)=0$, the FNO of App.~\ref{app:fno_architecture} defines a Fr\'echet $C^1$ map $\Fhat : H^s(\Omega;\R^{c_{\mathrm{in}}}) \to H^s(\Omega;\R^{c_{\mathrm{out}}})$.
\end{restatable}
\begin{proof}[Proof sketch:]
  All affine blocks are bounded linear maps on $H^s$, and the nonlinearity is a smooth Nemytskij superposition operator, which is Fr\'echet $C^1$ on $H^s$ for $s>d/2$. The composition of these blocks is therefore also Fr\'echet $C^1$. The full proof is in App.~\ref{prf:fno_frechet}.
\end{proof}

\begin{remark}
The same template covers most modern neural-operator architectures whose building blocks are either bounded linear maps on $H^s$ or smooth Nemytskij superpositions. We find empirically that the framework is quite robust if the activation is not smooth (e.g., the ReLU) (\cref{app:relu}). Future work could investigate this theoretically.
\end{remark}

\paragraph{Training an Aumann--Shapley explainer.} Since pNA $\subset$ ASYMP~\cite[Thm.~4]{neymanChapter56Values2002}, the game $v$ above also admits an \emph{asymptotic value}: along any refining sequence of finite coalition sets $\pi_0\subset \pi_1\subset \cdots$ with $\pi_k \subset \mathcal C$, the \emph{discrete Shapley values} $\psi v_{\pi_k}(S)$ of the game restricted to $\pi_k$ converge to $(\varphi v)(S)$ (\cref{app:aumann_shapley}). This means that we can use the discrete Shapley values to train an approximator of the Aumann--Shapley values. Strictly, this statement is about one fixed set of masks $\eta_i$ restricted to coarser partitions, whereas we define a separate set of masks at every resolution, which we find to work best in practice. Here, more granular samples will contribute to a more accurate explainer, but also less granular results can help to drive the overall convergence of the explainer, as we will see in the experiments. \cref{fig:shap_vs_operator} shows how discrete Shapley values approximate Aumann--Shapley values when we account for the size of the region the Shapley values represent.

\begin{remark}
\label{rem:bv_games}
The differentiability of $f_x$ can be relaxed to bounded variation with continuity at $0$ and at $\mu(\Omega)$~\cite{neymanChapter56Values2002}, which allows non-differentiable models and more interesting readouts (e.g., threshold games, \cref{app:midpoint}). The game then still has a value, but is not necessarily of class ASYMP anymore, so convergence has to be checked.
\end{remark}

\begin{figure}[t]
  \centering
  \includegraphics[width=\textwidth]{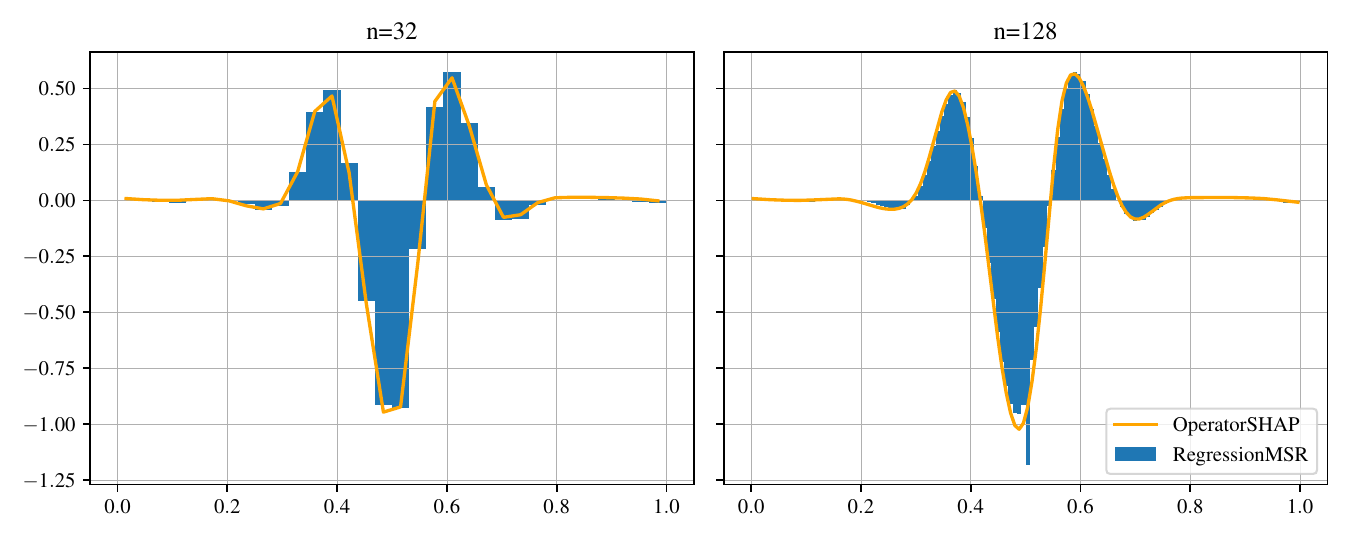}
  \caption{Comparison of classical Shapley values (blue bars) and Aumann--Shapley density (orange line) for player numbers (grid sizes) 32 (\emph{left}) and 128 (\emph{right}) for the 1D Burgers equation, approximated via RegressionMSR (for Shapley) and OperatorSHAP (Aumann--Shapley). In this figure, we scale Shapley values by the number of players to account for shrinking per-player attribution. We can see that with increasing player number, classical Shapley values approach the Aumann--Shapley density.
  }
  \label{fig:shap_vs_operator}
\end{figure}

\subsection{Connection to Integrated Gradients}
\label{sec:ig}
In this section, we show how the derived formula connects to the Integrated Gradients method.
\begin{definition}[Masked pointwise attribution]
\label{def:masked_attr}
Let $\Fhat:X\to Y$ be Fr\'echet $C^1$, fix $x\in \Omega$ and
$\eta\in C^\infty(\Omega)$, and set $\Fhat_x=\operatorname{ev}_x\circ \Fhat$.
For $\bar u,u_0\in X$, with $h=u_0-\bar u$, define the \emph{masked attribution at output
location $x$ associated with $\eta$} by
\[
\operatorname{Attr}_x^\eta(\bar u,u_0) \;:=\; \int_0^1 D\Fhat_x(\bar u+t h)[\eta h]\,dt.
\]
If $\Fhat$ is vector-valued, the definition is understood componentwise. 
\end{definition}
The above definition equals precisely the integrated gradients formula \cite{sundararajan_axiomatic_2017} (review in \cref{app:integrated_gradients}) but for function-valued games. Usually, we integrate the directional derivative of a vector-valued function, but here we integrate the directional derivative of a function-valued function. Choosing multiple smooth masks $\eta_1, \dots, \eta_m$ that sum to one $\sum_{j=1}^m \eta_j\equiv 1$ allows us again to achieve efficiency.

\begin{restatable}[Partition-of-unity attribution]{theorem}{PartitionUnityAttr}
\label{thm:partition_unity_attr}
Let $\Fhat:X\to Y$ be Fr\'echet $C^1$ and fix $x\in\Omega$.
Let $\eta_1,\dots,\eta_m\in C^\infty(\Omega)$ satisfy $\sum_{j=1}^m \eta_j\equiv 1$.
Then, for every $\bar u,u_0\in X$,
\[
\Fhat(u_0)(x)-\Fhat(\bar u)(x)
=
\sum_{j=1}^m \operatorname{Attr}_x^{\eta_j}(\bar u,u_0).
\]
\end{restatable}
\begin{proof}[Proof sketch:]
  We use linearity to pull the sum into the integral and use $\sum_j \eta_j h = h$. The rest follows from the fundamental theorem of calculus. The full proof is in \cref{prf:partitionOfUnity}.
\end{proof}
After substituting the partial derivative of $f_x$ from the proof of \cref{thm:smooth_game_pna}, evaluated along the diagonal $a=(t,\dots,t)$, and using the identity $\sum_j\eta_j\equiv 1$ to collapse the inner sum, the diagonal formula \eqref{eq:diagonal_formula} reduces to
\[
(\varphi v)(S)
\;=\;
\int_0^1 \sum_{i=1}^m \mu_i(S) D\Fhat_x\!\Big(\bar u + t\sum_{j=1}^m \eta_j h\Big)[\eta_i h]\,dt
\;=\;
\sum_{i=1}^m \mu_i(S) \int_0^1 D\Fhat_x\!\Big(\bar u + th\Big)[\eta_i h]\,dt,
\]
i.e., the Aumann--Shapley allocation to a coalition $S$ is the $\mu(S)$-weighted sum of the masked pointwise attributions of \cref{def:masked_attr}.

\section{OperatorSHAP: an amortized neural operator explainer}
\label{sec:method}

\begin{wrapfigure}{r}{0.48\textwidth}
\vspace{-\baselineskip}
\begin{algorithm}[H]
\DontPrintSemicolon
\textbf{Input:} backbone $\Fhat$, baseline $\bar u$, learning rate $\alpha$.\\
\textbf{Output:} explainer $\Phi_\theta$.\\[0.3em]
initialize $\Phi_\theta$\;
\While{not converged}{
  sample $u_0\sim p(u_0)$, $x^\star\sim\mathcal U(\Omega)$, coalition $S\sim p_w$\;
  $\bar v \gets v(\Omega)-v(\emptyset)$\;
  $\phi \gets \Phi_\theta(u_0, x^\star)$\tcp*{field on $\Omega$}
  $B \gets \text{area}(u_0,S)$\tcp*{get area of $S$}
  $\hat\phi_S \gets \!\int_{B}\!\phi$\tcp*{coarsen}
  \If{normalize}{
    $\hat\phi_S \gets \hat\phi_S + (|B|/|\Omega|)\bigl(\bar{v}-\int_{\Omega}\!\phi\bigr)$\;
  }
  $\mathcal L \gets \bigl(v(S)-\hat\phi_S\bigr)^2$\;
  $\theta \gets \theta - \alpha\,\nabla_\theta\mathcal L$\;
}
\caption{OperatorSHAP training}
\label{alg:operatorshap}
\end{algorithm}
\vspace{-\baselineskip}
\end{wrapfigure}

The previous section showed that we can define a smooth operator game $v$ for every query point $x\in\Omega$, with a well-defined Aumann--Shapley
value. That value is a function from the coalition space to the reals $\varphi v:\mathcal C\to\R,\,S \mapsto (\varphi v)(S)$. We know this
function is finitely additive, i.e., $(\varphi v)(S\cup T) = (\varphi v)(S) + (\varphi v)(T)$ for disjoint $S,T\in\mathcal C$, and that it
satisfies the efficiency property $(\varphi v)(\Omega)=v(\Omega)$. We are now interested in training one amortized explainer $\Phi_\theta$
that approximates this value function for all input functions $u_{0}$ and at all query positions $x$, i.e., $\Phi_\theta(u_{0}, x)(S) = \varphi v(S)$. Since pNA $\subset$ ASYMP~\cite[Thm.~4]{neymanChapter56Values2002}, the game $v$ possesses an \emph{asymptotic value}, so we can train a grid-agnostic explainer, like an
FNO, on the discrete Shapley values at multiple resolutions. The query point is passed to the explainer as an additional input channel that marks
its location on the input grid (a Gaussian bump of width $0.05$ of the domain around $x^\star$), so the explainer remains grid-agnostic. We then get the attribution of a coalition $S$ by integrating the explainer output over the area of $S$. 

\cref{alg:operatorshap} gives the resulting training loop. At each step we draw an input $u_0\in\R^p$ (where $p$ can be different for each sample), a
query point $x^\star$ uniformly from the grid, and fix a grid of partition cells $B_1, \dots, B_p$ for that sample. We then evaluate the explainer to obtain a
continuous attribution field $\phi=\Phi_\theta(u_0, x^\star)$, coarsen it to a discrete vector $\hat\phi\in\R^p$ by integrating over the partition
cells $B_i$, and minimize the FastSHAP weighted MSE $(v(S)-\hat\phi^\top\ind_S)^2$ for coalitions $S\subseteq\{1,\dots,p\}$ drawn from the
Shapley kernel $w_p$.\footnote{$w_p(S)\propto \tfrac{p-1}{\binom{p}{|S|}|S|(p-|S|)}$, the KernelSHAP weight~\cite{lundberg_unified_2017}, which puts most mass on very small or very large coalitions and is what makes the WLS objective unbiased for the Shapley value.} Hard normalization~\cite[Alg.~1]{jethani_fastshap_2022} is applied to enforce efficiency, $\sum_i\hat\phi_i=v(\Omega)-v(\emptyset)$. \cref{fig:shap_vs_operator} visualizes the attributions.

\section{Experiments}
\label{sec:experiments}

\paragraph{Setup.}
We use ten PDE datasets on structured 1D/2D/3D grids and the two MeshGraphNets benchmarks~\cite{pfaff_learning_2021} (\cref{tab:datasets}). We use FNO~\citep{li_fourier_2021}, LocalFNO~\citep{liu-schiaffini_neural_2024}, or GINO~\cite{li_geometry-informed_2023-1} (also MWT~\cite{gupta_multiwavelet-based_2021} for the analytic validation) as backbones that we explain (backbone errors in \cref{app:backbone_error}). The OperatorSHAP explainer mirrors the backbone family and is trained on the FastSHAP loss of \cref{sec:method}. To simulate heterogeneous grids on the structured PDEs, we stride each sample to a fixed grid granularity $n\in\{8, 16, 32, \dots, 1024\}$ ($n=512$ in 3D). We compare against exact Shapley ($n\le 16$), KernelSHAP~\citep{lundberg_unified_2017}, RegressionMSR~\cite{witter_regression-adjusted_2025}, and Integrated Gradients (IG, \cref{sec:ig}). KernelSHAP and RegressionMSR get the same budget of $2n$ (at least $1024$) backbone evaluations per query point. We mostly use the zero baseline, because it is the easiest to interpret and allows comparisons across samples. In some cases, however, it introduces physically implausible artifacts that hinder the backbone prediction. Other baselines and an ablation are discussed in \cref{app:baselines}. On the structured grids, every test input is explained at three query points, the domain center and two off-center points with coordinates at a quarter and three quarters of the domain, and on the meshes at three fixed points of the flow region. Metrics are Pearson correlation and normalized RMSE against the reference (RegressionMSR with $3\!\cdot\!10^5$ backbone evaluations, $5\!\cdot\!10^4$ on MGN Airfoil), plus faithfulness~\cite{DBLP:journals/jmlr/TsaiYR23} (weighted $R^2$ over $10{,}000$ random coalitions) as a reference-free diagnostic. We report medians over test inputs and query points. We also test preservation of spatial structure using the structural similarity index (SSIM) at two resolutions (\cref{app:ssim}). The full setup is in \cref{app:experimental_setup} and the hardware in \cref{app:compute}. We test for:

\paragraph{(i) Attribution quality across resolutions.}
We evaluate two games: (i) classical point prediction, where an attribution tells us what region influenced a point; (ii) thresholded outcome, which equals one when the prediction at a point is closer to its value at $u_0$ than to its value at the baseline and zero otherwise. This represents decision problems, where the precise change in value is less relevant than the general behavior of the PDE (\cref{app:midpoint}). The same trained explainer produces attributions at all resolutions, while the classical methods have to be rerun at each resolution. \cref{tab:cross_pde_metrics_compact} reports a subset of the results, with the full metrics in \cref{tab:cross_pde_metrics_full_top,tab:cross_pde_metrics_full_bot,tab:midpoint}. OperatorSHAP is competitive with the Shapley approximators and IG on most PDEs, while being faster on inference. On the non-differentiable decision tasks, IG is unable to perform meaningful predictions. We additionally ablate held-out resolutions (\cref{app:holdout}).

\begin{table}[t]
  \centering
  \caption{Attribution quality across resolutions $n$ for the point readout and the decision readout of
    \cref{app:midpoint} on two PDEs and MGN Cylinder. Each (PDE, $n$, method) entry reports the median over $30$
    test inputs and three query points per input on the top line and the inter-quartile range in brackets below. Metrics are the Shapley-kernel
    weighted $R^2$ (faithfulness), and Pearson correlation and NRMSE against the reference attribution. RegressionMSR (RMSR) and KernelSHAP (KS)
    use a budget of $2n$ (minimum $1024$) backbone evaluations per query point, IG uses a 50-step Riemann sum, and OperatorSHAP (OS) uses one
    query-conditioned model per PDE and readout at all resolutions. MGN Cylinder uses $10$ test inputs. On the decision readout, the gradient of the game vanishes, so IG attributions are identically zero and their Pearson correlation is undefined (--).
    Bold marks the best median and every median within one standard error of it over the test inputs, compared at three decimals. All PDEs
    are in \cref{tab:cross_pde_metrics_full_top,tab:cross_pde_metrics_full_bot,tab:midpoint}.}
  \label{tab:cross_pde_metrics_compact}
  \resizebox{\textwidth}{!}{%
    \input{./tables/all_pde_metrics_short.tex}  }
\end{table}

\paragraph{(ii) Inference speed.}
Once trained, OperatorSHAP returns full attributions in a single explainer forward pass and two backbone forward passes for the hard normalization. \cref{fig:matched_cost} in \cref{app:matched_cost} compares the runtime and faithfulness of all methods, and \cref{fig:burgers1d_budget_curves_error} shows the NRMSE of RegressionMSR and KernelSHAP vs.\ their budget. OperatorSHAP is orders of magnitude faster than RegressionMSR and KernelSHAP, while maintaining a low error. IG needs $\res{matched/ig_ratio_min}$ to $\res{matched/ig_ratio_max}$ times the runtime of OperatorSHAP where it reaches its faithfulness, and on Reaction--Diffusion 2D and Shallow Water 2D it does not reach it at any number of path steps. Again, the same trained explainer is used across resolutions, while baselines needed to be rerun at each resolution.

\paragraph{(iii) Amortization time.}
Training time of OperatorSHAP is costly. To illustrate when amortization pays off, we show the cumulative number of model evaluations of OperatorSHAP vs.\ baselines as the number of inference samples increases in \cref{fig:amortization}. Amortization pays off after $1.1\!\cdot\!10^4$ to $3.7\!\cdot\!10^5$ explanations depending on the PDE and the classical method. This is the regime of a deployed in-production surrogate that returns an explanation with every prediction. There, the latency of one explanation matters more than the one-time training cost.

\paragraph{(iv) Analytic validation.}
For linear PDEs, the true solution operator has a closed-form Aumann--Shapley density, which validates the explainer without a discrete estimator
in the loop. For the heat equation, it is an isotropic bump around the evaluation point, scaled by the initial condition. For the advection equation, it is a point mass at the characteristic foot point
$x^\star - \beta t^\star$. \cref{fig:analytic} in \cref{app:analytic} shows that the FNO explainer matches the heat kernel density. For advection, both explainers correctly place mass at the advection foot point. The largest value of the FNO explainer, however, is a spike at $x^\star$ itself. This spike is a
property of the backbone, whose pointwise skip connections pass $u_0(x^\star)$ directly to the output, and illustrates how attributions expose failure modes of a surrogate.

\section{Limitations}
\label{sec:limitations}
While amortized Shapley can be incredibly useful when needing explanations at scale, few people outside the industry will need that many explanations without changing the model in between. For companies that consistently need to monitor their AI application or want to give their users an explanation along the prediction, amortization can be very time-saving, however. Furthermore, PDE surrogate models often need high spatial resolution to make an accurate prediction. As in classical Shapley, however, more input features exponentially complicate accurate attributions, necessitating binning approaches that could obscure fine attribution spikes. Still, modern approaches to (amortized) SHAP have started to mitigate the computational burden while retaining accurate attributions. On the theory side, while our restriction to smooth masks is irrelevant for finite-dimensional data (\cref{app:pou_construction}) and differentiable operators can be robustly relaxed in practice (\cref{app:relu}), our assumptions are nonetheless restrictions that could potentially be relaxed in future work by deeper functional analysis concepts or other definitions of games. In our experiments, it should be noted that RMSR could be arbitrarily improved by increasing the budget. Where gradients of the backbone are available, Integrated Gradients reaches a similar performance on most of our PDEs (\cref{tab:cross_pde_metrics_full_top,tab:cross_pde_metrics_full_bot}), because their games are close to additive. The advantage of OperatorSHAP is then the lower cost per explanation and that it only needs forward passes of the backbone. On non-differentiable games, Integrated Gradients failed while OperatorSHAP still performed well. Lastly, we want to mention the risk that may arise when models produce false explanations. We are still working with black-box models whose output should not be trusted blindly when the result could be harmful in any way.
\section{Conclusion}
\label{sec:conclusion}
We have shown that Shapley attribution can be amortized for neural operators in a grid-agnostic way by training a neural operator explainer on a FastSHAP-like objective and connecting the resulting discrete attributions to the Aumann--Shapley value in the continuous limit. This allows direct evaluation of Shapley values on arbitrary grids and, in particular, training on datasets with heterogeneous grids.
For future work, we see three directions to explore. First, there is a lot of pre-existing work in classical Shapley estimation that could be applied to OperatorSHAP. For example, certain engineering tricks that focus the effort on high frequency regions could allow for higher player sizes with reduced computational effort. Second, our framework only allows for additive attributions, while interactions between regions are ignored. An extension to Shapley interactions would be interesting. Lastly, our theoretical assumptions may be unnecessarily restricting. One could explore how to integrate non-differentiable activations into the theory or even allow for discontinuous backbones.

\begin{ack}
JS acknowledges funding from the Munich Center for Machine Learning (MCML), as well as the support he received from the METEOR Research Training Group, which is sponsored by the German Research Foundation. FC gratefully acknowledges funding by the German Research Foundation (Deutsche Forschungsgemeinschaft, DFG) – GRK 3081 – Project number 534429653. The authors gratefully acknowledge the computational and data resources provided by the Leibniz Supercomputing Centre (www.lrz.de). We also thank Konstantinos Kotsakis for his support in keeping our jobs running on the cluster. The authors declare no competing interests.
\end{ack}

\bibliographystyle{plainnat}
{\small
\bibliography{OperatorSHAP}
}

\appendix
\newpage

\section{Additional background information}
\subsection{A short primer on PDEs and solution operators}
\label[appendix]{app:pde_primer}

Many physical systems are described by \emph{partial differential equations} (PDEs): the temperature in a metal bar, the velocity of a fluid in a pipe, the air pressure around an aircraft wing. A PDE is specified by a local rule---how a quantity at a specific location changes over time---that, together with a defined initial state and conditions at the boundary of the spatial region, determines the system's entire evolution.

\paragraph{The heat equation.}
An easy example is the one-dimensional heat equation. Imagine a thin rod of unit length, and let $u(x,t)$ denote its temperature at position $x\in[0,1]$ and time $t\ge 0$. If we hold both ends of the rod at temperature zero and start from some initial profile $u_0(x)$, the temperature evolves according to
\[
\partial_t u(x,t) \;=\; \kappa\,\partial_x^2 u(x,t),
\qquad u(0,t)=u(1,t)=0,
\qquad u(\cdot,0)=u_0,
\]
where $\kappa>0$ is a material constant determining how fast heat spreads. In plain words: at each point, the rate at which the temperature changes over time is proportional to how curved the temperature profile is at that point. If a point's surrounding is hotter than itself, it will heat up. If the surrounding is cooler, it will cool down. Due to the boundary conditions, the temperature everywhere will eventually cool down to $0$.

\paragraph{The solution operator.}
We are interested in modeling the temperature $u(x,t)$, when starting from initial temperature $u_0$ and letting the system run for a fixed time $t^\star$. The map that turns the initial profile into the final one,
\[
\mathcal F:\; u_0\;\longmapsto\; u(\cdot,t^\star),
\]
is the \emph{solution operator}. It takes a function as input and returns a function as output. Neural operators are parametric models that learn $\mathcal F$ from data. In essence, they are just neural networks that can handle arbitrary input and output discretizations.


\subsection{Integrated gradients}
\label[appendix]{app:integrated_gradients}
The differential analog of Shapley attribution in finite dimensions is
integrated gradients~\cite{sundararajan_axiomatic_2017}: for a continuously differentiable function $f:\R^n\to\R$ with baseline $x_0$, the contribution of coordinate $i$ to
$f(x)-f(x_0)$ is the path integral of its partial derivative,
\[
\operatorname{IG}_i(x_0,x)
\;=\;
(x_i-x_{0,i})\int_0^1 \partial_i f\bigl(x_0+t(x-x_0)\bigr)\,dt,
\]
which satisfies sensitivity, and
\emph{completeness} $\sum_i \operatorname{IG}_i(x_0,x)=f(x)-f(x_0)$.

\subsection{The Aumann--Shapley value: mathematical details}
\label[appendix]{app:aumann_shapley}

This appendix records the formal definitions used in \cref{sec:theory}. Our exposition follows~\cite{neymanChapter56Values2002}, to which we refer for proofs and historical context.

\paragraph{Spaces of games.}
The space of players is a measurable space $(\Omega,\mathcal C)$ that we assume isomorphic to $([0,1],\mathcal B([0,1]))$, with $\mathcal B$ the Borel $\sigma$-algebra. Members of $\Omega$ are \emph{players}, members of $\mathcal C$ are \emph{coalitions}, and a \emph{game} is a real-valued function $v:\mathcal C\to\R$ with $v(\emptyset)=0$. A game is \emph{monotonic} if $v(S)\ge v(T)$ whenever $S\supseteq T$, and of \emph{bounded variation} if it is the difference of two monotonic games. Its \emph{variation} is
\[
\|v\| \;:=\; \sup\Bigl\{\textstyle\sum_{j=1}^n |v(S_j)-v(S_{j-1})|\;:\; S_0\subseteq S_1\subseteq\cdots\subseteq S_n\text{ in }\mathcal C\Bigr\}.
\]
Under $\|\cdot\|$ the space $BV$ of all bounded-variation games is a Banach algebra. Within $BV$ live the finitely additive games $FA=\{v: v(S\cup T)=v(S)+v(T)\text{ for disjoint }S,T\}$, the countably additive measures $M\subset FA$, and the non-atomic measures $NA\subset M$, i.e.\ measures $\mu\in M$ for which every singleton in $\mathcal C$ has $\mu$-measure zero. Each of $NA,M,FA,BV$ is invariant under the action $(\Theta_*v)(S):=v(\Theta S)$ of the group $\mathcal G$ of measurable automorphisms of $(\Omega,\mathcal C)$. We write $Q^+$ for the monotonic elements of a set of games $Q$, and $Q^1$ for those with $v(\Omega)=1$.

\paragraph{The value axioms.}
Fix a linear subspace $Q\subseteq BV$ that is \emph{symmetric} in the sense of being closed under $\Theta_*$ for all $\Theta\in\mathcal G$. Following~\cite[Def.~1]{neymanChapter56Values2002}, a \emph{value on $Q$} is a linear map $\varphi:Q\to FA$ that is symmetric ($\varphi(\Theta_* v)=\Theta_*(\varphi v)$ whenever $v$ and $\Theta_*v$ both lie in $Q$), positive ($\varphi v\in BV^+$ whenever $v\in Q^+$), and efficient ($(\varphi v)(\Omega)=v(\Omega)$ for every $v\in Q$). On finite player sets these four axioms uniquely determine the classical Shapley value~\cite[Thm.~1]{neymanChapter56Values2002}. The substantive content of the infinite-player theory is that on several natural infinite-dimensional subspaces of $BV$ the same axioms still pin down a unique map $\varphi$.

\paragraph{The space $pNA$.}
The most important such subspace is $pNA$, the closure in $\|\cdot\|$ of the subalgebra generated by $NA$, equivalently the closed linear span of powers of non-atomic measures~\cite[Lem.~7.2]{aumannValuesNonAtomicGames2015}. The canonical examples are games of the form $v=f\circ\mu$ where $\mu=(\mu_1,\dots,\mu_m)\in(NA^1)^m$ is a vector of non-atomic probability measures and $f\in C^1([0,1]^m)$ with $f(0)=0$. Such $v$ lies in $pNA$~\cite[after Def.~1]{neymanChapter56Values2002}, and this is the form taken by the smooth operator games of \cref{thm:smooth_game_pna}. On $pNA$ there is a unique value~\cite[Thm.~2]{neymanChapter56Values2002}, given on the canonical games by the \emph{diagonal formula}
\[
(\varphi v)(T)\;=\;\int_0^1 \sum_{i=1}^m \mu_i(T)\,\partial_i f(t,\dots,t)\,dt,\qquad T\in\mathcal C,
\]
i.e.\ the integral of the partial derivatives of $f$ along the diagonal $a=(t,\dots,t)$ of the unit cube, weighted by the measures $\mu_i$.

\paragraph{The asymptotic value.}
A second, constructive route to a value approximates a continuum game by finite ones. Given a coalition $T\in\mathcal C$, a sequence $(\pi_k)_{k\ge 1}$ of finite measurable subfields of $\mathcal C$ is \emph{$T$-admissible} if $T\in\pi_1$, $\pi_k\subseteq\pi_{k+1}$ for all $k$, and $\bigcup_k \pi_k$ generates $\mathcal C$. For each $k$, the restriction $v_{\pi_k}$ of $v$ to $\pi_k$ is a finite-player game and admits the classical Shapley value $\psi v_{\pi_k}$. We ask whether these discrete Shapley values converge as the partition refines.

\begin{definition}[Asymptotic value, \protect{\cite[Def.~4]{neymanChapter56Values2002}}]
\label{def:asymp_value}
A finitely additive game $\varphi v$ is the \emph{asymptotic value} of $v\in BV$ if, for every $T\in\mathcal C$ and every $T$-admissible sequence $(\pi_k)$,
\[
\lim_{k\to\infty}\bigl(\psi v_{\pi_k}\bigr)(T)\;=\;(\varphi v)(T).
\]
The set of all bounded-variation games admitting an asymptotic value is denoted $ASYMP$. When it exists, the asymptotic value is unique, and the map $v\mapsto\varphi v$ is itself a value on $ASYMP$~\cite[Thm.~4]{neymanChapter56Values2002}.
\end{definition}

The link between the two routes is the inclusion $pNA\subset ASYMP$, with the asymptotic value coinciding with the unique value on $pNA$~\cite[Thm.~4]{neymanChapter56Values2002}. This is what we exploit in \cref{sec:method}: the discrete Shapley values $\psi v_{\pi_k}$ of any refining sequence of finite subfields converge to the same limit---the Aumann--Shapley value $(\varphi v)(T)$---which justifies training a discretization-based explainer at multiple grid resolutions and reading off the continuum density at the limit.


\subsection{Reviewing the FNO architecture}
\label{app:fno_architecture}
We repeat the definition of a truncated $L$-layer FNO with lifting and projection: starting from $v_0 := P_{\mathrm{up}}u + b_{\mathrm{up}}$, define recursively
\begin{equation}
\label{eq:fno_recursion}
v_{\ell+1} = \sigma\bigl(A_\ell v_\ell + b_\ell\bigr),
\qquad \ell=0,\dots,L-1,
\end{equation}
and set $\Fhat(u) := P_{\mathrm{down}}v_L + b_{\mathrm{down}}$. We assume $s>d/2$, and have that
\[
P_{\mathrm{up}}:H^s(\Omega;\R^{c_{\mathrm{in}}})\to H^s(\Omega;\R^{c_0}),
\qquad
P_{\mathrm{down}}:H^s(\Omega;\R^{c_L})\to H^s(\Omega;\R^{c_{\mathrm{out}}}),
\]
are bounded linear lifting and projection maps with biases $b_{\mathrm{up}}\in H^s(\Omega;\R^{c_0})$ and $b_{\mathrm{down}}\in H^s(\Omega;\R^{c_{\mathrm{out}}})$. Each $A_\ell = W_\ell + K_\ell$ has $W_\ell:H^s(\Omega;\R^{c_\ell})\to H^s(\Omega;\R^{c_{\ell+1}})$ bounded linear and $K_\ell$ a finite-mode Fourier multiplier defined by $\widehat{K_\ell v}(k)=R_{\ell,k}\widehat v(k)$ for $|k|_\infty\le m_\ell$ and zero otherwise, with cutoff $m_\ell\in\mathbb{N}$ and learned $R_{\ell,k}\in\R^{c_{\ell+1}\times c_\ell}$. Layer biases satisfy $b_\ell\in H^s(\Omega;\R^{c_{\ell+1}})$, and $\sigma:\R\to\R$ acts componentwise.

\section{Proofs}\label{app:proofs}

\subsection{The operator game is in pNA}

\SmoothGamePNA*
\begin{proof}
\label[appendix]{prf:smooth_game_pna}
Define $f_x:[0,1]^m\to\R$ by $f_x(a):=\Fhat_x\!\bigl(\bar u+\sum_{i=1}^m a_i\,\eta_i h\bigr)-\Fhat_x(\bar u)$ and $\mu_i(S):=\int_S \eta_i\,d\lam/\!\int_\Omega \eta_i\,d\lam$. Then $v = f_x\circ(\mu_1,\dots,\mu_m)$ by construction.

\emph{The $\mu_i$ are non-atomic probability measures.} Each $\mu_i$ is a probability measure on $(\Omega,\mathcal C)$ since $\eta_i\geq 0$ and $\mu_i(\Omega)=1$, and is absolutely continuous w.r.t.\ Lebesgue measure with density $\eta_i/\!\int_\Omega \eta_i\,d\lam$, hence non-atomic since Lebesgue measure on $\Omega=\mathbb T^d$ is non-atomic.

\emph{$f_x\in C^1([0,1]^m)$ with $f_x(0)=0$.} The map $\Lambda:\R^m\to X$, $a\mapsto \bar u+\sum_i a_i\eta_i h$, is affine with bounded derivative $D\Lambda[a]:b\mapsto \sum_i b_i\eta_i h$ (the $\eta_i h$ lie in $X$ since multiplication by $\eta_i\in C^\infty(\Omega)$ preserves $H^s$). Since $\Fhat_x$ is Fr\'echet $C^1$, the chain rule gives $f_x\in C^1([0,1]^m)$ with $\partial_i f_x(a)=D\Fhat_x(\Lambda(a))[\eta_i h]$. Setting $a=0$ yields $f_x(0)=\Fhat_x(\bar{u})-\Fhat_x(\bar{u})=0$.

The factorization $v=f_x\circ(\mu_1,\dots,\mu_m)$ therefore matches the canonical $pNA$ representation~\cite[after Def.~1]{neymanChapter56Values2002}, so $v\in pNA$. Existence, uniqueness, and the diagonal-formula expression of the Aumann--Shapley value follow from~\cite[Thm.~2]{neymanChapter56Values2002}.
\end{proof}

\subsection{Smooth FNOs are Fr\'echet \texorpdfstring{$C^1$}{C\textasciicircum 1}}
\label[appendix]{app:fno_frechet}

\SmoothFNOFrechet*

\begin{proof}

\label{prf:fno_frechet}
We proceed in four steps.

\smallskip
\noindent\textbf{Step 1: every linear component is bounded on the relevant Sobolev spaces.}
We first verify that all affine-linear pieces of the architecture act continuously on
the Sobolev spaces under consideration.

\emph{(a) Lifting and projection.}
By assumption,
\[
P_{\mathrm{up}}:H^s(\Omega;\R^{c_{\mathrm{in}}})\to H^s(\Omega;\R^{c_0}),
\qquad
P_{\mathrm{down}}:H^s(\Omega;\R^{c_L})\to H^s(\Omega;\R^{c_{\mathrm{out}}})
\]
are bounded linear maps.
Hence
\[
u\mapsto P_{\mathrm{up}}u+b_{\mathrm{up}}
\quad\text{and}\quad
v\mapsto P_{\mathrm{down}}v+b_{\mathrm{down}}
\]
are affine continuous maps between Banach spaces.

\emph{(b) Channel-mixing operators.}
For each layer $\ell$, the operator
\[
W_\ell:H^s(\Omega;\R^{c_\ell})\to H^s(\Omega;\R^{c_{\ell+1}})
\]
is assumed bounded linear.

\emph{(c) Truncated Fourier multipliers.}
Fix $\ell$ and $v\in H^s(\Omega;\R^{c_\ell})$.
On the torus and since $K_\ell$ has only finitely many active modes,
\[
\|K_\ell v\|_{H^s}^2
=
\sum_{k \in \mathbb{Z}^d}
(1+|k|^2)^s |\widehat {K_\ell v}(k)|^2
=
\sum_{|k|_\infty\le m_\ell}
(1+|k|^2)^s |R_{\ell,k}\widehat v(k)|^2.
\]
Using the operator norm of each matrix $R_{\ell,k}$,
\[
|R_{\ell,k}\widehat v(k)|^2
\le
\|R_{\ell,k}\|_{\mathrm{op}}^2\,|\widehat v(k)|^2,
\]
so
\[
\|K_\ell v\|_{H^s}^2
\le
\Bigl(\sup_{|k|_\infty\le m_\ell}\|R_{\ell,k}\|_{\mathrm{op}}^2\Bigr)
\sum_{|k|_\infty\le m_\ell}(1+|k|^2)^s |\widehat v(k)|^2
\le
C_\ell^2 \|v\|_{H^s}^2,
\]
where
\[
C_\ell:=\sup_{|k|_\infty\le m_\ell}\|R_{\ell,k}\|_{\mathrm{op}} < \infty.
\]
Thus $K_\ell$ is bounded linear on
$H^s(\Omega;\R^{c_\ell})\to H^s(\Omega;\R^{c_{\ell+1}})$.

\emph{(d) Layerwise affine map.}
Since both $W_\ell$ and $K_\ell$ are bounded linear, so is
\[
A_\ell=W_\ell+K_\ell.
\]
Therefore
\[
v\mapsto A_\ell v+b_\ell
\]
is an affine continuous map
\[
H^s(\Omega;\R^{c_\ell})\to H^s(\Omega;\R^{c_{\ell+1}}).
\]

In particular, every linear or affine block appearing in the architecture is well-defined
and continuous on the corresponding Sobolev space.

\smallskip
\noindent\textbf{Step 2: the activation acts smoothly on $H^s$.}
Because $s>d/2$, $H^s(\Omega)$ is a Banach algebra under pointwise multiplication and embeds continuously into $L^\infty(\Omega)$ (Runst--Sickel~\cite{runstSobolevSpacesFractional1996}, Thm.~4.6.4/1, applied with $p=q=2$ so that
$F^s_{2,2}=H^s$ and $s>d/2=d/p$).
The Nemytskij theorem for smooth nonlinearities
(\cite{runstSobolevSpacesFractional1996}, Thm.~5.5.3/2) implies that for
$\sigma\in C^\infty(\R)$ with $\sigma(0)=0$, the superposition map
\[
S_\sigma:H^s(\Omega;\R^m)\to H^s(\Omega;\R^m),
\qquad
S_\sigma(w)=\sigma\circ w,
\]
is Fr\'echet $C^1$ for every finite $m$, with derivative
\[
DS_\sigma(w)[z]=\sigma'(w)\,z
\]
componentwise.
Since $H^s$ is a Banach algebra and $\sigma'(w)\in H^s\cap L^\infty$,
the map $z\mapsto \sigma'(w)z$ is bounded linear on $H^s$.

\smallskip
\noindent\textbf{Step 3: each hidden layer map is Fr\'echet $C^1$.}
Define
\[
T_\ell(v):=S_\sigma(A_\ell v+b_\ell),
\qquad \ell=0,\dots,L-1.
\]
By Step 1, the map $v\mapsto A_\ell v+b_\ell$ is affine continuous.
By Step 2, $S_\sigma$ is Fr\'echet $C^1$.
Hence the chain rule yields that each $T_\ell$ is Fr\'echet $C^1$ as a map
\[
H^s(\Omega;\R^{c_\ell})\to H^s(\Omega;\R^{c_{\ell+1}}).
\]

Similarly, the lifting map
\[
T_{\mathrm{up}}(u):=P_{\mathrm{up}}u+b_{\mathrm{up}}
\]
and the readout map
\[
T_{\mathrm{down}}(v):=P_{\mathrm{down}}v+b_{\mathrm{down}}
\]
are affine continuous, hence Fr\'echet $C^1$.

\smallskip
\noindent\textbf{Step 4: finite composition preserves $C^1$.}
The full operator can be written as
\[
\Fhat
=
T_{\mathrm{down}}
\circ
T_{L-1}
\circ
\cdots
\circ
T_0
\circ
T_{\mathrm{up}}.
\]
Since each factor is Fr\'echet $C^1$, repeated application of the chain rule shows that
\[
\Fhat:H^s(\Omega;\R^{c_{\mathrm{in}}})\to H^s(\Omega;\R^{c_{\mathrm{out}}})
\]
is Fr\'echet $C^1$.
\end{proof}

\begin{remark}
The proposition covers standard smooth activations such as $\tanh$
and GELU.
It does not apply verbatim to ReLU, which is not $C^1$.
See the ablation in \cref{app:relu}.
\end{remark}

\subsection{Extension from the torus to general domains}
\label[appendix]{app:general_domains}
Proposition \ref{prop:fno_frechet} is stated on the torus because the Fourier multipliers $K_\ell$ act on periodic functions. Let $\Omega\subset\R^d$ be a
bounded $C^\infty$ domain contained in a cube that we identify with $\mathbb T^d$. By \cite[Thm.~2.4.2/2]{runstSobolevSpacesFractional1996}
(with $H^s=F^s_{2,2}$) there is a bounded linear extension operator $E:H^s(\Omega)\to H^s(\R^d)$ with $Eu|_\Omega=u$. Multiplying by a cutoff $\chi\in C_c^\infty(\R^d)$ with
$\chi\equiv 1$ on $\Omega$ and support inside the cube, and reading the result as a periodic function, gives a bounded linear map
$E_\chi:H^s(\Omega;\R^{c_{\mathrm{in}}})\to H^s(\mathbb T^d;\R^{c_{\mathrm{in}}})$. The restriction
$R:H^s(\mathbb T^d;\R^{c_{\mathrm{out}}})\to H^s(\Omega;\R^{c_{\mathrm{out}}})$ is bounded linear as well. Define
$\Fhat_\Omega := R\circ\Fhat\circ E_\chi$. In practice this composition is the padding used when FNOs are applied to non-periodic data.
Since $E_\chi$ and $R$ are bounded linear and $\Fhat$ is Fr\'echet $C^1$ by \Cref{prop:fno_frechet}, the chain rule gives that
$\Fhat_\Omega$ is Fr\'echet $C^1$ with $D\Fhat_\Omega(u) = R\circ D\Fhat(E_\chi u)\circ E_\chi$.

The remaining results transfer with $\Omega$ in place of $\mathbb T^d$. The Sobolev embedding $H^s(\Omega)\hookrightarrow C^0(\bar\Omega)$ for
$s>d/2$ follows by extending, embedding on $\R^d$, and restricting, so point evaluation and \cref{def:smooth_op_game} are well-defined. The proofs of
\cref{thm:smooth_game_pna,thm:partition_unity_attr} use only that Lebesgue measure on $\Omega$ is non-atomic, that multiplication by
$\eta_i\in C^\infty(\bar\Omega)$ preserves $H^s(\Omega)$, and the chain rule, all of which hold on $\Omega$.

\subsection{Partition of unity for function-valued Integrated Gradients}

\PartitionUnityAttr*
\begin{proof}
\label[appendix]{prf:partitionOfUnity}
Set $h=u_0-\bar u$ and $\Fhat_x=\operatorname{ev}_x\circ\Fhat$ as in \cref{def:masked_attr}. Since $\Fhat_x$ is Fr\'echet $C^1$, the chain rule gives $\tfrac{d}{dt}\Fhat_x(\bar u+t h) = D\Fhat_x(\bar u+t h)[h]$, which is continuous on $[0,1]$. The fundamental theorem of calculus then yields
\[
\Fhat(u_0)(x)-\Fhat(\bar u)(x) = \Fhat_x(u_0)-\Fhat_x(\bar u) =
\int_0^1 D\Fhat_x(\bar u+t h)[h]\,dt.
\]
Since $\sum_j \eta_j=1$, $h=\sum_j \eta_j h$ in $X$.
Linearity of $D\Fhat_x(\bar u+t h)$ in its direction argument gives
$D\Fhat_x(\bar u+t h)[h] = \sum_j D\Fhat_x(\bar u+t h)[\eta_j h]$.
Summing under the integral,
\[
\Fhat(u_0)(x)-\Fhat(\bar u)(x)
=
\sum_{j=1}^m \int_0^1 D\Fhat_x(\bar u+t h)[\eta_j h]\,dt
=
\sum_{j=1}^m \operatorname{Attr}_x^{\eta_j}(\bar u,u_0). \qedhere
\]
\end{proof}

\subsection{Construction of a partition of unity on a discrete grid}
\label[appendix]{app:pou_construction}
Here, we show how to construct partitions of unity for data of a given discrete grid. Our goal is to have a smooth bump for each grid point, with controllable overlap between them. Let $g(t)=e^{-1/t}$ for $t>0$ and $g(t)=0$ for $t\le 0$, and define the smooth step
$\sigma_\varepsilon(t) = g(t+\varepsilon)\,/\,\bigl(g(t+\varepsilon)+g(\varepsilon-t)\bigr)$,
which is $C^\infty$, equals $0$ for $t\le-\varepsilon$ and $1$ for $t\ge\varepsilon$, and is monotone in between. For grid points
$x_1<\dots<x_p$ in $[0,1]$ with midpoints $m_i=(x_i+x_{i+1})/2$ and a transition width $\varepsilon<\min_i (x_{i+1}-x_i)/2$, set
\begin{align*}
\eta_1(x) &= 1-\sigma_\varepsilon(x-m_1),\\
\eta_i(x) &= \sigma_\varepsilon(x-m_{i-1})-\sigma_\varepsilon(x-m_i)\quad(1<i<p),\\
\eta_p(x) &= \sigma_\varepsilon(x-m_{p-1}).
\end{align*}
The $\eta_i$ are smooth and non-negative, sum to one by telescoping, and satisfy $\eta_i(x_j)=\delta_{ij}$ because every grid point lies at distance
more than $\varepsilon$ from all midpoints. Products of one-dimensional masks give the construction on $d$-dimensional grids. On the torus, $\eta_1$
and $\eta_p$ are merged into one periodic mask.

We now compare the smooth game of \cref{def:smooth_op_game} with the binary grid game of the experiments. Let $B_i$ be the cell of grid
point $x_i$, bounded by the midpoints $m_{i-1}$ and $m_i$, and let $S$ be a union of cells. The theory masks the input as
$\bar u+\sum_i\mu_i(S)\,\eta_i h$ with $\mu_i(S)=\int_S\eta_i\,d\lam\,/\!\int_\Omega\eta_i\,d\lam$, while the experiments use
$\bar u(x_j)+\mathbf 1_{x_j\in S}\,h(x_j)$ at every grid point. The backbone reads its input only at the grid points, where
$\eta_i(x_j)=\delta_{ij}$, so the smoothly masked input evaluates to $\bar u(x_j)+\mu_j(S)\,h(x_j)$ there. The two inputs therefore differ at $x_j$ by
$|\mu_j(S)-\mathbf 1_{x_j\in S}|\,|h(x_j)|$.

The support of $\eta_i$ is $[m_{i-1}-\varepsilon,\,m_i+\varepsilon]$, so it extends by at most $\varepsilon$ into each of the two neighboring
cells, and since $\eta_i\le 1$ the mass of $\eta_i$ outside $B_i$ is at most $2\varepsilon$. If $x_i\in S$ then $B_i\subseteq S$ and
$\int_\Omega\eta_i-\int_S\eta_i\le\int_{\Omega\setminus B_i}\eta_i\le 2\varepsilon$. If $x_i\notin S$ then $S\cap B_i=\emptyset$ and
$\int_S\eta_i\le\int_{\Omega\setminus B_i}\eta_i\le 2\varepsilon$. In both cases
\[
|\mu_i(S)-\mathbf 1_{x_i\in S}| \;\le\; \frac{2\varepsilon}{\int_\Omega\eta_i\,d\lam} \;\le\; \frac{2\varepsilon}{|B_i|-2\varepsilon}.
\]
The width $\varepsilon$ is free as long as it is positive, which keeps the $\eta_i$ in $H^s$, and no result requires it to be large. We
therefore choose $\varepsilon$ so small that this deviation is below floating point precision. The smooth and the binary masked input are then
the same tensor for the backbone, so \cref{alg:operatorshap} computes the values of the smooth game with binary masks. The integral of
\cref{alg:operatorshap}, the coarsening $\hat\phi_i=\int_{B_i}\phi$, involves no mask: by the
expansion of the diagonal formula at the end of \cref{sec:ig}, the masks enter the value $(\varphi v)(S)$ only through
$\sum_i\mu_i(S)\int_0^1 D\Fhat_x(\bar u+th)[\eta_i h]\,dt$, while $\int_{B_i}\phi$ integrates the explainer
output over a cell and is the same in the smooth and in the binary view. Note that the FastSHAP loss regresses the discrete Shapley values of the
game at the current partition and not this diagonal formula, which the explainer only approaches as the partitions get finer (\cref{app:aumann_shapley}). \cref{fig:pou_example} shows the construction.

\begin{figure}[h]
  \centering
  \includegraphics[width=\textwidth]{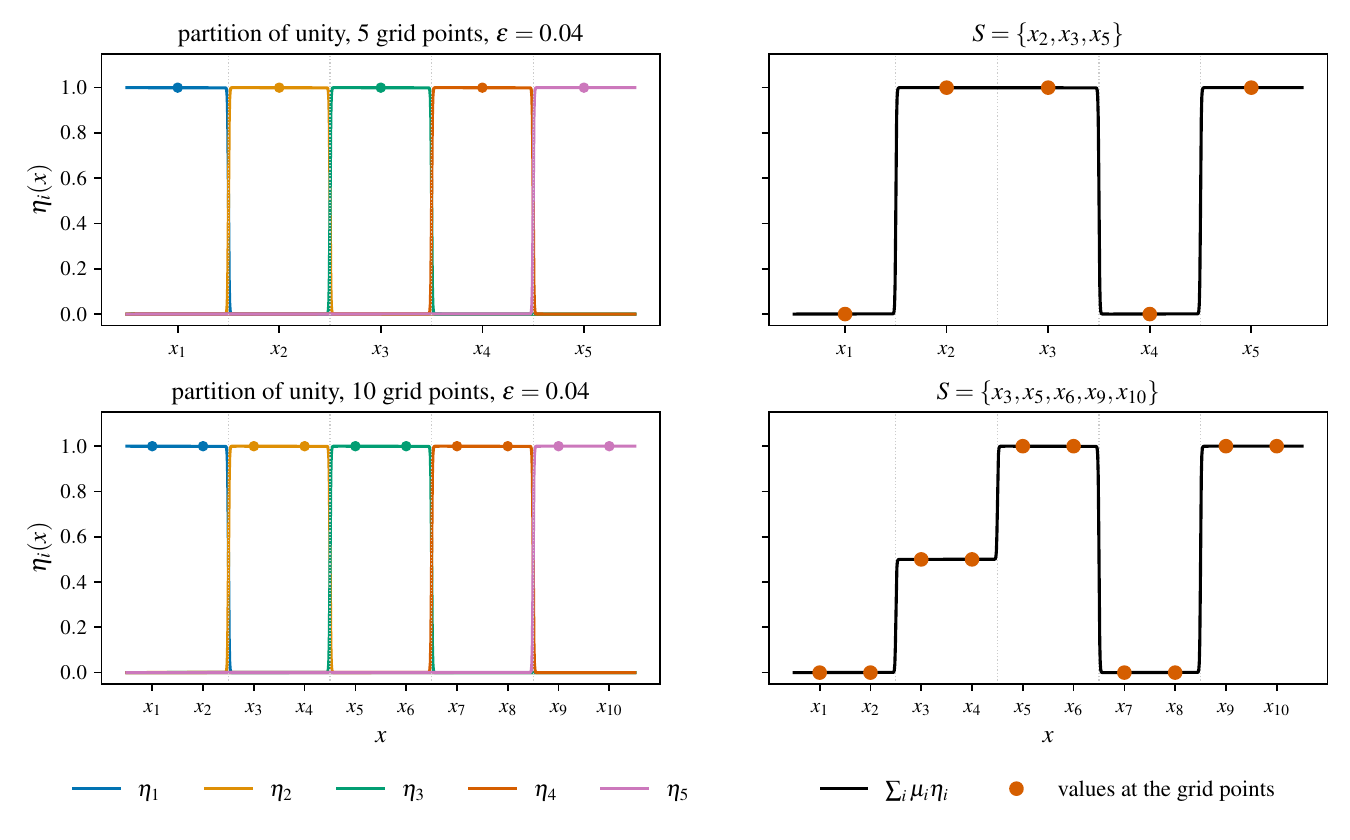}
  \caption{Grid-supported partition of unity with five masks. The transition width $\varepsilon=0.04$ is chosen large so that the transitions are
  visible. Dotted lines mark the midpoints $m_i$. \emph{Top:} one grid point per mask. \emph{Left:} the masks $\eta_i$. Dots mark $\eta_i(x_i)=1$.
  \emph{Right:} the smooth mask $\sum_i\mu_i\eta_i$ of the coalition $S=\{x_2,x_3,x_5\}$ and its values at the grid points, which form the binary
  grid mask. \emph{Bottom:} the same masks on a grid with two points per mask, so each $\eta_i$ equals one at both grid points of its cell.
  \emph{Right:} the coalition $S=\{x_3,x_5,x_6,x_9,x_{10}\}$ holds one of the two grid points of $\eta_2$, so $\mu_2=1/2$ and the smooth mask is
  $1/2$ at both $x_3$ and $x_4$. In total, $\mu=(0,1/2,1,0,1)$.}
  \label{fig:pou_example}
\end{figure}

\subsection{Multiple output channels}
\label[appendix]{app:multi_channel}
\cref{def:smooth_op_game} needs a scalar game. We therefore reduce the backbone output at $x$ to one number before we build
the game. On the grid datasets we read the first output channel (the $u$ channel of Gray--Scott) and on the airfoil and cylinder flow the
velocity magnitude. The proofs do not change, because both readouts are $C^1$ maps of the output vector (the norm away from
zero) and compose with $\operatorname{ev}_x$ by the chain rule. One could instead train one explainer per output channel, use a vectorized version of the loss, which averages over the dimensions, or pass the channel index to the explainer as an additional input, like the query point. We do not do this.


\section{Experimental setup}
\label[appendix]{app:experimental_setup}

\subsection{Datasets}
\label[appendix]{app:datasets}

We use nine PDE datasets across 1D and 2D, the 3D heat equation, plus two unstructured-mesh
datasets from the MeshGraphNets benchmark (\cref{tab:datasets}). The 1D PDEs all use a native
resolution of $1024$ points on $[0,1]$, the structured 2D PDEs use
$128\times 128$ points on a unit square (or $[0,2.5]^2$ for Gray--Scott,
see below), the 3D heat equation uses $32^3$ points on the unit cube, and every structured dataset is sampled at the training
resolutions $\{8, 16, 32, 64, 128, 256, 512, 1024\}$ by stride-downsampling from native
resolution (for 2D data, we used $\{16, 64, 256, 1024\}$ players and for 3D data $512$ players). Each dataset has $5\,000$ generated samples split into 4220 train, 750 val, and 30 test samples.

\paragraph{Linear PDEs (FFT-analytic solver).}
The advection equation $\partial_t u + \beta\partial_x u = 0$ on a periodic
domain ($\beta=1.0$), the heat equation $\partial_t u = \kappa\partial_x^2 u$ in
1D ($\kappa=0.01$), and its 2D and 3D analogs
$\partial_t u = \kappa\Delta u$ ($\kappa=0.01$) admit
closed-form transfer functions in Fourier space, so the solver applies
the analytic propagator at each requested time.

\paragraph{Nonlinear PDEs (pseudo-spectral method of lines).}
Burgers $\partial_t u + u\partial_x u = \nu\partial_x^2 u$ ($\nu=0.01$) and
Allen--Cahn $\partial_t u = \varepsilon^2\partial_x^2 u + u - u^3$
($\varepsilon=0.1$) are integrated with FFT-based spatial derivatives (for Burgers, the nonlinear term keeps the modes
$|k|\le 170$ of $512$), and an adaptive time integrator
(\texttt{RK45} for Burgers, the stiff \texttt{Radau} for Allen--Cahn). The 2D
Gray--Scott reaction--diffusion system,
\(\partial_t u = D_u\Delta u - uv^2 + F(1-u),\;
  \partial_t v = D_v\Delta v + uv^2 - (F+k)v\)
with $D_u=2\!\cdot\!10^{-5}$, $D_v=10^{-5}$, $F=0.04$, $k=0.06$ on
$[0,2.5]^2$, is integrated the same way (\texttt{RK45}) but for a long
time horizon $T=5000$ in order to reach the pattern-forming regime.
Gray--Scott has two channels in input and output ($u$ and $v$), while the
other PDEs are scalar.

\paragraph{Finite-difference / finite-volume PDEs.}
Diffusion--sorption,
$\partial_t u = D\partial_x^2 u / R(u)$ with Freundlich retardation
$R(u)=1+(\rho_s/\phi)k_f n_f u^{n_f-1}$, lives on a non-periodic
domain with Dirichlet inflow on the left and zero-flux on the right. We
use centered finite differences and \texttt{Radau} time stepping
($D=5\!\cdot\!10^{-4}$, $k_f=3.5$, $n_f=0.4$, $\rho_s=1.0$, $\phi=0.29$,
$c_{\mathrm{left}}=1.0$). The 2D shallow-water system
$\partial_t (h, hu, hv) + \nabla\cdot F(h, hu, hv) = 0$ with $g=9.81$ uses
a first-order finite-volume scheme with HLL Riemann fluxes, an SSP-RK3
time integrator at $\mathrm{CFL}=0.4$, and reflective boundaries. We
predict only the height field $h$ as the backbone target. Darcy flow
$-\nabla\cdot(a(x)\nabla u) = f$ is the only static problem in the set:
the solver assembles a sparse stencil with harmonic averaging at cell
faces and inverts it directly with \texttt{scipy.sparse.linalg.spsolve},
with permeability $a$ drawn from a log-normal random field with power
spectrum $|k|^{-4}$ and forcing $f=1$.

\paragraph{MeshGraphNets benchmarks.}
We additionally use the airfoil and cylinder-flow datasets from the
MeshGraphNets release~\cite{pfaff_learning_2021}. The airfoil
dataset is a transonic Euler flow around a NACA airfoil with
$3\,994$ unstructured mesh nodes per sample after the cutout (the node count is fixed across samples). We predict
the velocity field. The cylinder-flow dataset is an incompressible flow
past a cylinder with a genuinely heterogeneous mesh -- the node count
varies sample-to-sample because of adaptive refinement -- and we
predict the velocity field. Cylinder-flow is the dataset that motivates
a grid-agnostic explainer at all: a fixed-grid amortized explainer
cannot apply across samples that do not share a common discretization.
We inherit the train/val/test split from the original release.

\paragraph{Multi-resolution sampling.}
At training time, each sample is deterministically pinned to one of the
possible pre-selected resolutions via the \texttt{ResolutionBatchSampler}: the
samples are split into contiguous blocks, one per resolution, so a given sample is always served at the same resolution across
epochs. The mesh-based datasets do not use multi-resolution sampling, but assign each node to one of $64$ boxes of an
$8\times 8$ grid over the normalized mesh frame to reduce the number of players.

\begin{table}[h]
\centering
\caption{The PDE datasets. All structured-grid datasets have native
spatial resolution $1024$ (1D), $128\times 128$ (2D), or $32^3$ (3D), $5\,000$
generated samples, split into 4220 train, 750 val, and 30 test samples. ``$t_{\mathrm{idx}}$''
is the snapshot index used as the backbone target on a uniform grid of
$n_{\mathrm{snap}}$ snapshots over $[0, T]$. The label ``static'' marks the only
time-independent PDE. ``IC family'' names the random initial-condition
generator.}
\label{tab:datasets}
\resizebox{\textwidth}{!}{%
\small
\begin{tabular}{@{}lllrrl@{}}
\toprule
\textbf{PDE} & \textbf{Equation / parameters} & \textbf{IC family} & $T$ & $n_{\mathrm{snap}}$ & $t_{\mathrm{idx}}$ \\
\midrule
Advection 1D & $\partial_t u + \beta\partial_x u = 0$, $\beta=1$ & mixed-1D & $1.0$ & $101$ & $30$ \\
Heat 1D & $\partial_t u = \kappa\partial_x^2 u$, $\kappa=0.01$ & mixed-1D & $1.0$ & $101$ & $50$ \\
Burgers 1D & $\partial_t u + u\partial_x u = \nu\partial_x^2 u$, $\nu=0.01$ & mixed-1D & $1.0$ & $101$ & $50$ \\
Allen--Cahn 1D & $\partial_t u = \varepsilon^2\partial_x^2 u + u - u^3$, $\varepsilon=0.1$ & tanh-mixture & $0.5$ & $101$ & $50$ \\
Diff.--Sorp. 1D & Freundlich retardation; $D=5\!\cdot\!10^{-4}$ & boundary-driven & $500$ & $101$ & $50$ \\
Heat 2D & $\partial_t u = \kappa\Delta u$, $\kappa=0.01$ & rand.~Fourier (8 modes) & $1.0$ & $51$ & $25$ \\
Reac.--Diff. 2D (Gray--Scott) & $D_u=2\!\cdot\!10^{-5},\,D_v=10^{-5},\,F{=}0.04,\,k{=}0.06$ & seeded patches & $5000$ & $51$ & $25$ \\
Shallow Water 2D & $g=9.81$, $\mathrm{CFL}=0.4$ & Gaussian bump & $0.5$ & $51$ & $25$ \\
Darcy 2D & $-\nabla\!\cdot\!(a\nabla u){=}1$, $a\sim$log-normal & log-normal field & static & --- & --- \\
Heat 3D & $\partial_t u = \kappa\Delta u$, $\kappa=0.01$ & rand.~Fourier (4 modes) & $1.0$ & $11$ & $5$ \\
\midrule
MGN Airfoil & transonic Euler around NACA & --- & --- & --- & --- \\
MGN Cylinder & incompressible cylinder flow & --- & --- & --- & --- \\
\bottomrule
\end{tabular}
}
\end{table}

The 1D ``mixed-1D'' IC family draws with equal probability from random
20-mode Fourier series with power-law decay, single Fourier modes,
$1$--$4$ Gaussian bumps, and low-mode ($1$--$3$) Fourier sums. This is
designed so the backbone sees both smooth multiscale data and localized
features. The Allen--Cahn ICs are random piecewise-tanh profiles with
three transitions and interface width $0.05$. Diffusion--sorption uses smooth left-driven
exponentials $c_{\mathrm{left}}\exp(-x/\lambda)$ with $\lambda$ randomized
around the design value $0.05L$. The 2D ICs are detailed in their
respective generators.

\subsection{Data generation}
\label[appendix]{app:datagen}

All structured-grid datasets are generated by a single script
(\texttt{generate\_data.py}) configured per PDE through Hydra~\cite{Yadan2019Hydra}: a top-level
config selects the solver, the IC generator, the spatial grid, and the
time grid (\cref{app:datasets}). Generation proceeds sample-by-sample
and is embarrassingly parallel--each sample is one IC draw followed by
one solver call--so we dispatch all $n=5\,000$ samples across CPU
workers via \texttt{joblib} with the \texttt{loky} backend. Reproducibility
is independent of worker count: we seed a single
\texttt{SeedSequence(42)} and \texttt{spawn} $n$ child sequences from it,
one per sample, so a sample with index $i$ gets the same RNG state
regardless of which worker picks it up. After all workers
finish we sort the results by index and stack them into the dataset
arrays.

The 1D outputs are saved as $u_0$ of shape $(n, 1024)$ together with the
solution $u(\cdot, t)$ of shape $(n, n_{\mathrm{snap}}, 1024)$, and the 2D
and 3D outputs add one or two spatial axes (and a leading channel axis for
Gray--Scott). Darcy is the exception: it is static and stored as
permeability and pressure pairs of shape $(n, 128, 128)$ each.

The MeshGraphNets datasets are not generated by us. We reuse the original
trajectories from~\cite{pfaff_learning_2021}, repackaged as
per-trajectory \texttt{.npz} files under the same train/val/test split as
the original release.

\subsection{Models}
\label[appendix]{app:models}

\paragraph{Backbones.} For every dataset we use one of four backbones,
each chosen so that its inductive bias matches the data, and MWT-Gupta-1D for the analytic validation. The 1D PDEs use
a standard 1D Fourier neural operator (\textbf{FNO1D}: hidden width $64$,
$4$ layers, $16$ retained Fourier modes, domain padding $0.5$) and, as a
non-FNO comparison architecture, the multiwavelet-transform operator of
Gupta et al.~\cite{gupta_multiwavelet-based_2021}
(\textbf{MWT-Gupta-1D}: Legendre basis, channel width $c=16$, wavelet
order $k=4$, $\alpha=10$, $n_{CZ}=2$, $L=0$, with the input $u_0$
concatenated with the spatial coordinate as a second channel). The
structured 2D PDEs use the local-integral variant of
\citet{liu-schiaffini_neural_2024}
(\textbf{LocalFNO2D}: a 2D FNO with hidden width $64$, $4$ layers, $16\times 16$ modes, and padding $0.5$, plus a DISCO local-convolution
branch sized for $\texttt{default\_in\_shape}=[128, 128]$). The 3D heat
equation uses a 3D FNO (\textbf{FNO3D}: hidden $32$, $4$ layers,
$8\times 8\times 8$ modes, padding $0.5$).

The unstructured-mesh datasets use the geometry-informed neural operator
(\textbf{GINO}) of Li et al.~\cite{li_geometry-informed_2023-1}: GNO
encoder and decoder with neighbor-search radius $0.05$ ($0.12$ for the airfoil encoder, in the
normalized $[0,1]^2$ frame) and a nonlinear kernel transform, around a
$16\times 16$-mode FNO at hidden width $64$ on a $64\times 64$ regular
latent grid. For the MGN datasets the only change is the I/O width:
$2$ input channels and $2$ output channels for the velocity field
$(v_x, v_y)$.

\paragraph{Explainers.} Every explainer reuses the architecture of its
backbone--FNO backbone with FNO explainer, GINO with GINO, and so on--with three changes. First, the output channel count is always set to
$1$, since the explainer returns a single scalar attribution density
$\phi(x)$ regardless of the I/O width of the backbone (the GINO-MGN
explainer therefore drops from $2$ output channels to $1$). Second, on
the LocalFNO2D explainer the DISCO $\texttt{default\_in\_shape}$ is set
to the largest training resolution served by the experiment ($[32, 32]$
in the runs we report) so the local kernels match the receptive field
actually seen during the FastSHAP loop. Third, the explainer gets one additional input channel that encodes the query point. The explainer therefore matches the backbone in parameter budget up to
the added input channel and the dropped output channels. The query channel is a Gaussian bump of width $0.05$ of the domain centered at $x^\star$, independent of the grid. During training, $x^\star$ is drawn uniformly over the
grid for every sample in every batch, and the masked backbone output is read at that point.

\subsection{Training}

We use ``Weights and Biases'' \cite{biewaldExperimentTrackingWeights2020} for our experiment tracking.
\label[appendix]{app:training}

\paragraph{Backbones.} All backbones are trained with the same recipe:
\texttt{AdamW} (learning rate $10^{-3}$, weight decay $10^{-4}$), cosine
learning-rate schedule, MSE loss against the held-out solution at the
target snapshot, $200$ epochs with early stopping on validation MSE
(patience $50$, no minimum-delta requirement), and the best-val-loss
checkpoint kept. The structured-grid backbones are trained on $16$ to $128$ points per axis ($8$ and $16$ in 3D), so explanations on
coarser and finer grids evaluate the backbone outside its training resolutions. The batch size is $64$ for the structured-grid
backbones (FNO1D, LocalFNO2D, FNO3D, MWT-Gupta-1D) and $1$ for the GINO
backbones, where the per-sample mesh forces single-mesh forwards
(\texttt{neuralop.GINO} only batches across samples that share a
geometry).

\paragraph{Explainer training.} The OperatorSHAP loop of
\cref{alg:operatorshap} is trained with \texttt{AdamW}
(learning rate $3\!\cdot\!10^{-4}$, weight decay $10^{-5}$), cosine
decay without warmup, and the FastSHAP
weighted-MSE loss already shown in \cref{alg:operatorshap}. We
run for at most $300$ epochs with patience $50$ and batch size $64$. At every
step we draw $K=128$ coalitions $S\subseteq P$ from the Shapley kernel
weighting $w(s)\propto 1/(s(n-s))$ in paired form (each $S$ is drawn
together with its complement to reduce gradient variance, so that a step uses $2K=256$ coalitions), evaluate the
backbone on the masked inputs $u_S$ to obtain the targets $v(S)$, and
update the explainer against the FastSHAP weighted MSE. The coalitions
are binary masks at the level of the explainer's output grid
($n_{\mathrm{players}}$ equals the spatial resolution of the current
batch), efficiency is enforced by hard normalization, and the query
point $x^\star$ is uniformly sampled on the domain for every sample and given as a Gaussian bump in an additional input channel to the model.
Because the query point changes, $v(\Omega)$ and $v(\emptyset)$ are recomputed every batch. We use the zero baseline for all runs except
the MGN data and the baseline ablation of \cref{app:baselines}. For the MGN data we use the global mean for the cylinder flow and the per-node
mean for the airfoil flow.

\paragraph{MGN-specific deltas.} The MGN explainer training uses the
same optimizer settings as the structured-grid case but differs in a
handful of places dictated by the mesh setting. Batch size drops to $1$
(GINO constraint) and we run for at most $100$ epochs at patience $15$. The coalition partition is a fixed $8\!\times\!8$
grid in the normalized mesh frame ($n=64$ players, independent of mesh
density). Integrals over a mesh use the lumped-mass quadrature of its
triangulation. Every node carries one third of the area of each triangle
it belongs to. The coarsening $\hat\phi_i=\int_{B_i}\phi$ is the sum
of the explainer output times this weight over the nodes in box $B_i$,
and the measure $|B_i|$ in the normalization is the sum of the weights
in $B_i$. A box is therefore measured by the area it covers and not by
the number of nodes in it, so the refined boundary layer and the wake
carry no extra weight. The meshes have no nodes inside the cylinder or
the airfoil. A box that overlaps the obstacle contains only the nodes of
the flow region, and its measure is the flow area it covers, up to the
triangles that straddle its boundary. A box without nodes has measure
zero. Masking it changes no node and it receives no attribution. We
draw $K=32$ coalition pairs ($64$ coalitions) per step from the Shapley distribution. Coalition forwards are chunked at most $128$ (cylinder flow) and $32$ (airfoil) at a time through
GINO to keep the per-coalition activation in GPU memory. The coalition forwards run in bf16 autocast, with the neighborhood mean of the input GNO in fp32. In bf16, this mean over about $170$
neighbors moves the airfoil coalition values by $0.29$, against a spread of $0.07$ across coalitions. With the fp32 mean, the coalition values
of both meshes change by less than $0.003$ (spread $0.07$ on the airfoil and $0.17$ on the cylinder). Because the
randomly sampled initial time $t_0$ also changes the input function, $v(\Omega)$ cannot be
cached and is recomputed every batch. During training, the query point is a mesh node drawn per sample with probability proportional to its quadrature weight, that is,
uniformly over the flow area. For evaluation, the query
points are given in physical coordinates and resolved per sample to the nearest mesh node.

\subsection{Attribution methods}
\label[appendix]{app:attribution_methods}

\paragraph{Notation and conventions.} For every attribution method
below we partition the input domain $\Omega$ into $n$ cells -- the
\emph{player grid} $P$ -- and define the discrete coalition game
\(
  v(S) = \Fhat(u_S)(x^\star) - \Fhat(\bar u)(x^\star),
\qquad S\subseteq P,
\)
where $u_S = \bar u + \mathbf 1_S\odot(u_0-\bar u)$ is the
input with cells in $S$ revealed and the rest held at the baseline
$\bar u$. In the decision game of \cref{app:midpoint}, the readout $\Fhat(u_S)(x^\star)$ is replaced by the
midpoint indicator, and every method below, IG included, uses that readout. The baseline defaults to zero for the ten structured-grid PDEs, to per-node mean for MGN airfoil, and to global mean for MGN cylinder flow. On the structured grids, every test input is explained at three query points $x^\star$, the domain center and two off-center points with coordinates at a quarter and three quarters of the domain. On the meshes, the three query points are fixed physical locations, resolved per sample to the nearest mesh node. All attribution runs use
seed~$42$. The OperatorSHAP and \textsc{KernelSHAP} / \textsc{Exact} /
\textsc{RegressionMSR} estimators all see the same coalition game, but they
differ only in how they explore the coalition space. \textsc{IG} is a
non-Shapley density baseline that does not see the game at all but
integrates gradients of $\Fhat$ along the line from
$\bar u$ to $u_0$.

\paragraph{Exact Shapley.} For $n\le 16$ we enumerate all $2^n$
coalitions and apply the closed-form Shapley sum of
\eqref{eq:shapley_def} via \texttt{shapiq.ExactComputer}~\cite{fumagalli_shap-iq_2023}.
This gives a deterministic ground-truth attribution
($65{,}536$ coalitions at $n=16$), which we use as the zero-error
reference at the smallest player count in our budget-curve plots.

\paragraph{KernelSHAP.} The Lundberg--Lee weighted-least-squares
estimator~\cite{lundberg_unified_2017}, run through
\texttt{shapiq.KernelSHAP}~\cite{fumagalli_shap-iq_2023}. We use it as a
baseline at budgets from $64$ to $8192$ to track convergence to the exact
solution (at $n=16$) and to the higher-budget RegressionMSR reference
(at $n\in\{512, 1024\}$) in \cref{fig:burgers1d_budget_curves_error}, and at budgets from $128$ to $32768$
at $n=256$ in \cref{fig:amortization,fig:matched_cost}.

\paragraph{RegressionMSR.} The regression-adjusted Monte Carlo
estimator of Witter et al.~\cite{witter_regression-adjusted_2025},
also run through shapiq~\cite{fumagalli_shap-iq_2023}. RegressionMSR is a more modern variant of KernelSHAP that often performs better at higher player numbers than KernelSHAP. It is therefore our default discrete reference with a budget of $300{,}000$, where the exact $2^n$ enumeration is infeasible.

\paragraph{OperatorSHAP / FastSHAP.} The amortized explainer of
\cref{sec:method}. The training procedure is documented in
\cref{app:training}. At inference time, a single forward pass of
$\Phi_\theta(u_0, x^\star)$ produces the full-resolution density $\phi$, and the
discrete Shapley value over a partition cell $B_i$ is recovered by
integrating $\phi$ over $B_i$. On a regular grid every node carries
the cell volume, and on a mesh it carries the lumped-mass weight of
\cref{app:training}. Hard normalization enforces
$\sum_i \hat\phi_i = v(\Omega) - v(\emptyset)$ exactly.

\paragraph{Integrated gradients (IG).} A non-Shapley
density baseline that bypasses the coalition game entirely. Following
the construction of \cref{sec:ig}, we compute the IG
density via a $50$-step midpoint Riemann sum of the directional
derivative $D \Fhat_{x^\star}\!\bigl(\bar u + th\bigr)[h]$ along
the straight line $\bar u\to u_0$ with the same baseline $\bar u$ as the game, with the gradient at each
step obtained by a backward pass through the backbone. Any region attribution is then a cheap weighted sum of this density
against a smooth mask.
Recent PyTorch versions before 2.11 return wrong gradients through complex \texttt{torch.einsum} on the Apple MPS backend (PyTorch issue \#177474), which the Fourier layers of our backbones rely on.
This bug affected the IG results in an earlier version of this paper and users should be aware to use an up to date PyTorch version or run the methods on CPU or CUDA. The IG results reported in this version of the paper do not use MPS anymore and are correct.

\subsection{Evaluation metrics}
\label[appendix]{app:metrics}

For every metric below we reduce along the player or grid dimension to
one scalar per test sample, and report the median and inter-quartile ranges (IQR) 
over a held-out set of $30$ test inputs ($10$ for the MGN meshes) and three query points per
PDE. We use the following metrics.

\paragraph{Pearson correlation.} The standard sample Pearson
coefficient between the predicted and reference Shapley vectors at the
reference's partition,
\(
  r(\hat\varphi, \varphi^{\mathrm{ref}})
  = \frac{(\hat\varphi-\bar{\hat\varphi})^\top(\varphi^{\mathrm{ref}}-\bar{\varphi}^{\mathrm{ref}})}
         {\|\hat\varphi-\bar{\hat\varphi}\|\,\|\varphi^{\mathrm{ref}}-\bar{\varphi}^{\mathrm{ref}}\|},
\)
computed per sample along the player dimension. Pearson is invariant to
constant shifts and global rescalings, so it captures the
\emph{shape} of the attribution profile rather than its magnitude.

\paragraph{Normalized RMSE.} The relative $\ell_2$ error along the
player dimension,
\(
  \mathrm{NRMSE}(\hat\varphi, \varphi^{\mathrm{ref}})
  = \|\hat\varphi - \varphi^{\mathrm{ref}}\|_2 \,/\, \|\varphi^{\mathrm{ref}}\|_2,
\)
which complements Pearson by rewarding the correct attribution
\emph{magnitude}.

\paragraph{Faithfulness (weighted $R^2$).} Following the
faithfulness diagnostic of~\citet{DBLP:journals/jmlr/TsaiYR23}, we sample $10{,}000$ random coalitions
$S\subseteq P$ from the same Shapley-kernel distribution used during
training and report the coefficient of determination of the predicted
linear surrogate $\hat\varphi^\top \mathbf 1_S$ against the true
coalition gap $v(S) - v(\emptyset)$,
\[
R^2 \;=\; 1 - \frac{\sum_S\!\bigl(v(S)-v(\emptyset)-\hat\varphi^\top\mathbf 1_S\bigr)^2}
                    {\sum_S\!\bigl(v(S)-v(\emptyset)-\overline{(v(S)-v(\emptyset))}\bigr)^2}.
\]
This measures how well the explainer's linear surrogate generalizes to
coalitions it never saw at training time, and unlike Pearson and NRMSE
it does not need a reference attribution. Since a binary coalition
on the grid is the sample of a smooth coalition (\cref{app:pou_construction}), this score is also the faithfulness of the explainer with respect
to the smooth game of \cref{def:smooth_op_game}.

\section{Additional results}
\label[appendix]{app:full_results}

\cref{tab:cross_pde_metrics_full_top} and \cref{tab:cross_pde_metrics_full_bot} reproduce \cref{tab:cross_pde_metrics_compact}
for all PDEs of our suite, not just those shown in the main text. The same metrics, references, budgets, and reporting convention apply.

\begin{figure}[t]
  \centering
  \includegraphics[width=\textwidth]{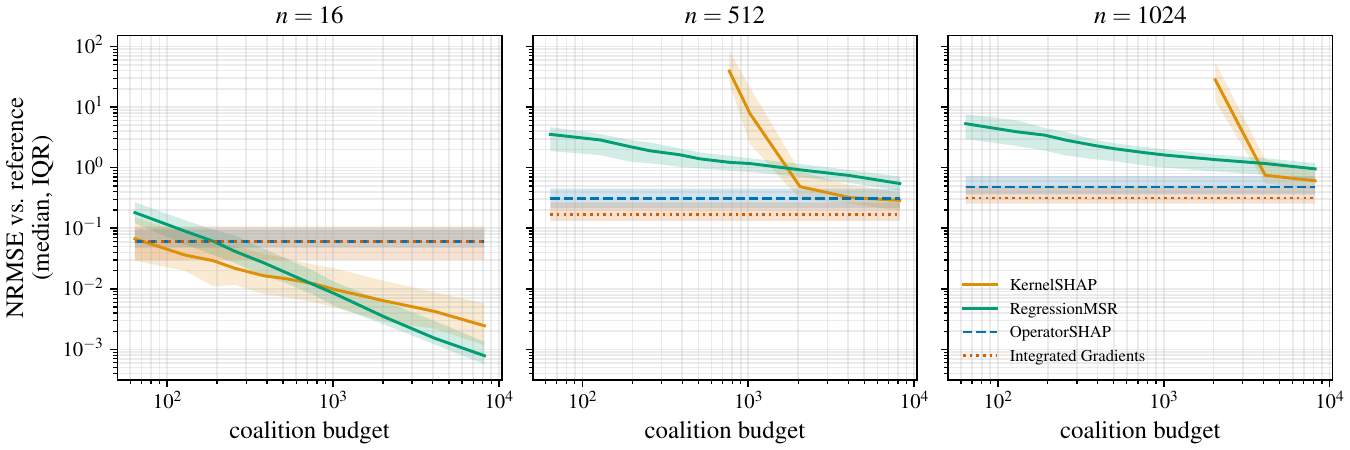}
  \caption{Normalized RMSE (NRMSE) of attribution methods vs.\ ground truth (for $n=16$: Exact Shapley, otherwise RegressionMSR with 300k budget) for different player numbers (spatial resolution) 16 (\emph{left}), 512 (\emph{middle}), 1024 (\emph{right}). The PDE is Burgers 1D. The lines show the median with $0.25$- and $0.75$-quantile bands over $30$ test inputs and three query points, and the horizontal lines mark OperatorSHAP and Integrated Gradients. The same OperatorSHAP model was used at all resolutions. We can see that for low player numbers, classical Shapley approximators outperform OperatorSHAP already for low budgets, while it takes significantly longer for higher player numbers.}
  \label{fig:burgers1d_budget_curves_error}
\end{figure}

\begin{figure}[t]
  \centering
  \includegraphics[width=\textwidth]{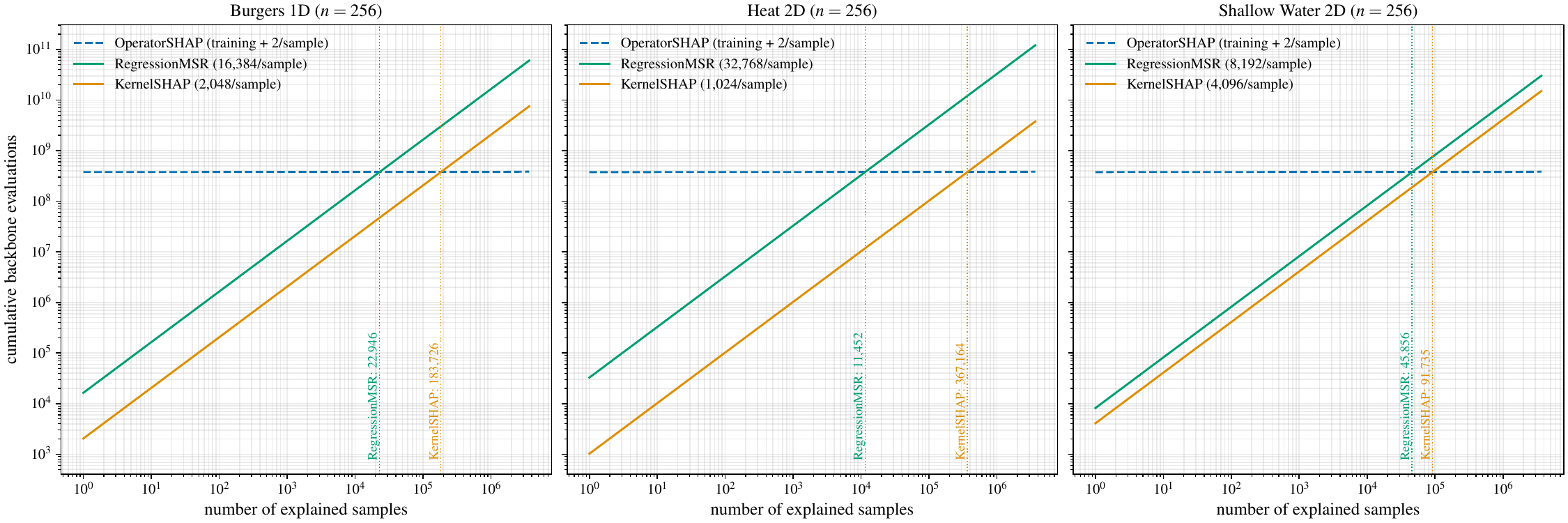}
  \caption{Amortization visualization: We depict the total number of necessary backbone evaluations for inference of increasing sample sizes. OperatorSHAP includes training evaluations, but stays almost constant with increasing sample size, while classical methods grow linearly with the number of samples. We tuned the budget of classical methods to reach a similar error as OperatorSHAP at $n=256$. We can see that numbers depend on the PDE.}
  \label{fig:amortization}
\end{figure}

\begin{table}[h]
  \centering
  \caption{Full-suite version of \cref{tab:cross_pde_metrics_compact}: Cross-PDE attribution quality across resolutions $n$ on the 1D PDEs (other half in \cref{tab:cross_pde_metrics_full_bot}). Each (PDE, $n$, method) entry
    reports the median over $30$ test inputs and three query points per input on the top line and
    the inter-quartile range in brackets below. Metrics are
    Pearson correlation and NRMSE against the per-PDE reference attribution,
    and the weighted-$R^2$ Faithfulness score. RegressionMSR and KernelSHAP
    use a budget of $2n$ (minimum $1024$) backbone evaluations per query point, and IG is calculated via a 50-step Riemann sum. Across resolutions OperatorSHAP uses the same model per PDE. Higher Pearson and $R^2$ are better, and lower NRMSE
    is better. Bold marks the best median and every median within one standard error of it over the test inputs, compared at three decimals.}
  \label{tab:cross_pde_metrics_full_top}
  \resizebox{\textwidth}{!}{%
    \input{./tables/all_pde_metrics_full_top.tex}  }
\end{table}

\begin{table}[h]
  \centering
  \caption{Full-suite version of \cref{tab:cross_pde_metrics_compact}: Cross-PDE attribution quality across resolutions $n$ on the 2D and 3D PDEs and the MGN meshes (other half in \cref{tab:cross_pde_metrics_full_top}). Each (PDE, $n$, method) entry
    reports the median over $30$ test inputs and three query points per input on the top line and
    the inter-quartile range in brackets below. Metrics are
    Pearson correlation and NRMSE against the per-PDE reference attribution,
    and the weighted-$R^2$ Faithfulness score. RegressionMSR and KernelSHAP
    use a budget of $2n$ (minimum $1024$) backbone evaluations per query point, and IG is calculated via a 50-step Riemann sum. Across resolutions OperatorSHAP uses the same model per PDE. Higher Pearson and $R^2$ are better, and lower NRMSE
    is better. Heat 3D is evaluated at $n=512$ only, and the MGN meshes use $10$ test inputs. Bold marks the best median and every median within one standard error of it over the test inputs, compared at three decimals.}
  \label{tab:cross_pde_metrics_full_bot}
  \resizebox{\textwidth}{!}{%
    \input{./tables/all_pde_metrics_full_bot.tex}  }
\end{table}

\subsection{Runtime and faithfulness vs.\ budget}
\label[appendix]{app:matched_cost}
\cref{fig:matched_cost} scales the budget of the other methods with a multiplier $m$ at $n=256$: Integrated Gradients uses $m$ path steps, and
RegressionMSR and KernelSHAP use $128m$ coalitions. OperatorSHAP does not depend on $m$. On Burgers 1D and Heat 2D, IG reaches the faithfulness
of OperatorSHAP with one or two path steps, while RegressionMSR does not reach it within $m=256$. On Reaction--Diffusion 2D and Shallow Water 2D, IG
stays below OperatorSHAP for every $m$. RegressionMSR and KernelSHAP reach OperatorSHAP on Reaction--Diffusion 2D at $m=64$, and KernelSHAP on
Shallow Water 2D at $m=16$, at about $3600$ and $950$ times its runtime. RegressionMSR reaches it on Shallow Water 2D at $m=256$. With $512$ coalitions ($m=4$), the weighted least-squares problem of
KernelSHAP is close to singular, and its $R^2$ falls below the plotted range on three PDEs. Every point is a single run.

\begin{figure}[t]
  \centering
  \includegraphics[width=\textwidth]{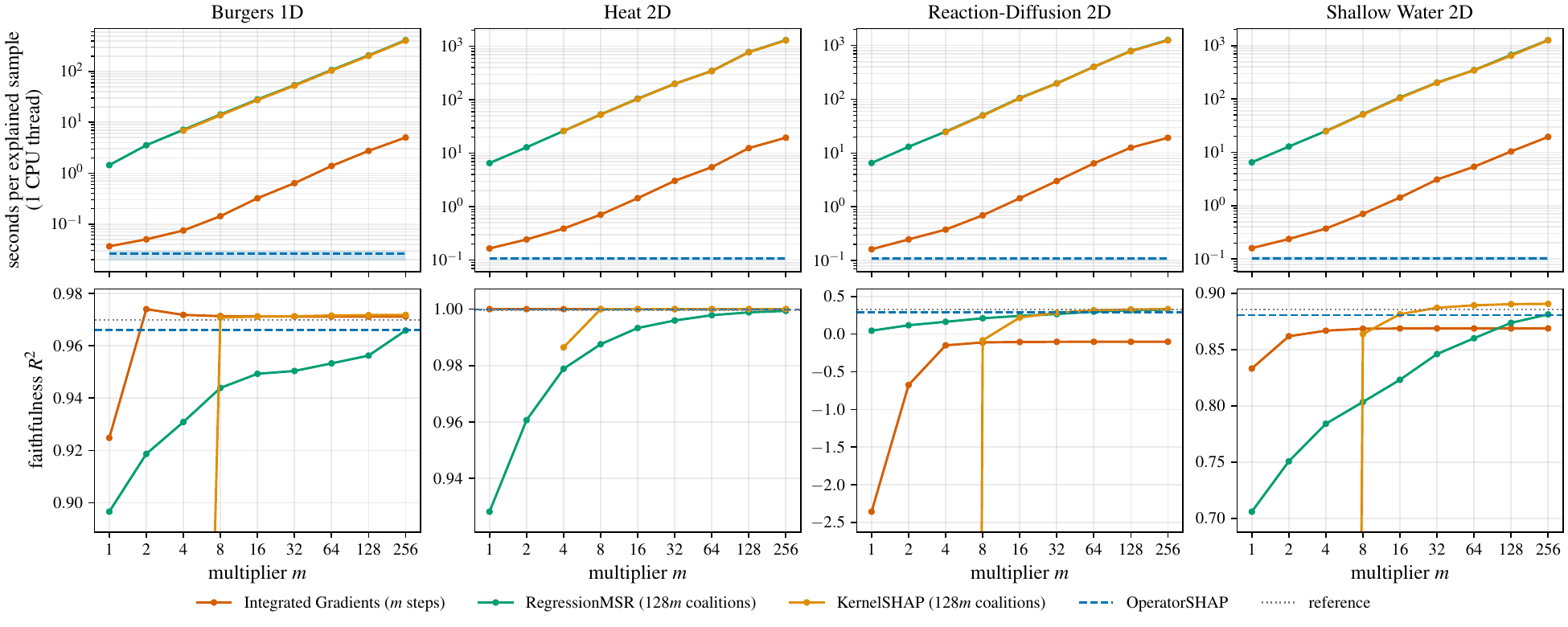}
  \caption{Runtime and faithfulness on four PDEs at $n=256$. \emph{Top:} wall-clock seconds per explained sample on one thread of an
  Intel Xeon Platinum 8380, median and inter-quartile range over $10$ validation inputs with three query points each. \emph{Bottom:} median
  faithfulness $R^2$ over $30$ test inputs and three query points. Integrated Gradients uses $m$ path steps, RegressionMSR and KernelSHAP use $128m$
  coalitions. The dashed line marks OperatorSHAP and the dotted line the reference. KernelSHAP is shown for budgets above $n$, from $m=4$ on. OperatorSHAP is the fastest and needs no tuneable budget. IG sometimes reaches OperatorSHAP at comparable speed and both RegressionMSR and KernelSHAP reach OperatorSHAP at high enough budget, though at significantly lower speed.}
  \label{fig:matched_cost}
\end{figure}

\subsection{Held-out resolutions}
\label[appendix]{app:holdout}
The explainers of the main experiments see every evaluated resolution during training. \cref{tab:holdout} repeats the evaluation with explainers
whose training resolutions leave out one resolution. For $n=256$, the left-out resolution lies inside the training range ($n=128$ and $n=512$ in
1D, $8\times8$ and $32\times32$ in 2D). For $n=1024$, it is the finest resolution, so the explainer has to extrapolate, and on Heat 2D its DISCO branch is sized for
$16\times16$ instead of $32\times32$. Everything else, including
the backbone, the references, and the training recipe, is the same as for the full explainers. On Burgers 1D, both held-out explainers reach the
quality of the full explainer. On Heat 2D, the loss is small at $n=256$ and large at $n=1024$.

\begin{table}[h]
  \centering
  \caption{OperatorSHAP trained on all resolutions (full ladder) or with $n=256$ or $n=1024$ left out, evaluated at the left-out resolution.
  Medians over $30$ test inputs and three query points with the inter-quartile range below. The classical methods and IG do not depend on the explainer and are repeated for
  reference. Bold marks the best median and every median within one standard error of it over the test inputs, compared at three decimals.}
  \label{tab:holdout}
  \resizebox{\textwidth}{!}{%
    \input{./tables/holdout.tex}  }
\end{table}

\subsection{Analytic validation}
\label[appendix]{app:analytic}
\cref{fig:analytic} compares the explainers with the closed-form Aumann--Shapley densities of Heat 1D and Advection 1D from paragraph (iv) of
\cref{sec:experiments}.

\begin{figure}[t]
  \centering
  \includegraphics[width=\textwidth]{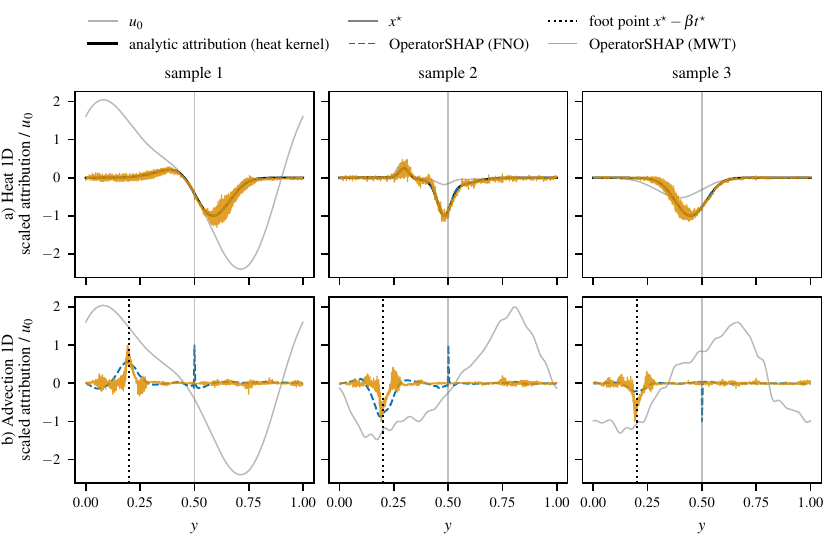}
  \caption{Analytic validation on three test samples. \emph{Top:} Heat 1D. The closed-form density $K_{t^\star}(x^\star,\cdot)\,u_0$ (black) and
  the FNO- and MWT-based explainers. \emph{Bottom:} Advection 1D. The analytic density is a point mass at the foot point $x^\star-\beta t^\star$
  (dotted). Both explainers place mass there, but the largest value of the FNO explainer is a spike at $x^\star$ (gray), which exposes the pointwise skip
  connections of the FNO backbone. In the top row, the attributions are divided by the maximum of the analytic density, and in the bottom row each explainer by its own
  maximum. $u_0$ is not scaled.}
  \label{fig:analytic}
\end{figure}

\subsection{Baseline choice}
\label[appendix]{app:baselines}
Choosing the correct baseline is already a difficult task in the classical Shapley literature. The baseline influences the interpretation of the attributions both by changing the reference point and by affecting the backbone model. The latter is particularly important in the context of physical data. Where a zero baseline in the heat equation just represents a state of no heat, in the Darcy flow, a zero permeability obstructs the flow of fluid through a masked cell. The choice of baseline should be determined on a case-by-case basis. We use the
following:
\begin{itemize}
  \item \textbf{Zero.} Masked cells are set to zero, so the attribution of a region is its contribution compared to the region being absent. This
    is the easiest to interpret and our default. However, the masked input can be physically implausible: in Darcy flow, zero permeability makes a masked
    cell impermeable, and zero velocity next to a wall does not match the boundary condition seen in training.
  \item \textbf{Per-sample mean.} Masked cells are set to the spatial mean of the sample. The attribution then measures the effect of the deviation
    from this mean. The fully masked input is the constant mean and keeps the total mass of the sample, which the zero baseline does not. For a
    partial coalition the mass changes by the integral of $u_0-\bar u$ over the coalition. For fields with a spatial mean close to zero (heat,
    Burgers) this baseline is almost the same as the zero baseline.
  \item \textbf{Global and per-node mean.} Masked cells are set to the mean of the training set, either one value for the whole field or one value
    per node for meshes with a fixed node set. The masked input is then a typical field that satisfies the boundary conditions and the geometry.
    We use these for the MGN datasets.
  \item \textbf{Conditional expectation.} Masked cells could also be set to their expected value given the revealed cells. This needs a generative
    model of the inputs, and we do not use it.
\end{itemize}
Guidelines: (1) Report the baseline together with the attribution, because different baselines answer different questions and the resulting maps
need not be correlated. (2) Use a mean baseline for coefficient fields whose zero value is not admissible (permeability, density) and near
boundaries where a zero value violates the boundary condition.

\cref{tab:mean_baseline} compares the zero baseline with the per-sample mean on Darcy 2D and diffusion--sorption 1D, each with its own explainer
and reference. We can see that the choice of baseline does not influence the reported faithfulness meaningfully. \cref{fig:baseline_ablation} compares the attributions of one Darcy 2D sample under both baselines, which shows grave differences that a user would need to reflect in their interpretation.

\begin{table}[h]
  \centering
  \caption{Zero and per-sample mean baseline on Darcy 2D and diffusion--sorption 1D. Each baseline has its own explainer and reference.
  Medians over $30$ test inputs and three query points with the inter-quartile range below. Bold marks the best median and every median within one standard error of it over the test inputs, compared at three decimals.}
  \label{tab:mean_baseline}
  \resizebox{\textwidth}{!}{%
    \input{./tables/mean_baseline.tex}  }
\end{table}

\begin{figure}[h]
  \centering
  \includegraphics[width=\textwidth]{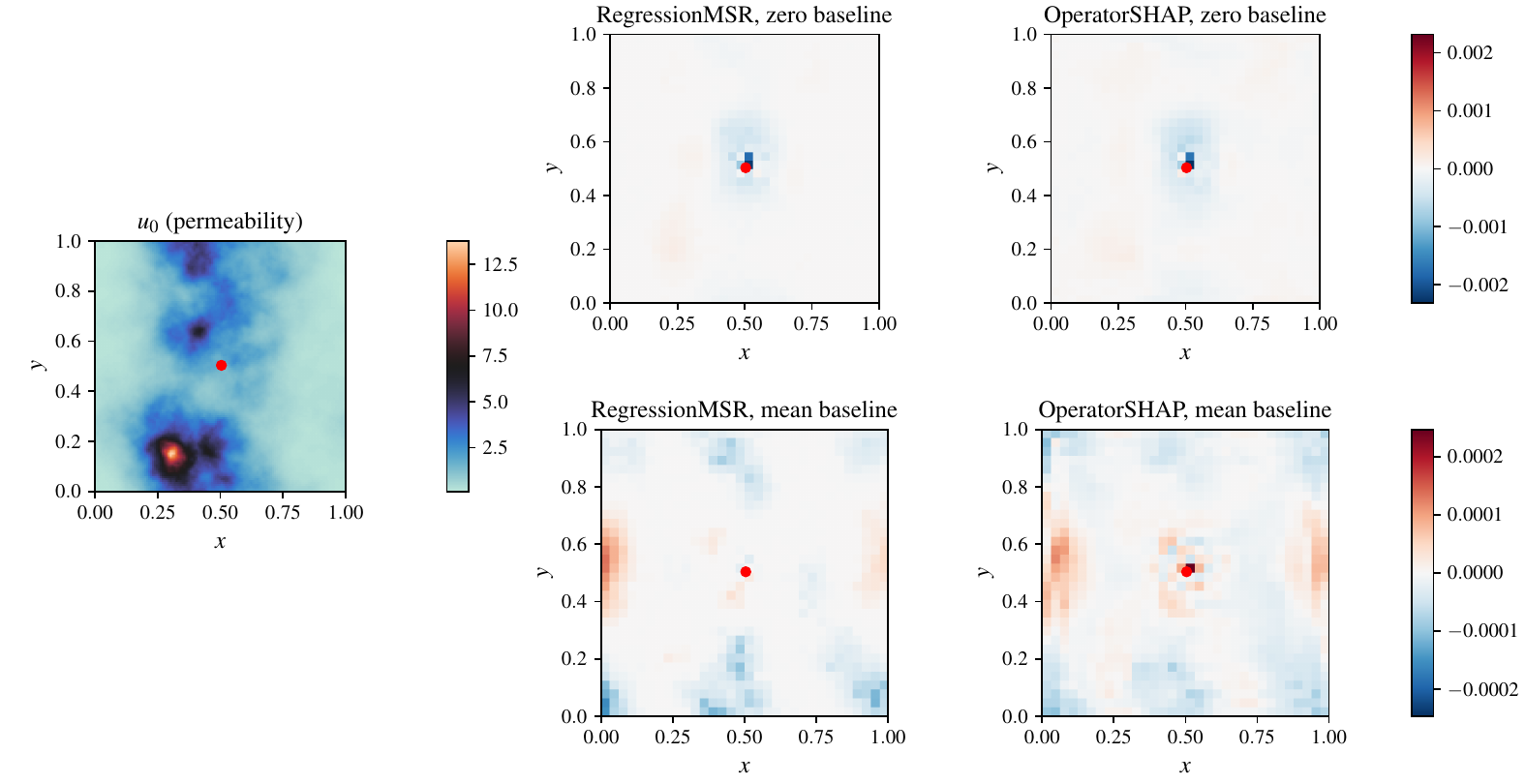}
  \caption{Baseline ablation on one Darcy 2D sample at $n=1024$ and the domain center $x^\star$ (red): the permeability field, and the reference
  (RegressionMSR with $3\!\cdot\!10^5$ backbone evaluations) and OperatorSHAP attributions with the zero baseline (\emph{top}) and the per-sample
  mean baseline (\emph{bottom}). With the zero baseline the contributions concentrate in the cells around $x^\star$ that gate the flow into it.
  With the mean baseline the whole field contributes to the game, as there now exists non-zero permeability. The two references are uncorrelated, with median Pearson $\res{baseline/darcy2d/1024/reference_pearson}$ on Darcy 2D and $\res{baseline/diffsorp1d/1024/reference_pearson}$ on diffusion--sorption 1D at
  $n=1024$.}
  \label{fig:baseline_ablation}
\end{figure}

\subsection{ReLU backbone}
\label[appendix]{app:relu}
Proposition \ref{prop:fno_frechet} requires a smooth activation. \cref{tab:relu} repeats the Burgers 1D experiment with an FNO backbone in which every GELU is
replaced by a ReLU, with the same architecture, training recipe, and its own reference. The explainer reaches a similar quality as for the GELU
backbone (Pearson $\res{burgers1d/fnoRelu/64/OperatorSHAP/pearson}$ and $\res{burgers1d/fnoRelu/256/OperatorSHAP/pearson}$ against $\res{burgers1d/fno/64/OperatorSHAP/pearson}$ and $\res{burgers1d/fno/256/OperatorSHAP/pearson}$ at $n=64$ and $256$), so the smoothness assumption is a requirement of the differentiability proof and not of the algorithm.

\begin{table}[h]
  \centering
  \caption{Burgers 1D with a ReLU and a GELU backbone. Medians over $30$ test inputs and three query points with the inter-quartile range below, with the same methods and budgets as
  \cref{tab:cross_pde_metrics_compact}. Bold marks the best median and every median within one standard error of it over the test inputs, compared at three decimals.}
  \label{tab:relu}
  \resizebox{\textwidth}{!}{%
    \input{./tables/relu.tex}  }
\end{table}

\subsection{Non-differentiable readout}
\label[appendix]{app:midpoint}
In decision problems, the precise output value matters less than whether the prediction at $x^\star$ reaches a given outcome. We model this with
the midpoint readout. With the masked input $u_S=\bar u+\mathbf 1_S\odot(u_0-\bar u)$, the value of a coalition $S$ is
\[
v(S)=\mathbf 1\Big[\,\big|\Fhat_{x^\star}(u_S)-\Fhat_{x^\star}(u_0)\big|<\big|\Fhat_{x^\star}(u_S)-\Fhat_{x^\star}(\bar u)\big|\,\Big],
\]
so $v(S)=1$ if the prediction for the masked input is closer to the prediction for $u_0$ than to the prediction for the baseline. Then
$v(\emptyset)=0$, and $v(\Omega)=1$ whenever $\Fhat_{x^\star}(u_0)\neq\Fhat_{x^\star}(\bar u)$. By efficiency, the Shapley values sum to one and give
the share of each region in moving the prediction past the midpoint between the two values.

The game is piecewise constant in the coalition weights. Along the straight path from $\bar u$ to $u_0$ its gradient vanishes almost everywhere,
so IG returns zero attributions for every input, with $R^2$ between $\res{sw2dmid/localFNO/16/IG/r2}$ and $\res{rxndiff2dmid/localFNO/16/IG/r2}$, an undefined Pearson correlation, and NRMSE $1$ in \cref{tab:midpoint}. The
game is of the type in \Cref{rem:bv_games}. It has a value, but we do not prove that the value is asymptotic, so we compare the explainer with a
reference at every resolution instead.

We use the backbones and data of the nine 1D and 2D PDEs, Heat 3D, and the MGN meshes, train a new explainer per PDE on this game, and compute new references. The learning rate and batch size
are chosen per PDE by validation loss: $10^{-3}$ and $64$, except $3\cdot10^{-4}$ and $32$ for Advection 1D. The MGN meshes use the training setup of the mesh explainers (\cref{app:training}).
From $n=64$ on, OperatorSHAP has the highest faithfulness or one
within one standard error of it, except on Advection 1D at $n=1024$ and MGN Airfoil. On MGN Cylinder ($n=64$), RegressionMSR has a marginally higher Pearson correlation
($\res{mgn_cylmid/gino/64/RMSR/pearson}$ against $\res{mgn_cylmid/gino/64/OperatorSHAP/pearson}$), and the faithfulness of RegressionMSR, KernelSHAP, and OperatorSHAP lies within one standard error. On MGN Airfoil ($n=64$), RegressionMSR
and KernelSHAP have a higher Pearson correlation ($\res{mgn_airmid/gino/64/RMSR/pearson}$ and $\res{mgn_airmid/gino/64/KernelSHAP/pearson}$ against
$\res{mgn_airmid/gino/64/OperatorSHAP/pearson}$) and a higher faithfulness ($\res{mgn_airmid/gino/64/RMSR/r2}$ and $\res{mgn_airmid/gino/64/KernelSHAP/r2}$ against
$\res{mgn_airmid/gino/64/OperatorSHAP/r2}$).

\begin{table}[h]
  \centering
  \caption{Game with the non-differentiable readout of \cref{app:midpoint} on the nine 1D and 2D PDEs, Heat 3D, and the MGN meshes. The reference is RegressionMSR over all $2^{16}$ coalitions at
  $n=16$ and RegressionMSR with $3\!\cdot\!10^5$ backbone evaluations otherwise ($5\!\cdot\!10^4$ on MGN Airfoil). Medians over $30$ test inputs and three query points with the
  inter-quartile range below, with the
  same methods and budgets as \cref{tab:cross_pde_metrics_compact}. The MGN meshes use $10$ test inputs. On the decision readout, the gradient of the game vanishes, so IG attributions are identically zero and their Pearson correlation is undefined (--).
  Bold marks the best median and every median within one standard error of it over the test inputs, compared at three decimals.}
  \label{tab:midpoint}
  \resizebox{\textwidth}{!}{%
    \input{./tables/midpoint.tex}  }
\end{table}

\subsection{Structural stability across resolutions}
\label[appendix]{app:ssim}
Pearson and NRMSE compare attributions at one resolution. To test whether the spatial structure of a map is preserved when the same sample is served
at two resolutions, we compute the structural similarity index (SSIM) between the two maps. Since the discrete values at every partition are
integrals of one density, the finer map is average-pooled to the coarser grid, both maps are nearest-upsampled to a common $32\times32$ canvas,
standardized, and compared with SSIM ($7\times7$ Gaussian window). \cref{tab:ssim} reports the median over $30$ test inputs on Heat 2D and Shallow Water 2D for
OperatorSHAP, the reference with $3\!\cdot\!10^5$ backbone evaluations, and RegressionMSR at budget $2n$ (at least $1024$). OperatorSHAP preserves structure at least as well as the reference, though this is dependent on the granularity and jump between discretizations. If the jump between the compared granularities is too large, the finer granularity has detail that the coarser one lost.

\begin{table}[h]
  \centering
  \caption{SSIM between attribution maps of the same sample at two resolutions at the domain center. Median over $30$ test inputs, inter-quartile range below, and $1$ is perfect structural agreement.}
  \label{tab:ssim}
  \resizebox{\textwidth}{!}{%
    \input{./tables/ssim.tex}  }
\end{table}

\subsection{Backbone accuracy}
\label[appendix]{app:backbone_error}
\cref{tab:backbone_error} lists the relative $L^2$ test error of every grid backbone at every resolution at which it is explained. The
columns give the grid points per axis $n_{\mathrm{1D}}$, so the player count $n$ of the other tables is $n_{\mathrm{1D}}^d$. OperatorSHAP explains
the backbone it is given, so these numbers do not influence the quality of the attribution. Explaining an imperfect backbone may be useful if the developer wants to understand the backbone's behavior. We include the numbers for completeness.

\begin{table}[h]
  \centering
  \caption{Relative $L^2$ test error of the backbones per grid points per axis $n_{\mathrm{1D}}$, median over $30$ test inputs.}
  \label{tab:backbone_error}
  \small
  \input{./tables/backbone_error.tex}
\end{table}

\section{Computing resources}
\label[appendix]{app:compute}

The runs reported here were carried out on four academic clusters, depending on availability: NVIDIA A100 (80 GB), GH200, A40,
and H100 GPUs with one GPU per job, and CPU nodes with two Intel Xeon Platinum 8380. Wall-clock times therefore depend on the node.
The runtimes of \cref{fig:matched_cost} are measured on one thread of an Intel Xeon Platinum 8380 CPU. Medians over the reported
runs: backbone training takes $0.1$\,h (1D), $0.4$\,h (2D), $0.3$\,h (3D), and $2.8$\,h (meshes) on an A100. OperatorSHAP training
takes $3$ to $5$\,h (1D), $6$ to $8$\,h (2D), and $3$ to $9$\,h (meshes). The reference runs with $3\!\cdot\!10^5$ backbone
evaluations take about $1$\,h (1D), $1$ to $2$\,h (2D), $3$\,h (3D), and $8.5$\,h (meshes) on one GPU, and RegressionMSR and
KernelSHAP at budget $2n$ (at least $1024$) under $0.5$\,h on one GPU or eight CPU cores. The metrics of one PDE and player
number take under a minute on a GPU for the grids and up to $1.4$\,h for the meshes. The data-generation pipeline
(\cref{app:datagen}) is CPU-only and runs in under an hour per PDE on
a single $96$-core node thanks to the per-sample joblib parallelism. The full project required more compute, due to failed
experiments, parameter tuning, or experiments that did not make it into the final paper.

\section{Third-party assets and licenses}
\label[appendix]{app:assets}

Most of the code we build on--PyTorch, Hydra, joblib, scipy, the
\texttt{neuraloperator} package--is standard scientific Python tooling. The
two third-party assets that warrant explicit credit are the \texttt{shapiq}
library~\cite{fumagalli_shap-iq_2023}, which we use for all classical Shapley
estimators (\textsc{Exact}, \textsc{KernelSHAP}, \textsc{RegressionMSR}) and
which is released under the MIT license, and the MeshGraphNets airfoil and
cylinder-flow datasets~\cite{pfaff_learning_2021} produced and released
by Google DeepMind under the Apache~$2.0$ license, which we use unmodified
beyond the per-trajectory NPZ repackaging described in
\cref{app:datagen}. We use both within the bounds of their respective
licenses and redistribute neither.


\section{Use of large language models}
\label[appendix]{app:llm_usage}

Large language models were used throughout the writing and engineering of this
work: as a writing aid for editing and rephrasing parts of the manuscript, as
a coding partner during implementation of the data-generation, training, and
evaluation pipelines, and to discuss the mathematical arguments and the
experimental design. Every theoretical statement, derivation, and final
phrasing was checked by the authors, and every numerical result reported here
was produced by the released code.


\clearpage

\end{document}

%% file: tables/numbers.tex
\expandafter\def\csname res@ac1d/fno/1024/IG/nrmse\endcsname{0.375}
\expandafter\def\csname res@ac1d/fno/1024/IG/pearson\endcsname{0.953}
\expandafter\def\csname res@ac1d/fno/1024/IG/r2\endcsname{0.998}
\expandafter\def\csname res@ac1d/fno/1024/KernelSHAP/nrmse\endcsname{15.427}
\expandafter\def\csname res@ac1d/fno/1024/KernelSHAP/pearson\endcsname{0.071}
\expandafter\def\csname res@ac1d/fno/1024/KernelSHAP/r2\endcsname{0.108}
\expandafter\def\csname res@ac1d/fno/1024/OperatorSHAP/nrmse\endcsname{0.731}
\expandafter\def\csname res@ac1d/fno/1024/OperatorSHAP/pearson\endcsname{0.654}
\expandafter\def\csname res@ac1d/fno/1024/OperatorSHAP/r2\endcsname{0.997}
\expandafter\def\csname res@ac1d/fno/1024/RMSR/nrmse\endcsname{0.944}
\expandafter\def\csname res@ac1d/fno/1024/RMSR/pearson\endcsname{0.621}
\expandafter\def\csname res@ac1d/fno/1024/RMSR/r2\endcsname{0.995}
\expandafter\def\csname res@ac1d/fno/1024/Reference/nrmse\endcsname{0.000}
\expandafter\def\csname res@ac1d/fno/1024/Reference/pearson\endcsname{1.000}
\expandafter\def\csname res@ac1d/fno/1024/Reference/r2\endcsname{0.999}
\expandafter\def\csname res@ac1d/fno/128/IG/nrmse\endcsname{0.155}
\expandafter\def\csname res@ac1d/fno/128/IG/pearson\endcsname{0.987}
\expandafter\def\csname res@ac1d/fno/128/IG/r2\endcsname{0.997}
\expandafter\def\csname res@ac1d/fno/128/KernelSHAP/nrmse\endcsname{0.072}
\expandafter\def\csname res@ac1d/fno/128/KernelSHAP/pearson\endcsname{0.997}
\expandafter\def\csname res@ac1d/fno/128/KernelSHAP/r2\endcsname{0.998}
\expandafter\def\csname res@ac1d/fno/128/OperatorSHAP/nrmse\endcsname{0.090}
\expandafter\def\csname res@ac1d/fno/128/OperatorSHAP/pearson\endcsname{0.996}
\expandafter\def\csname res@ac1d/fno/128/OperatorSHAP/r2\endcsname{0.998}
\expandafter\def\csname res@ac1d/fno/128/RMSR/nrmse\endcsname{0.416}
\expandafter\def\csname res@ac1d/fno/128/RMSR/pearson\endcsname{0.893}
\expandafter\def\csname res@ac1d/fno/128/RMSR/r2\endcsname{0.994}
\expandafter\def\csname res@ac1d/fno/128/Reference/nrmse\endcsname{0.000}
\expandafter\def\csname res@ac1d/fno/128/Reference/pearson\endcsname{1.000}
\expandafter\def\csname res@ac1d/fno/128/Reference/r2\endcsname{0.998}
\expandafter\def\csname res@ac1d/fno/256/IG/nrmse\endcsname{0.220}
\expandafter\def\csname res@ac1d/fno/256/IG/pearson\endcsname{0.977}
\expandafter\def\csname res@ac1d/fno/256/IG/r2\endcsname{0.998}
\expandafter\def\csname res@ac1d/fno/256/KernelSHAP/nrmse\endcsname{0.191}
\expandafter\def\csname res@ac1d/fno/256/KernelSHAP/pearson\endcsname{0.980}
\expandafter\def\csname res@ac1d/fno/256/KernelSHAP/r2\endcsname{0.998}
\expandafter\def\csname res@ac1d/fno/256/OperatorSHAP/nrmse\endcsname{0.306}
\expandafter\def\csname res@ac1d/fno/256/OperatorSHAP/pearson\endcsname{0.952}
\expandafter\def\csname res@ac1d/fno/256/OperatorSHAP/r2\endcsname{0.997}
\expandafter\def\csname res@ac1d/fno/256/RMSR/nrmse\endcsname{0.718}
\expandafter\def\csname res@ac1d/fno/256/RMSR/pearson\endcsname{0.700}
\expandafter\def\csname res@ac1d/fno/256/RMSR/r2\endcsname{0.992}
\expandafter\def\csname res@ac1d/fno/256/Reference/nrmse\endcsname{0.000}
\expandafter\def\csname res@ac1d/fno/256/Reference/pearson\endcsname{1.000}
\expandafter\def\csname res@ac1d/fno/256/Reference/r2\endcsname{0.998}
\expandafter\def\csname res@ac1d/fno/32/IG/nrmse\endcsname{0.058}
\expandafter\def\csname res@ac1d/fno/32/IG/pearson\endcsname{0.998}
\expandafter\def\csname res@ac1d/fno/32/IG/r2\endcsname{0.997}
\expandafter\def\csname res@ac1d/fno/32/KernelSHAP/nrmse\endcsname{0.010}
\expandafter\def\csname res@ac1d/fno/32/KernelSHAP/pearson\endcsname{1.000}
\expandafter\def\csname res@ac1d/fno/32/KernelSHAP/r2\endcsname{0.997}
\expandafter\def\csname res@ac1d/fno/32/OperatorSHAP/nrmse\endcsname{0.036}
\expandafter\def\csname res@ac1d/fno/32/OperatorSHAP/pearson\endcsname{0.999}
\expandafter\def\csname res@ac1d/fno/32/OperatorSHAP/r2\endcsname{0.997}
\expandafter\def\csname res@ac1d/fno/32/RMSR/nrmse\endcsname{0.044}
\expandafter\def\csname res@ac1d/fno/32/RMSR/pearson\endcsname{0.999}
\expandafter\def\csname res@ac1d/fno/32/RMSR/r2\endcsname{0.997}
\expandafter\def\csname res@ac1d/fno/32/Reference/nrmse\endcsname{0.000}
\expandafter\def\csname res@ac1d/fno/32/Reference/pearson\endcsname{1.000}
\expandafter\def\csname res@ac1d/fno/32/Reference/r2\endcsname{0.997}
\expandafter\def\csname res@ac1d/fno/512/IG/nrmse\endcsname{0.292}
\expandafter\def\csname res@ac1d/fno/512/IG/pearson\endcsname{0.966}
\expandafter\def\csname res@ac1d/fno/512/IG/r2\endcsname{0.998}
\expandafter\def\csname res@ac1d/fno/512/KernelSHAP/nrmse\endcsname{4.344}
\expandafter\def\csname res@ac1d/fno/512/KernelSHAP/pearson\endcsname{0.225}
\expandafter\def\csname res@ac1d/fno/512/KernelSHAP/r2\endcsname{0.874}
\expandafter\def\csname res@ac1d/fno/512/OperatorSHAP/nrmse\endcsname{0.542}
\expandafter\def\csname res@ac1d/fno/512/OperatorSHAP/pearson\endcsname{0.827}
\expandafter\def\csname res@ac1d/fno/512/OperatorSHAP/r2\endcsname{0.997}
\expandafter\def\csname res@ac1d/fno/512/RMSR/nrmse\endcsname{1.020}
\expandafter\def\csname res@ac1d/fno/512/RMSR/pearson\endcsname{0.481}
\expandafter\def\csname res@ac1d/fno/512/RMSR/r2\endcsname{0.992}
\expandafter\def\csname res@ac1d/fno/512/Reference/nrmse\endcsname{0.000}
\expandafter\def\csname res@ac1d/fno/512/Reference/pearson\endcsname{1.000}
\expandafter\def\csname res@ac1d/fno/512/Reference/r2\endcsname{0.998}
\expandafter\def\csname res@ac1d/fno/64/IG/nrmse\endcsname{0.101}
\expandafter\def\csname res@ac1d/fno/64/IG/pearson\endcsname{0.994}
\expandafter\def\csname res@ac1d/fno/64/IG/r2\endcsname{0.997}
\expandafter\def\csname res@ac1d/fno/64/KernelSHAP/nrmse\endcsname{0.026}
\expandafter\def\csname res@ac1d/fno/64/KernelSHAP/pearson\endcsname{1.000}
\expandafter\def\csname res@ac1d/fno/64/KernelSHAP/r2\endcsname{0.998}
\expandafter\def\csname res@ac1d/fno/64/OperatorSHAP/nrmse\endcsname{0.077}
\expandafter\def\csname res@ac1d/fno/64/OperatorSHAP/pearson\endcsname{0.997}
\expandafter\def\csname res@ac1d/fno/64/OperatorSHAP/r2\endcsname{0.997}
\expandafter\def\csname res@ac1d/fno/64/RMSR/nrmse\endcsname{0.208}
\expandafter\def\csname res@ac1d/fno/64/RMSR/pearson\endcsname{0.974}
\expandafter\def\csname res@ac1d/fno/64/RMSR/r2\endcsname{0.995}
\expandafter\def\csname res@ac1d/fno/64/Reference/nrmse\endcsname{0.000}
\expandafter\def\csname res@ac1d/fno/64/Reference/pearson\endcsname{1.000}
\expandafter\def\csname res@ac1d/fno/64/Reference/r2\endcsname{0.998}
\expandafter\def\csname res@ac1dmid/fno/1024/IG/nrmse\endcsname{1.000}
\expandafter\def\csname res@ac1dmid/fno/1024/IG/r2\endcsname{-1.027}
\expandafter\def\csname res@ac1dmid/fno/1024/KernelSHAP/nrmse\endcsname{178.767}
\expandafter\def\csname res@ac1dmid/fno/1024/KernelSHAP/pearson\endcsname{0.002}
\expandafter\def\csname res@ac1dmid/fno/1024/KernelSHAP/r2\endcsname{-69.217}
\expandafter\def\csname res@ac1dmid/fno/1024/OperatorSHAP/nrmse\endcsname{0.776}
\expandafter\def\csname res@ac1dmid/fno/1024/OperatorSHAP/pearson\endcsname{0.577}
\expandafter\def\csname res@ac1dmid/fno/1024/OperatorSHAP/r2\endcsname{0.897}
\expandafter\def\csname res@ac1dmid/fno/1024/RMSR/nrmse\endcsname{2.058}
\expandafter\def\csname res@ac1dmid/fno/1024/RMSR/pearson\endcsname{0.018}
\expandafter\def\csname res@ac1dmid/fno/1024/RMSR/r2\endcsname{0.888}
\expandafter\def\csname res@ac1dmid/fno/1024/Reference/nrmse\endcsname{0.000}
\expandafter\def\csname res@ac1dmid/fno/1024/Reference/pearson\endcsname{1.000}
\expandafter\def\csname res@ac1dmid/fno/1024/Reference/r2\endcsname{0.898}
\expandafter\def\csname res@ac1dmid/fno/16/IG/nrmse\endcsname{1.000}
\expandafter\def\csname res@ac1dmid/fno/16/IG/r2\endcsname{-1.014}
\expandafter\def\csname res@ac1dmid/fno/16/KernelSHAP/nrmse\endcsname{0.142}
\expandafter\def\csname res@ac1dmid/fno/16/KernelSHAP/pearson\endcsname{0.989}
\expandafter\def\csname res@ac1dmid/fno/16/KernelSHAP/r2\endcsname{0.783}
\expandafter\def\csname res@ac1dmid/fno/16/OperatorSHAP/nrmse\endcsname{0.075}
\expandafter\def\csname res@ac1dmid/fno/16/OperatorSHAP/pearson\endcsname{0.998}
\expandafter\def\csname res@ac1dmid/fno/16/OperatorSHAP/r2\endcsname{0.787}
\expandafter\def\csname res@ac1dmid/fno/16/RMSR/nrmse\endcsname{0.024}
\expandafter\def\csname res@ac1dmid/fno/16/RMSR/pearson\endcsname{1.000}
\expandafter\def\csname res@ac1dmid/fno/16/RMSR/r2\endcsname{0.788}
\expandafter\def\csname res@ac1dmid/fno/16/Reference/nrmse\endcsname{0.000}
\expandafter\def\csname res@ac1dmid/fno/16/Reference/pearson\endcsname{1.000}
\expandafter\def\csname res@ac1dmid/fno/16/Reference/r2\endcsname{0.789}
\expandafter\def\csname res@ac1dmid/fno/256/IG/nrmse\endcsname{1.000}
\expandafter\def\csname res@ac1dmid/fno/256/IG/r2\endcsname{-1.032}
\expandafter\def\csname res@ac1dmid/fno/256/KernelSHAP/nrmse\endcsname{5.011}
\expandafter\def\csname res@ac1dmid/fno/256/KernelSHAP/pearson\endcsname{0.103}
\expandafter\def\csname res@ac1dmid/fno/256/KernelSHAP/r2\endcsname{0.687}
\expandafter\def\csname res@ac1dmid/fno/256/OperatorSHAP/nrmse\endcsname{0.399}
\expandafter\def\csname res@ac1dmid/fno/256/OperatorSHAP/pearson\endcsname{0.899}
\expandafter\def\csname res@ac1dmid/fno/256/OperatorSHAP/r2\endcsname{0.874}
\expandafter\def\csname res@ac1dmid/fno/256/RMSR/nrmse\endcsname{1.528}
\expandafter\def\csname res@ac1dmid/fno/256/RMSR/pearson\endcsname{0.122}
\expandafter\def\csname res@ac1dmid/fno/256/RMSR/r2\endcsname{0.858}
\expandafter\def\csname res@ac1dmid/fno/256/Reference/nrmse\endcsname{0.000}
\expandafter\def\csname res@ac1dmid/fno/256/Reference/pearson\endcsname{1.000}
\expandafter\def\csname res@ac1dmid/fno/256/Reference/r2\endcsname{0.875}
\expandafter\def\csname res@ac1dmid/fno/64/IG/nrmse\endcsname{1.000}
\expandafter\def\csname res@ac1dmid/fno/64/IG/r2\endcsname{-1.044}
\expandafter\def\csname res@ac1dmid/fno/64/KernelSHAP/nrmse\endcsname{0.981}
\expandafter\def\csname res@ac1dmid/fno/64/KernelSHAP/pearson\endcsname{0.672}
\expandafter\def\csname res@ac1dmid/fno/64/KernelSHAP/r2\endcsname{0.804}
\expandafter\def\csname res@ac1dmid/fno/64/OperatorSHAP/nrmse\endcsname{0.159}
\expandafter\def\csname res@ac1dmid/fno/64/OperatorSHAP/pearson\endcsname{0.984}
\expandafter\def\csname res@ac1dmid/fno/64/OperatorSHAP/r2\endcsname{0.838}
\expandafter\def\csname res@ac1dmid/fno/64/RMSR/nrmse\endcsname{0.614}
\expandafter\def\csname res@ac1dmid/fno/64/RMSR/pearson\endcsname{0.762}
\expandafter\def\csname res@ac1dmid/fno/64/RMSR/r2\endcsname{0.825}
\expandafter\def\csname res@ac1dmid/fno/64/Reference/nrmse\endcsname{0.000}
\expandafter\def\csname res@ac1dmid/fno/64/Reference/pearson\endcsname{1.000}
\expandafter\def\csname res@ac1dmid/fno/64/Reference/r2\endcsname{0.838}
\expandafter\def\csname res@adv1d/fno/1024/IG/nrmse\endcsname{0.604}
\expandafter\def\csname res@adv1d/fno/1024/IG/pearson\endcsname{0.874}
\expandafter\def\csname res@adv1d/fno/1024/IG/r2\endcsname{0.970}
\expandafter\def\csname res@adv1d/fno/1024/KernelSHAP/nrmse\endcsname{24.505}
\expandafter\def\csname res@adv1d/fno/1024/KernelSHAP/pearson\endcsname{0.045}
\expandafter\def\csname res@adv1d/fno/1024/KernelSHAP/r2\endcsname{-8.525}
\expandafter\def\csname res@adv1d/fno/1024/OperatorSHAP/nrmse\endcsname{0.799}
\expandafter\def\csname res@adv1d/fno/1024/OperatorSHAP/pearson\endcsname{0.641}
\expandafter\def\csname res@adv1d/fno/1024/OperatorSHAP/r2\endcsname{0.963}
\expandafter\def\csname res@adv1d/fno/1024/RMSR/nrmse\endcsname{0.758}
\expandafter\def\csname res@adv1d/fno/1024/RMSR/pearson\endcsname{0.705}
\expandafter\def\csname res@adv1d/fno/1024/RMSR/r2\endcsname{0.971}
\expandafter\def\csname res@adv1d/fno/1024/Reference/nrmse\endcsname{0.000}
\expandafter\def\csname res@adv1d/fno/1024/Reference/pearson\endcsname{1.000}
\expandafter\def\csname res@adv1d/fno/1024/Reference/r2\endcsname{0.982}
\expandafter\def\csname res@adv1d/fno/128/IG/nrmse\endcsname{0.439}
\expandafter\def\csname res@adv1d/fno/128/IG/pearson\endcsname{0.914}
\expandafter\def\csname res@adv1d/fno/128/IG/r2\endcsname{0.946}
\expandafter\def\csname res@adv1d/fno/128/KernelSHAP/nrmse\endcsname{0.331}
\expandafter\def\csname res@adv1d/fno/128/KernelSHAP/pearson\endcsname{0.943}
\expandafter\def\csname res@adv1d/fno/128/KernelSHAP/r2\endcsname{0.957}
\expandafter\def\csname res@adv1d/fno/128/OperatorSHAP/nrmse\endcsname{0.375}
\expandafter\def\csname res@adv1d/fno/128/OperatorSHAP/pearson\endcsname{0.940}
\expandafter\def\csname res@adv1d/fno/128/OperatorSHAP/r2\endcsname{0.953}
\expandafter\def\csname res@adv1d/fno/128/RMSR/nrmse\endcsname{0.423}
\expandafter\def\csname res@adv1d/fno/128/RMSR/pearson\endcsname{0.927}
\expandafter\def\csname res@adv1d/fno/128/RMSR/r2\endcsname{0.948}
\expandafter\def\csname res@adv1d/fno/128/Reference/nrmse\endcsname{0.000}
\expandafter\def\csname res@adv1d/fno/128/Reference/pearson\endcsname{1.000}
\expandafter\def\csname res@adv1d/fno/128/Reference/r2\endcsname{0.963}
\expandafter\def\csname res@adv1d/fno/256/IG/nrmse\endcsname{0.506}
\expandafter\def\csname res@adv1d/fno/256/IG/pearson\endcsname{0.903}
\expandafter\def\csname res@adv1d/fno/256/IG/r2\endcsname{0.961}
\expandafter\def\csname res@adv1d/fno/256/KernelSHAP/nrmse\endcsname{0.668}
\expandafter\def\csname res@adv1d/fno/256/KernelSHAP/pearson\endcsname{0.814}
\expandafter\def\csname res@adv1d/fno/256/KernelSHAP/r2\endcsname{0.955}
\expandafter\def\csname res@adv1d/fno/256/OperatorSHAP/nrmse\endcsname{0.397}
\expandafter\def\csname res@adv1d/fno/256/OperatorSHAP/pearson\endcsname{0.931}
\expandafter\def\csname res@adv1d/fno/256/OperatorSHAP/r2\endcsname{0.962}
\expandafter\def\csname res@adv1d/fno/256/RMSR/nrmse\endcsname{0.591}
\expandafter\def\csname res@adv1d/fno/256/RMSR/pearson\endcsname{0.795}
\expandafter\def\csname res@adv1d/fno/256/RMSR/r2\endcsname{0.955}
\expandafter\def\csname res@adv1d/fno/256/Reference/nrmse\endcsname{0.000}
\expandafter\def\csname res@adv1d/fno/256/Reference/pearson\endcsname{1.000}
\expandafter\def\csname res@adv1d/fno/256/Reference/r2\endcsname{0.972}
\expandafter\def\csname res@adv1d/fno/32/IG/nrmse\endcsname{0.308}
\expandafter\def\csname res@adv1d/fno/32/IG/pearson\endcsname{0.954}
\expandafter\def\csname res@adv1d/fno/32/IG/r2\endcsname{0.895}
\expandafter\def\csname res@adv1d/fno/32/KernelSHAP/nrmse\endcsname{0.095}
\expandafter\def\csname res@adv1d/fno/32/KernelSHAP/pearson\endcsname{0.995}
\expandafter\def\csname res@adv1d/fno/32/KernelSHAP/r2\endcsname{0.912}
\expandafter\def\csname res@adv1d/fno/32/OperatorSHAP/nrmse\endcsname{0.251}
\expandafter\def\csname res@adv1d/fno/32/OperatorSHAP/pearson\endcsname{0.968}
\expandafter\def\csname res@adv1d/fno/32/OperatorSHAP/r2\endcsname{0.898}
\expandafter\def\csname res@adv1d/fno/32/RMSR/nrmse\endcsname{0.112}
\expandafter\def\csname res@adv1d/fno/32/RMSR/pearson\endcsname{0.995}
\expandafter\def\csname res@adv1d/fno/32/RMSR/r2\endcsname{0.912}
\expandafter\def\csname res@adv1d/fno/32/Reference/nrmse\endcsname{0.000}
\expandafter\def\csname res@adv1d/fno/32/Reference/pearson\endcsname{1.000}
\expandafter\def\csname res@adv1d/fno/32/Reference/r2\endcsname{0.914}
\expandafter\def\csname res@adv1d/fno/512/IG/nrmse\endcsname{0.540}
\expandafter\def\csname res@adv1d/fno/512/IG/pearson\endcsname{0.894}
\expandafter\def\csname res@adv1d/fno/512/IG/r2\endcsname{0.966}
\expandafter\def\csname res@adv1d/fno/512/KernelSHAP/nrmse\endcsname{10.774}
\expandafter\def\csname res@adv1d/fno/512/KernelSHAP/pearson\endcsname{0.077}
\expandafter\def\csname res@adv1d/fno/512/KernelSHAP/r2\endcsname{-2.277}
\expandafter\def\csname res@adv1d/fno/512/OperatorSHAP/nrmse\endcsname{0.600}
\expandafter\def\csname res@adv1d/fno/512/OperatorSHAP/pearson\endcsname{0.841}
\expandafter\def\csname res@adv1d/fno/512/OperatorSHAP/r2\endcsname{0.960}
\expandafter\def\csname res@adv1d/fno/512/RMSR/nrmse\endcsname{0.773}
\expandafter\def\csname res@adv1d/fno/512/RMSR/pearson\endcsname{0.663}
\expandafter\def\csname res@adv1d/fno/512/RMSR/r2\endcsname{0.962}
\expandafter\def\csname res@adv1d/fno/512/Reference/nrmse\endcsname{0.000}
\expandafter\def\csname res@adv1d/fno/512/Reference/pearson\endcsname{1.000}
\expandafter\def\csname res@adv1d/fno/512/Reference/r2\endcsname{0.978}
\expandafter\def\csname res@adv1d/fno/64/IG/nrmse\endcsname{0.381}
\expandafter\def\csname res@adv1d/fno/64/IG/pearson\endcsname{0.934}
\expandafter\def\csname res@adv1d/fno/64/IG/r2\endcsname{0.926}
\expandafter\def\csname res@adv1d/fno/64/KernelSHAP/nrmse\endcsname{0.182}
\expandafter\def\csname res@adv1d/fno/64/KernelSHAP/pearson\endcsname{0.982}
\expandafter\def\csname res@adv1d/fno/64/KernelSHAP/r2\endcsname{0.944}
\expandafter\def\csname res@adv1d/fno/64/OperatorSHAP/nrmse\endcsname{0.308}
\expandafter\def\csname res@adv1d/fno/64/OperatorSHAP/pearson\endcsname{0.953}
\expandafter\def\csname res@adv1d/fno/64/OperatorSHAP/r2\endcsname{0.928}
\expandafter\def\csname res@adv1d/fno/64/RMSR/nrmse\endcsname{0.253}
\expandafter\def\csname res@adv1d/fno/64/RMSR/pearson\endcsname{0.974}
\expandafter\def\csname res@adv1d/fno/64/RMSR/r2\endcsname{0.937}
\expandafter\def\csname res@adv1d/fno/64/Reference/nrmse\endcsname{0.000}
\expandafter\def\csname res@adv1d/fno/64/Reference/pearson\endcsname{1.000}
\expandafter\def\csname res@adv1d/fno/64/Reference/r2\endcsname{0.947}
\expandafter\def\csname res@adv1dmid/fno/1024/IG/nrmse\endcsname{1.000}
\expandafter\def\csname res@adv1dmid/fno/1024/IG/r2\endcsname{-0.988}
\expandafter\def\csname res@adv1dmid/fno/1024/KernelSHAP/nrmse\endcsname{85.246}
\expandafter\def\csname res@adv1dmid/fno/1024/KernelSHAP/pearson\endcsname{0.011}
\expandafter\def\csname res@adv1dmid/fno/1024/KernelSHAP/r2\endcsname{-81.406}
\expandafter\def\csname res@adv1dmid/fno/1024/OperatorSHAP/nrmse\endcsname{0.948}
\expandafter\def\csname res@adv1dmid/fno/1024/OperatorSHAP/pearson\endcsname{0.325}
\expandafter\def\csname res@adv1dmid/fno/1024/OperatorSHAP/r2\endcsname{0.861}
\expandafter\def\csname res@adv1dmid/fno/1024/RMSR/nrmse\endcsname{1.044}
\expandafter\def\csname res@adv1dmid/fno/1024/RMSR/pearson\endcsname{0.238}
\expandafter\def\csname res@adv1dmid/fno/1024/RMSR/r2\endcsname{0.879}
\expandafter\def\csname res@adv1dmid/fno/1024/Reference/nrmse\endcsname{0.000}
\expandafter\def\csname res@adv1dmid/fno/1024/Reference/pearson\endcsname{1.000}
\expandafter\def\csname res@adv1dmid/fno/1024/Reference/r2\endcsname{0.891}
\expandafter\def\csname res@adv1dmid/fno/16/IG/nrmse\endcsname{1.000}
\expandafter\def\csname res@adv1dmid/fno/16/IG/r2\endcsname{-0.991}
\expandafter\def\csname res@adv1dmid/fno/16/KernelSHAP/nrmse\endcsname{0.055}
\expandafter\def\csname res@adv1dmid/fno/16/KernelSHAP/pearson\endcsname{0.999}
\expandafter\def\csname res@adv1dmid/fno/16/KernelSHAP/r2\endcsname{0.896}
\expandafter\def\csname res@adv1dmid/fno/16/OperatorSHAP/nrmse\endcsname{0.120}
\expandafter\def\csname res@adv1dmid/fno/16/OperatorSHAP/pearson\endcsname{0.995}
\expandafter\def\csname res@adv1dmid/fno/16/OperatorSHAP/r2\endcsname{0.890}
\expandafter\def\csname res@adv1dmid/fno/16/RMSR/nrmse\endcsname{0.040}
\expandafter\def\csname res@adv1dmid/fno/16/RMSR/pearson\endcsname{0.999}
\expandafter\def\csname res@adv1dmid/fno/16/RMSR/r2\endcsname{0.897}
\expandafter\def\csname res@adv1dmid/fno/16/Reference/nrmse\endcsname{0.000}
\expandafter\def\csname res@adv1dmid/fno/16/Reference/pearson\endcsname{1.000}
\expandafter\def\csname res@adv1dmid/fno/16/Reference/r2\endcsname{0.898}
\expandafter\def\csname res@adv1dmid/fno/256/IG/nrmse\endcsname{1.000}
\expandafter\def\csname res@adv1dmid/fno/256/IG/r2\endcsname{-0.981}
\expandafter\def\csname res@adv1dmid/fno/256/KernelSHAP/nrmse\endcsname{2.382}
\expandafter\def\csname res@adv1dmid/fno/256/KernelSHAP/pearson\endcsname{0.287}
\expandafter\def\csname res@adv1dmid/fno/256/KernelSHAP/r2\endcsname{0.708}
\expandafter\def\csname res@adv1dmid/fno/256/OperatorSHAP/nrmse\endcsname{0.755}
\expandafter\def\csname res@adv1dmid/fno/256/OperatorSHAP/pearson\endcsname{0.672}
\expandafter\def\csname res@adv1dmid/fno/256/OperatorSHAP/r2\endcsname{0.835}
\expandafter\def\csname res@adv1dmid/fno/256/RMSR/nrmse\endcsname{0.943}
\expandafter\def\csname res@adv1dmid/fno/256/RMSR/pearson\endcsname{0.396}
\expandafter\def\csname res@adv1dmid/fno/256/RMSR/r2\endcsname{0.843}
\expandafter\def\csname res@adv1dmid/fno/256/Reference/nrmse\endcsname{0.000}
\expandafter\def\csname res@adv1dmid/fno/256/Reference/pearson\endcsname{1.000}
\expandafter\def\csname res@adv1dmid/fno/256/Reference/r2\endcsname{0.864}
\expandafter\def\csname res@adv1dmid/fno/64/IG/nrmse\endcsname{1.000}
\expandafter\def\csname res@adv1dmid/fno/64/IG/r2\endcsname{-0.973}
\expandafter\def\csname res@adv1dmid/fno/64/KernelSHAP/nrmse\endcsname{0.499}
\expandafter\def\csname res@adv1dmid/fno/64/KernelSHAP/pearson\endcsname{0.882}
\expandafter\def\csname res@adv1dmid/fno/64/KernelSHAP/r2\endcsname{0.791}
\expandafter\def\csname res@adv1dmid/fno/64/OperatorSHAP/nrmse\endcsname{0.366}
\expandafter\def\csname res@adv1dmid/fno/64/OperatorSHAP/pearson\endcsname{0.936}
\expandafter\def\csname res@adv1dmid/fno/64/OperatorSHAP/r2\endcsname{0.799}
\expandafter\def\csname res@adv1dmid/fno/64/RMSR/nrmse\endcsname{0.408}
\expandafter\def\csname res@adv1dmid/fno/64/RMSR/pearson\endcsname{0.922}
\expandafter\def\csname res@adv1dmid/fno/64/RMSR/r2\endcsname{0.800}
\expandafter\def\csname res@adv1dmid/fno/64/Reference/nrmse\endcsname{0.000}
\expandafter\def\csname res@adv1dmid/fno/64/Reference/pearson\endcsname{1.000}
\expandafter\def\csname res@adv1dmid/fno/64/Reference/r2\endcsname{0.818}
\expandafter\def\csname res@burgers1d/fno/1024/IG/nrmse\endcsname{0.317}
\expandafter\def\csname res@burgers1d/fno/1024/IG/pearson\endcsname{0.958}
\expandafter\def\csname res@burgers1d/fno/1024/IG/r2\endcsname{0.985}
\expandafter\def\csname res@burgers1d/fno/1024/KernelSHAP/nrmse\endcsname{28.083}
\expandafter\def\csname res@burgers1d/fno/1024/KernelSHAP/pearson\endcsname{0.050}
\expandafter\def\csname res@burgers1d/fno/1024/KernelSHAP/r2\endcsname{-2.497}
\expandafter\def\csname res@burgers1d/fno/1024/OperatorSHAP/nrmse\endcsname{0.480}
\expandafter\def\csname res@burgers1d/fno/1024/OperatorSHAP/pearson\endcsname{0.900}
\expandafter\def\csname res@burgers1d/fno/1024/OperatorSHAP/r2\endcsname{0.970}
\expandafter\def\csname res@burgers1d/fno/1024/RMSR/nrmse\endcsname{1.357}
\expandafter\def\csname res@burgers1d/fno/1024/RMSR/pearson\endcsname{0.240}
\expandafter\def\csname res@burgers1d/fno/1024/RMSR/r2\endcsname{0.960}
\expandafter\def\csname res@burgers1d/fno/1024/Reference/nrmse\endcsname{0.000}
\expandafter\def\csname res@burgers1d/fno/1024/Reference/pearson\endcsname{1.000}
\expandafter\def\csname res@burgers1d/fno/1024/Reference/r2\endcsname{0.967}
\expandafter\def\csname res@burgers1d/fno/128/IG/nrmse\endcsname{0.046}
\expandafter\def\csname res@burgers1d/fno/128/IG/pearson\endcsname{0.999}
\expandafter\def\csname res@burgers1d/fno/128/IG/r2\endcsname{0.965}
\expandafter\def\csname res@burgers1d/fno/128/KernelSHAP/nrmse\endcsname{0.118}
\expandafter\def\csname res@burgers1d/fno/128/KernelSHAP/pearson\endcsname{0.992}
\expandafter\def\csname res@burgers1d/fno/128/KernelSHAP/r2\endcsname{0.966}
\expandafter\def\csname res@burgers1d/fno/128/OperatorSHAP/nrmse\endcsname{0.129}
\expandafter\def\csname res@burgers1d/fno/128/OperatorSHAP/pearson\endcsname{0.992}
\expandafter\def\csname res@burgers1d/fno/128/OperatorSHAP/r2\endcsname{0.960}
\expandafter\def\csname res@burgers1d/fno/128/RMSR/nrmse\endcsname{0.488}
\expandafter\def\csname res@burgers1d/fno/128/RMSR/pearson\endcsname{0.900}
\expandafter\def\csname res@burgers1d/fno/128/RMSR/r2\endcsname{0.932}
\expandafter\def\csname res@burgers1d/fno/128/Reference/nrmse\endcsname{0.000}
\expandafter\def\csname res@burgers1d/fno/128/Reference/pearson\endcsname{1.000}
\expandafter\def\csname res@burgers1d/fno/128/Reference/r2\endcsname{0.965}
\expandafter\def\csname res@burgers1d/fno/16/IG/nrmse\endcsname{0.060}
\expandafter\def\csname res@burgers1d/fno/16/IG/pearson\endcsname{0.999}
\expandafter\def\csname res@burgers1d/fno/16/IG/r2\endcsname{0.934}
\expandafter\def\csname res@burgers1d/fno/16/KernelSHAP/nrmse\endcsname{0.010}
\expandafter\def\csname res@burgers1d/fno/16/KernelSHAP/pearson\endcsname{1.000}
\expandafter\def\csname res@burgers1d/fno/16/KernelSHAP/r2\endcsname{0.937}
\expandafter\def\csname res@burgers1d/fno/16/OperatorSHAP/nrmse\endcsname{0.060}
\expandafter\def\csname res@burgers1d/fno/16/OperatorSHAP/pearson\endcsname{0.998}
\expandafter\def\csname res@burgers1d/fno/16/OperatorSHAP/r2\endcsname{0.926}
\expandafter\def\csname res@burgers1d/fno/16/RMSR/nrmse\endcsname{0.008}
\expandafter\def\csname res@burgers1d/fno/16/RMSR/pearson\endcsname{1.000}
\expandafter\def\csname res@burgers1d/fno/16/RMSR/r2\endcsname{0.937}
\expandafter\def\csname res@burgers1d/fno/16/Reference/nrmse\endcsname{0.000}
\expandafter\def\csname res@burgers1d/fno/16/Reference/pearson\endcsname{1.000}
\expandafter\def\csname res@burgers1d/fno/16/Reference/r2\endcsname{0.937}
\expandafter\def\csname res@burgers1d/fno/256/IG/nrmse\endcsname{0.094}
\expandafter\def\csname res@burgers1d/fno/256/IG/pearson\endcsname{0.997}
\expandafter\def\csname res@burgers1d/fno/256/IG/r2\endcsname{0.971}
\expandafter\def\csname res@burgers1d/fno/256/KernelSHAP/nrmse\endcsname{0.317}
\expandafter\def\csname res@burgers1d/fno/256/KernelSHAP/pearson\endcsname{0.946}
\expandafter\def\csname res@burgers1d/fno/256/KernelSHAP/r2\endcsname{0.971}
\expandafter\def\csname res@burgers1d/fno/256/OperatorSHAP/nrmse\endcsname{0.195}
\expandafter\def\csname res@burgers1d/fno/256/OperatorSHAP/pearson\endcsname{0.985}
\expandafter\def\csname res@burgers1d/fno/256/OperatorSHAP/r2\endcsname{0.966}
\expandafter\def\csname res@burgers1d/fno/256/RMSR/nrmse\endcsname{0.805}
\expandafter\def\csname res@burgers1d/fno/256/RMSR/pearson\endcsname{0.596}
\expandafter\def\csname res@burgers1d/fno/256/RMSR/r2\endcsname{0.944}
\expandafter\def\csname res@burgers1d/fno/256/Reference/nrmse\endcsname{0.000}
\expandafter\def\csname res@burgers1d/fno/256/Reference/pearson\endcsname{1.000}
\expandafter\def\csname res@burgers1d/fno/256/Reference/r2\endcsname{0.970}
\expandafter\def\csname res@burgers1d/fno/32/IG/nrmse\endcsname{0.050}
\expandafter\def\csname res@burgers1d/fno/32/IG/pearson\endcsname{0.999}
\expandafter\def\csname res@burgers1d/fno/32/IG/r2\endcsname{0.941}
\expandafter\def\csname res@burgers1d/fno/32/KernelSHAP/nrmse\endcsname{0.025}
\expandafter\def\csname res@burgers1d/fno/32/KernelSHAP/pearson\endcsname{1.000}
\expandafter\def\csname res@burgers1d/fno/32/KernelSHAP/r2\endcsname{0.946}
\expandafter\def\csname res@burgers1d/fno/32/OperatorSHAP/nrmse\endcsname{0.078}
\expandafter\def\csname res@burgers1d/fno/32/OperatorSHAP/pearson\endcsname{0.997}
\expandafter\def\csname res@burgers1d/fno/32/OperatorSHAP/r2\endcsname{0.939}
\expandafter\def\csname res@burgers1d/fno/32/RMSR/nrmse\endcsname{0.081}
\expandafter\def\csname res@burgers1d/fno/32/RMSR/pearson\endcsname{0.998}
\expandafter\def\csname res@burgers1d/fno/32/RMSR/r2\endcsname{0.940}
\expandafter\def\csname res@burgers1d/fno/32/Reference/nrmse\endcsname{0.000}
\expandafter\def\csname res@burgers1d/fno/32/Reference/pearson\endcsname{1.000}
\expandafter\def\csname res@burgers1d/fno/32/Reference/r2\endcsname{0.946}
\expandafter\def\csname res@burgers1d/fno/512/IG/nrmse\endcsname{0.168}
\expandafter\def\csname res@burgers1d/fno/512/IG/pearson\endcsname{0.988}
\expandafter\def\csname res@burgers1d/fno/512/IG/r2\endcsname{0.980}
\expandafter\def\csname res@burgers1d/fno/512/KernelSHAP/nrmse\endcsname{7.788}
\expandafter\def\csname res@burgers1d/fno/512/KernelSHAP/pearson\endcsname{0.119}
\expandafter\def\csname res@burgers1d/fno/512/KernelSHAP/r2\endcsname{0.343}
\expandafter\def\csname res@burgers1d/fno/512/OperatorSHAP/nrmse\endcsname{0.310}
\expandafter\def\csname res@burgers1d/fno/512/OperatorSHAP/pearson\endcsname{0.961}
\expandafter\def\csname res@burgers1d/fno/512/OperatorSHAP/r2\endcsname{0.969}
\expandafter\def\csname res@burgers1d/fno/512/RMSR/nrmse\endcsname{1.169}
\expandafter\def\csname res@burgers1d/fno/512/RMSR/pearson\endcsname{0.268}
\expandafter\def\csname res@burgers1d/fno/512/RMSR/r2\endcsname{0.951}
\expandafter\def\csname res@burgers1d/fno/512/Reference/nrmse\endcsname{0.000}
\expandafter\def\csname res@burgers1d/fno/512/Reference/pearson\endcsname{1.000}
\expandafter\def\csname res@burgers1d/fno/512/Reference/r2\endcsname{0.968}
\expandafter\def\csname res@burgers1d/fno/64/IG/nrmse\endcsname{0.043}
\expandafter\def\csname res@burgers1d/fno/64/IG/pearson\endcsname{0.999}
\expandafter\def\csname res@burgers1d/fno/64/IG/r2\endcsname{0.956}
\expandafter\def\csname res@burgers1d/fno/64/KernelSHAP/nrmse\endcsname{0.049}
\expandafter\def\csname res@burgers1d/fno/64/KernelSHAP/pearson\endcsname{0.999}
\expandafter\def\csname res@burgers1d/fno/64/KernelSHAP/r2\endcsname{0.957}
\expandafter\def\csname res@burgers1d/fno/64/OperatorSHAP/nrmse\endcsname{0.100}
\expandafter\def\csname res@burgers1d/fno/64/OperatorSHAP/pearson\endcsname{0.995}
\expandafter\def\csname res@burgers1d/fno/64/OperatorSHAP/r2\endcsname{0.952}
\expandafter\def\csname res@burgers1d/fno/64/RMSR/nrmse\endcsname{0.256}
\expandafter\def\csname res@burgers1d/fno/64/RMSR/pearson\endcsname{0.979}
\expandafter\def\csname res@burgers1d/fno/64/RMSR/r2\endcsname{0.924}
\expandafter\def\csname res@burgers1d/fno/64/Reference/nrmse\endcsname{0.000}
\expandafter\def\csname res@burgers1d/fno/64/Reference/pearson\endcsname{1.000}
\expandafter\def\csname res@burgers1d/fno/64/Reference/r2\endcsname{0.958}
\expandafter\def\csname res@burgers1d/fnoRelu/256/IG/nrmse\endcsname{0.121}
\expandafter\def\csname res@burgers1d/fnoRelu/256/IG/pearson\endcsname{0.994}
\expandafter\def\csname res@burgers1d/fnoRelu/256/IG/r2\endcsname{0.978}
\expandafter\def\csname res@burgers1d/fnoRelu/256/KernelSHAP/nrmse\endcsname{0.329}
\expandafter\def\csname res@burgers1d/fnoRelu/256/KernelSHAP/pearson\endcsname{0.936}
\expandafter\def\csname res@burgers1d/fnoRelu/256/KernelSHAP/r2\endcsname{0.976}
\expandafter\def\csname res@burgers1d/fnoRelu/256/OperatorSHAP/nrmse\endcsname{0.174}
\expandafter\def\csname res@burgers1d/fnoRelu/256/OperatorSHAP/pearson\endcsname{0.990}
\expandafter\def\csname res@burgers1d/fnoRelu/256/OperatorSHAP/r2\endcsname{0.965}
\expandafter\def\csname res@burgers1d/fnoRelu/256/RMSR/nrmse\endcsname{0.820}
\expandafter\def\csname res@burgers1d/fnoRelu/256/RMSR/pearson\endcsname{0.561}
\expandafter\def\csname res@burgers1d/fnoRelu/256/RMSR/r2\endcsname{0.937}
\expandafter\def\csname res@burgers1d/fnoRelu/256/Reference/nrmse\endcsname{0.000}
\expandafter\def\csname res@burgers1d/fnoRelu/256/Reference/pearson\endcsname{1.000}
\expandafter\def\csname res@burgers1d/fnoRelu/256/Reference/r2\endcsname{0.976}
\expandafter\def\csname res@burgers1d/fnoRelu/64/IG/nrmse\endcsname{0.071}
\expandafter\def\csname res@burgers1d/fnoRelu/64/IG/pearson\endcsname{0.998}
\expandafter\def\csname res@burgers1d/fnoRelu/64/IG/r2\endcsname{0.964}
\expandafter\def\csname res@burgers1d/fnoRelu/64/KernelSHAP/nrmse\endcsname{0.055}
\expandafter\def\csname res@burgers1d/fnoRelu/64/KernelSHAP/pearson\endcsname{0.998}
\expandafter\def\csname res@burgers1d/fnoRelu/64/KernelSHAP/r2\endcsname{0.963}
\expandafter\def\csname res@burgers1d/fnoRelu/64/OperatorSHAP/nrmse\endcsname{0.099}
\expandafter\def\csname res@burgers1d/fnoRelu/64/OperatorSHAP/pearson\endcsname{0.995}
\expandafter\def\csname res@burgers1d/fnoRelu/64/OperatorSHAP/r2\endcsname{0.948}
\expandafter\def\csname res@burgers1d/fnoRelu/64/RMSR/nrmse\endcsname{0.262}
\expandafter\def\csname res@burgers1d/fnoRelu/64/RMSR/pearson\endcsname{0.977}
\expandafter\def\csname res@burgers1d/fnoRelu/64/RMSR/r2\endcsname{0.918}
\expandafter\def\csname res@burgers1d/fnoRelu/64/Reference/nrmse\endcsname{0.000}
\expandafter\def\csname res@burgers1d/fnoRelu/64/Reference/pearson\endcsname{1.000}
\expandafter\def\csname res@burgers1d/fnoRelu/64/Reference/r2\endcsname{0.964}
\expandafter\def\csname res@burgers1dho1024/fno/1024/IG/nrmse\endcsname{0.317}
\expandafter\def\csname res@burgers1dho1024/fno/1024/IG/pearson\endcsname{0.958}
\expandafter\def\csname res@burgers1dho1024/fno/1024/IG/r2\endcsname{0.985}
\expandafter\def\csname res@burgers1dho1024/fno/1024/KernelSHAP/nrmse\endcsname{28.083}
\expandafter\def\csname res@burgers1dho1024/fno/1024/KernelSHAP/pearson\endcsname{0.050}
\expandafter\def\csname res@burgers1dho1024/fno/1024/KernelSHAP/r2\endcsname{-2.497}
\expandafter\def\csname res@burgers1dho1024/fno/1024/OperatorSHAP/nrmse\endcsname{0.481}
\expandafter\def\csname res@burgers1dho1024/fno/1024/OperatorSHAP/pearson\endcsname{0.900}
\expandafter\def\csname res@burgers1dho1024/fno/1024/OperatorSHAP/r2\endcsname{0.970}
\expandafter\def\csname res@burgers1dho1024/fno/1024/RMSR/nrmse\endcsname{1.357}
\expandafter\def\csname res@burgers1dho1024/fno/1024/RMSR/pearson\endcsname{0.240}
\expandafter\def\csname res@burgers1dho1024/fno/1024/RMSR/r2\endcsname{0.960}
\expandafter\def\csname res@burgers1dho1024/fno/1024/Reference/nrmse\endcsname{0.000}
\expandafter\def\csname res@burgers1dho1024/fno/1024/Reference/pearson\endcsname{1.000}
\expandafter\def\csname res@burgers1dho1024/fno/1024/Reference/r2\endcsname{0.967}
\expandafter\def\csname res@burgers1dho256/fno/256/IG/nrmse\endcsname{0.094}
\expandafter\def\csname res@burgers1dho256/fno/256/IG/pearson\endcsname{0.997}
\expandafter\def\csname res@burgers1dho256/fno/256/IG/r2\endcsname{0.971}
\expandafter\def\csname res@burgers1dho256/fno/256/KernelSHAP/nrmse\endcsname{0.317}
\expandafter\def\csname res@burgers1dho256/fno/256/KernelSHAP/pearson\endcsname{0.946}
\expandafter\def\csname res@burgers1dho256/fno/256/KernelSHAP/r2\endcsname{0.971}
\expandafter\def\csname res@burgers1dho256/fno/256/OperatorSHAP/nrmse\endcsname{0.195}
\expandafter\def\csname res@burgers1dho256/fno/256/OperatorSHAP/pearson\endcsname{0.985}
\expandafter\def\csname res@burgers1dho256/fno/256/OperatorSHAP/r2\endcsname{0.965}
\expandafter\def\csname res@burgers1dho256/fno/256/RMSR/nrmse\endcsname{0.805}
\expandafter\def\csname res@burgers1dho256/fno/256/RMSR/pearson\endcsname{0.596}
\expandafter\def\csname res@burgers1dho256/fno/256/RMSR/r2\endcsname{0.944}
\expandafter\def\csname res@burgers1dho256/fno/256/Reference/nrmse\endcsname{0.000}
\expandafter\def\csname res@burgers1dho256/fno/256/Reference/pearson\endcsname{1.000}
\expandafter\def\csname res@burgers1dho256/fno/256/Reference/r2\endcsname{0.970}
\expandafter\def\csname res@burgers1dmid/fno/1024/IG/nrmse\endcsname{1.000}
\expandafter\def\csname res@burgers1dmid/fno/1024/IG/r2\endcsname{-1.096}
\expandafter\def\csname res@burgers1dmid/fno/1024/KernelSHAP/nrmse\endcsname{173.467}
\expandafter\def\csname res@burgers1dmid/fno/1024/KernelSHAP/pearson\endcsname{0.006}
\expandafter\def\csname res@burgers1dmid/fno/1024/KernelSHAP/r2\endcsname{-86.894}
\expandafter\def\csname res@burgers1dmid/fno/1024/OperatorSHAP/nrmse\endcsname{0.564}
\expandafter\def\csname res@burgers1dmid/fno/1024/OperatorSHAP/pearson\endcsname{0.831}
\expandafter\def\csname res@burgers1dmid/fno/1024/OperatorSHAP/r2\endcsname{0.865}
\expandafter\def\csname res@burgers1dmid/fno/1024/RMSR/nrmse\endcsname{1.828}
\expandafter\def\csname res@burgers1dmid/fno/1024/RMSR/pearson\endcsname{0.038}
\expandafter\def\csname res@burgers1dmid/fno/1024/RMSR/r2\endcsname{0.858}
\expandafter\def\csname res@burgers1dmid/fno/1024/Reference/nrmse\endcsname{0.000}
\expandafter\def\csname res@burgers1dmid/fno/1024/Reference/pearson\endcsname{1.000}
\expandafter\def\csname res@burgers1dmid/fno/1024/Reference/r2\endcsname{0.868}
\expandafter\def\csname res@burgers1dmid/fno/16/IG/nrmse\endcsname{1.000}
\expandafter\def\csname res@burgers1dmid/fno/16/IG/r2\endcsname{-1.195}
\expandafter\def\csname res@burgers1dmid/fno/16/KernelSHAP/nrmse\endcsname{0.088}
\expandafter\def\csname res@burgers1dmid/fno/16/KernelSHAP/pearson\endcsname{0.996}
\expandafter\def\csname res@burgers1dmid/fno/16/KernelSHAP/r2\endcsname{0.701}
\expandafter\def\csname res@burgers1dmid/fno/16/OperatorSHAP/nrmse\endcsname{0.191}
\expandafter\def\csname res@burgers1dmid/fno/16/OperatorSHAP/pearson\endcsname{0.983}
\expandafter\def\csname res@burgers1dmid/fno/16/OperatorSHAP/r2\endcsname{0.693}
\expandafter\def\csname res@burgers1dmid/fno/16/RMSR/nrmse\endcsname{0.036}
\expandafter\def\csname res@burgers1dmid/fno/16/RMSR/pearson\endcsname{0.999}
\expandafter\def\csname res@burgers1dmid/fno/16/RMSR/r2\endcsname{0.703}
\expandafter\def\csname res@burgers1dmid/fno/16/Reference/nrmse\endcsname{0.000}
\expandafter\def\csname res@burgers1dmid/fno/16/Reference/pearson\endcsname{1.000}
\expandafter\def\csname res@burgers1dmid/fno/16/Reference/r2\endcsname{0.704}
\expandafter\def\csname res@burgers1dmid/fno/256/IG/nrmse\endcsname{1.000}
\expandafter\def\csname res@burgers1dmid/fno/256/IG/r2\endcsname{-1.113}
\expandafter\def\csname res@burgers1dmid/fno/256/KernelSHAP/nrmse\endcsname{3.391}
\expandafter\def\csname res@burgers1dmid/fno/256/KernelSHAP/pearson\endcsname{0.193}
\expandafter\def\csname res@burgers1dmid/fno/256/KernelSHAP/r2\endcsname{0.702}
\expandafter\def\csname res@burgers1dmid/fno/256/OperatorSHAP/nrmse\endcsname{0.232}
\expandafter\def\csname res@burgers1dmid/fno/256/OperatorSHAP/pearson\endcsname{0.975}
\expandafter\def\csname res@burgers1dmid/fno/256/OperatorSHAP/r2\endcsname{0.832}
\expandafter\def\csname res@burgers1dmid/fno/256/RMSR/nrmse\endcsname{1.221}
\expandafter\def\csname res@burgers1dmid/fno/256/RMSR/pearson\endcsname{0.178}
\expandafter\def\csname res@burgers1dmid/fno/256/RMSR/r2\endcsname{0.810}
\expandafter\def\csname res@burgers1dmid/fno/256/Reference/nrmse\endcsname{0.000}
\expandafter\def\csname res@burgers1dmid/fno/256/Reference/pearson\endcsname{1.000}
\expandafter\def\csname res@burgers1dmid/fno/256/Reference/r2\endcsname{0.832}
\expandafter\def\csname res@burgers1dmid/fno/64/IG/nrmse\endcsname{1.000}
\expandafter\def\csname res@burgers1dmid/fno/64/IG/r2\endcsname{-1.120}
\expandafter\def\csname res@burgers1dmid/fno/64/KernelSHAP/nrmse\endcsname{0.535}
\expandafter\def\csname res@burgers1dmid/fno/64/KernelSHAP/pearson\endcsname{0.855}
\expandafter\def\csname res@burgers1dmid/fno/64/KernelSHAP/r2\endcsname{0.764}
\expandafter\def\csname res@burgers1dmid/fno/64/OperatorSHAP/nrmse\endcsname{0.160}
\expandafter\def\csname res@burgers1dmid/fno/64/OperatorSHAP/pearson\endcsname{0.990}
\expandafter\def\csname res@burgers1dmid/fno/64/OperatorSHAP/r2\endcsname{0.776}
\expandafter\def\csname res@burgers1dmid/fno/64/RMSR/nrmse\endcsname{0.449}
\expandafter\def\csname res@burgers1dmid/fno/64/RMSR/pearson\endcsname{0.884}
\expandafter\def\csname res@burgers1dmid/fno/64/RMSR/r2\endcsname{0.765}
\expandafter\def\csname res@burgers1dmid/fno/64/Reference/nrmse\endcsname{0.000}
\expandafter\def\csname res@burgers1dmid/fno/64/Reference/pearson\endcsname{1.000}
\expandafter\def\csname res@burgers1dmid/fno/64/Reference/r2\endcsname{0.777}
\expandafter\def\csname res@darcy2d/localFNO/1024/IG/nrmse\endcsname{0.485}
\expandafter\def\csname res@darcy2d/localFNO/1024/IG/pearson\endcsname{0.957}
\expandafter\def\csname res@darcy2d/localFNO/1024/IG/r2\endcsname{0.951}
\expandafter\def\csname res@darcy2d/localFNO/1024/KernelSHAP/nrmse\endcsname{13.905}
\expandafter\def\csname res@darcy2d/localFNO/1024/KernelSHAP/pearson\endcsname{0.084}
\expandafter\def\csname res@darcy2d/localFNO/1024/KernelSHAP/r2\endcsname{-6.410}
\expandafter\def\csname res@darcy2d/localFNO/1024/OperatorSHAP/nrmse\endcsname{0.444}
\expandafter\def\csname res@darcy2d/localFNO/1024/OperatorSHAP/pearson\endcsname{0.946}
\expandafter\def\csname res@darcy2d/localFNO/1024/OperatorSHAP/r2\endcsname{0.950}
\expandafter\def\csname res@darcy2d/localFNO/1024/RMSR/nrmse\endcsname{0.694}
\expandafter\def\csname res@darcy2d/localFNO/1024/RMSR/pearson\endcsname{0.736}
\expandafter\def\csname res@darcy2d/localFNO/1024/RMSR/r2\endcsname{0.931}
\expandafter\def\csname res@darcy2d/localFNO/1024/Reference/nrmse\endcsname{0.000}
\expandafter\def\csname res@darcy2d/localFNO/1024/Reference/pearson\endcsname{1.000}
\expandafter\def\csname res@darcy2d/localFNO/1024/Reference/r2\endcsname{0.948}
\expandafter\def\csname res@darcy2d/localFNO/256/IG/nrmse\endcsname{0.246}
\expandafter\def\csname res@darcy2d/localFNO/256/IG/pearson\endcsname{0.987}
\expandafter\def\csname res@darcy2d/localFNO/256/IG/r2\endcsname{0.941}
\expandafter\def\csname res@darcy2d/localFNO/256/KernelSHAP/nrmse\endcsname{0.216}
\expandafter\def\csname res@darcy2d/localFNO/256/KernelSHAP/pearson\endcsname{0.986}
\expandafter\def\csname res@darcy2d/localFNO/256/KernelSHAP/r2\endcsname{0.942}
\expandafter\def\csname res@darcy2d/localFNO/256/OperatorSHAP/nrmse\endcsname{0.216}
\expandafter\def\csname res@darcy2d/localFNO/256/OperatorSHAP/pearson\endcsname{0.982}
\expandafter\def\csname res@darcy2d/localFNO/256/OperatorSHAP/r2\endcsname{0.941}
\expandafter\def\csname res@darcy2d/localFNO/256/RMSR/nrmse\endcsname{0.561}
\expandafter\def\csname res@darcy2d/localFNO/256/RMSR/pearson\endcsname{0.876}
\expandafter\def\csname res@darcy2d/localFNO/256/RMSR/r2\endcsname{0.895}
\expandafter\def\csname res@darcy2d/localFNO/256/Reference/nrmse\endcsname{0.000}
\expandafter\def\csname res@darcy2d/localFNO/256/Reference/pearson\endcsname{1.000}
\expandafter\def\csname res@darcy2d/localFNO/256/Reference/r2\endcsname{0.944}
\expandafter\def\csname res@darcy2d/localFNO/64/IG/nrmse\endcsname{0.155}
\expandafter\def\csname res@darcy2d/localFNO/64/IG/pearson\endcsname{0.993}
\expandafter\def\csname res@darcy2d/localFNO/64/IG/r2\endcsname{0.904}
\expandafter\def\csname res@darcy2d/localFNO/64/KernelSHAP/nrmse\endcsname{0.041}
\expandafter\def\csname res@darcy2d/localFNO/64/KernelSHAP/pearson\endcsname{0.999}
\expandafter\def\csname res@darcy2d/localFNO/64/KernelSHAP/r2\endcsname{0.915}
\expandafter\def\csname res@darcy2d/localFNO/64/OperatorSHAP/nrmse\endcsname{0.124}
\expandafter\def\csname res@darcy2d/localFNO/64/OperatorSHAP/pearson\endcsname{0.993}
\expandafter\def\csname res@darcy2d/localFNO/64/OperatorSHAP/r2\endcsname{0.904}
\expandafter\def\csname res@darcy2d/localFNO/64/RMSR/nrmse\endcsname{0.304}
\expandafter\def\csname res@darcy2d/localFNO/64/RMSR/pearson\endcsname{0.977}
\expandafter\def\csname res@darcy2d/localFNO/64/RMSR/r2\endcsname{0.854}
\expandafter\def\csname res@darcy2d/localFNO/64/Reference/nrmse\endcsname{0.000}
\expandafter\def\csname res@darcy2d/localFNO/64/Reference/pearson\endcsname{1.000}
\expandafter\def\csname res@darcy2d/localFNO/64/Reference/r2\endcsname{0.915}
\expandafter\def\csname res@darcy2dmean/localFNO/1024/IG/nrmse\endcsname{0.345}
\expandafter\def\csname res@darcy2dmean/localFNO/1024/IG/pearson\endcsname{0.955}
\expandafter\def\csname res@darcy2dmean/localFNO/1024/IG/r2\endcsname{0.987}
\expandafter\def\csname res@darcy2dmean/localFNO/1024/KernelSHAP/nrmse\endcsname{14.531}
\expandafter\def\csname res@darcy2dmean/localFNO/1024/KernelSHAP/pearson\endcsname{0.075}
\expandafter\def\csname res@darcy2dmean/localFNO/1024/KernelSHAP/r2\endcsname{-1.143}
\expandafter\def\csname res@darcy2dmean/localFNO/1024/OperatorSHAP/nrmse\endcsname{0.549}
\expandafter\def\csname res@darcy2dmean/localFNO/1024/OperatorSHAP/pearson\endcsname{0.855}
\expandafter\def\csname res@darcy2dmean/localFNO/1024/OperatorSHAP/r2\endcsname{0.985}
\expandafter\def\csname res@darcy2dmean/localFNO/1024/RMSR/nrmse\endcsname{1.045}
\expandafter\def\csname res@darcy2dmean/localFNO/1024/RMSR/pearson\endcsname{0.364}
\expandafter\def\csname res@darcy2dmean/localFNO/1024/RMSR/r2\endcsname{0.976}
\expandafter\def\csname res@darcy2dmean/localFNO/1024/Reference/nrmse\endcsname{0.000}
\expandafter\def\csname res@darcy2dmean/localFNO/1024/Reference/pearson\endcsname{1.000}
\expandafter\def\csname res@darcy2dmean/localFNO/1024/Reference/r2\endcsname{0.986}
\expandafter\def\csname res@darcy2dmean/localFNO/64/IG/nrmse\endcsname{0.059}
\expandafter\def\csname res@darcy2dmean/localFNO/64/IG/pearson\endcsname{0.998}
\expandafter\def\csname res@darcy2dmean/localFNO/64/IG/r2\endcsname{0.973}
\expandafter\def\csname res@darcy2dmean/localFNO/64/KernelSHAP/nrmse\endcsname{0.016}
\expandafter\def\csname res@darcy2dmean/localFNO/64/KernelSHAP/pearson\endcsname{1.000}
\expandafter\def\csname res@darcy2dmean/localFNO/64/KernelSHAP/r2\endcsname{0.975}
\expandafter\def\csname res@darcy2dmean/localFNO/64/OperatorSHAP/nrmse\endcsname{0.229}
\expandafter\def\csname res@darcy2dmean/localFNO/64/OperatorSHAP/pearson\endcsname{0.974}
\expandafter\def\csname res@darcy2dmean/localFNO/64/OperatorSHAP/r2\endcsname{0.963}
\expandafter\def\csname res@darcy2dmean/localFNO/64/RMSR/nrmse\endcsname{0.282}
\expandafter\def\csname res@darcy2dmean/localFNO/64/RMSR/pearson\endcsname{0.966}
\expandafter\def\csname res@darcy2dmean/localFNO/64/RMSR/r2\endcsname{0.962}
\expandafter\def\csname res@darcy2dmean/localFNO/64/Reference/nrmse\endcsname{0.000}
\expandafter\def\csname res@darcy2dmean/localFNO/64/Reference/pearson\endcsname{1.000}
\expandafter\def\csname res@darcy2dmean/localFNO/64/Reference/r2\endcsname{0.975}
\expandafter\def\csname res@darcy2dmid/localFNO/1024/IG/nrmse\endcsname{1.000}
\expandafter\def\csname res@darcy2dmid/localFNO/1024/IG/r2\endcsname{-1.146}
\expandafter\def\csname res@darcy2dmid/localFNO/1024/KernelSHAP/nrmse\endcsname{74.211}
\expandafter\def\csname res@darcy2dmid/localFNO/1024/KernelSHAP/pearson\endcsname{0.018}
\expandafter\def\csname res@darcy2dmid/localFNO/1024/KernelSHAP/r2\endcsname{-99.623}
\expandafter\def\csname res@darcy2dmid/localFNO/1024/OperatorSHAP/nrmse\endcsname{0.507}
\expandafter\def\csname res@darcy2dmid/localFNO/1024/OperatorSHAP/pearson\endcsname{0.905}
\expandafter\def\csname res@darcy2dmid/localFNO/1024/OperatorSHAP/r2\endcsname{0.846}
\expandafter\def\csname res@darcy2dmid/localFNO/1024/RMSR/nrmse\endcsname{1.024}
\expandafter\def\csname res@darcy2dmid/localFNO/1024/RMSR/pearson\endcsname{0.301}
\expandafter\def\csname res@darcy2dmid/localFNO/1024/RMSR/r2\endcsname{0.831}
\expandafter\def\csname res@darcy2dmid/localFNO/1024/Reference/nrmse\endcsname{0.000}
\expandafter\def\csname res@darcy2dmid/localFNO/1024/Reference/pearson\endcsname{1.000}
\expandafter\def\csname res@darcy2dmid/localFNO/1024/Reference/r2\endcsname{0.848}
\expandafter\def\csname res@darcy2dmid/localFNO/16/IG/nrmse\endcsname{1.000}
\expandafter\def\csname res@darcy2dmid/localFNO/16/IG/r2\endcsname{-1.011}
\expandafter\def\csname res@darcy2dmid/localFNO/16/KernelSHAP/nrmse\endcsname{0.081}
\expandafter\def\csname res@darcy2dmid/localFNO/16/KernelSHAP/pearson\endcsname{0.997}
\expandafter\def\csname res@darcy2dmid/localFNO/16/KernelSHAP/r2\endcsname{0.664}
\expandafter\def\csname res@darcy2dmid/localFNO/16/OperatorSHAP/nrmse\endcsname{0.134}
\expandafter\def\csname res@darcy2dmid/localFNO/16/OperatorSHAP/pearson\endcsname{0.991}
\expandafter\def\csname res@darcy2dmid/localFNO/16/OperatorSHAP/r2\endcsname{0.664}
\expandafter\def\csname res@darcy2dmid/localFNO/16/RMSR/nrmse\endcsname{0.054}
\expandafter\def\csname res@darcy2dmid/localFNO/16/RMSR/pearson\endcsname{0.999}
\expandafter\def\csname res@darcy2dmid/localFNO/16/RMSR/r2\endcsname{0.666}
\expandafter\def\csname res@darcy2dmid/localFNO/16/Reference/nrmse\endcsname{0.000}
\expandafter\def\csname res@darcy2dmid/localFNO/16/Reference/pearson\endcsname{1.000}
\expandafter\def\csname res@darcy2dmid/localFNO/16/Reference/r2\endcsname{0.668}
\expandafter\def\csname res@darcy2dmid/localFNO/256/IG/nrmse\endcsname{1.000}
\expandafter\def\csname res@darcy2dmid/localFNO/256/IG/r2\endcsname{-1.110}
\expandafter\def\csname res@darcy2dmid/localFNO/256/KernelSHAP/nrmse\endcsname{1.592}
\expandafter\def\csname res@darcy2dmid/localFNO/256/KernelSHAP/pearson\endcsname{0.480}
\expandafter\def\csname res@darcy2dmid/localFNO/256/KernelSHAP/r2\endcsname{0.682}
\expandafter\def\csname res@darcy2dmid/localFNO/256/OperatorSHAP/nrmse\endcsname{0.268}
\expandafter\def\csname res@darcy2dmid/localFNO/256/OperatorSHAP/pearson\endcsname{0.968}
\expandafter\def\csname res@darcy2dmid/localFNO/256/OperatorSHAP/r2\endcsname{0.816}
\expandafter\def\csname res@darcy2dmid/localFNO/256/RMSR/nrmse\endcsname{0.776}
\expandafter\def\csname res@darcy2dmid/localFNO/256/RMSR/pearson\endcsname{0.619}
\expandafter\def\csname res@darcy2dmid/localFNO/256/RMSR/r2\endcsname{0.782}
\expandafter\def\csname res@darcy2dmid/localFNO/256/Reference/nrmse\endcsname{0.000}
\expandafter\def\csname res@darcy2dmid/localFNO/256/Reference/pearson\endcsname{1.000}
\expandafter\def\csname res@darcy2dmid/localFNO/256/Reference/r2\endcsname{0.822}
\expandafter\def\csname res@darcy2dmid/localFNO/64/IG/nrmse\endcsname{1.000}
\expandafter\def\csname res@darcy2dmid/localFNO/64/IG/r2\endcsname{-1.125}
\expandafter\def\csname res@darcy2dmid/localFNO/64/KernelSHAP/nrmse\endcsname{0.315}
\expandafter\def\csname res@darcy2dmid/localFNO/64/KernelSHAP/pearson\endcsname{0.955}
\expandafter\def\csname res@darcy2dmid/localFNO/64/KernelSHAP/r2\endcsname{0.730}
\expandafter\def\csname res@darcy2dmid/localFNO/64/OperatorSHAP/nrmse\endcsname{0.208}
\expandafter\def\csname res@darcy2dmid/localFNO/64/OperatorSHAP/pearson\endcsname{0.981}
\expandafter\def\csname res@darcy2dmid/localFNO/64/OperatorSHAP/r2\endcsname{0.745}
\expandafter\def\csname res@darcy2dmid/localFNO/64/RMSR/nrmse\endcsname{0.286}
\expandafter\def\csname res@darcy2dmid/localFNO/64/RMSR/pearson\endcsname{0.965}
\expandafter\def\csname res@darcy2dmid/localFNO/64/RMSR/r2\endcsname{0.728}
\expandafter\def\csname res@darcy2dmid/localFNO/64/Reference/nrmse\endcsname{0.000}
\expandafter\def\csname res@darcy2dmid/localFNO/64/Reference/pearson\endcsname{1.000}
\expandafter\def\csname res@darcy2dmid/localFNO/64/Reference/r2\endcsname{0.755}
\expandafter\def\csname res@diffsorp1d/fno/1024/IG/nrmse\endcsname{0.342}
\expandafter\def\csname res@diffsorp1d/fno/1024/IG/pearson\endcsname{0.971}
\expandafter\def\csname res@diffsorp1d/fno/1024/IG/r2\endcsname{0.999}
\expandafter\def\csname res@diffsorp1d/fno/1024/KernelSHAP/nrmse\endcsname{7.074}
\expandafter\def\csname res@diffsorp1d/fno/1024/KernelSHAP/pearson\endcsname{0.174}
\expandafter\def\csname res@diffsorp1d/fno/1024/KernelSHAP/r2\endcsname{-0.164}
\expandafter\def\csname res@diffsorp1d/fno/1024/OperatorSHAP/nrmse\endcsname{0.766}
\expandafter\def\csname res@diffsorp1d/fno/1024/OperatorSHAP/pearson\endcsname{0.704}
\expandafter\def\csname res@diffsorp1d/fno/1024/OperatorSHAP/r2\endcsname{0.994}
\expandafter\def\csname res@diffsorp1d/fno/1024/RMSR/nrmse\endcsname{0.985}
\expandafter\def\csname res@diffsorp1d/fno/1024/RMSR/pearson\endcsname{0.403}
\expandafter\def\csname res@diffsorp1d/fno/1024/RMSR/r2\endcsname{0.956}
\expandafter\def\csname res@diffsorp1d/fno/1024/Reference/nrmse\endcsname{0.000}
\expandafter\def\csname res@diffsorp1d/fno/1024/Reference/pearson\endcsname{1.000}
\expandafter\def\csname res@diffsorp1d/fno/1024/Reference/r2\endcsname{0.995}
\expandafter\def\csname res@diffsorp1d/fno/128/IG/nrmse\endcsname{0.047}
\expandafter\def\csname res@diffsorp1d/fno/128/IG/pearson\endcsname{0.999}
\expandafter\def\csname res@diffsorp1d/fno/128/IG/r2\endcsname{0.997}
\expandafter\def\csname res@diffsorp1d/fno/128/KernelSHAP/nrmse\endcsname{0.037}
\expandafter\def\csname res@diffsorp1d/fno/128/KernelSHAP/pearson\endcsname{0.999}
\expandafter\def\csname res@diffsorp1d/fno/128/KernelSHAP/r2\endcsname{0.997}
\expandafter\def\csname res@diffsorp1d/fno/128/OperatorSHAP/nrmse\endcsname{0.373}
\expandafter\def\csname res@diffsorp1d/fno/128/OperatorSHAP/pearson\endcsname{0.931}
\expandafter\def\csname res@diffsorp1d/fno/128/OperatorSHAP/r2\endcsname{0.968}
\expandafter\def\csname res@diffsorp1d/fno/128/RMSR/nrmse\endcsname{0.479}
\expandafter\def\csname res@diffsorp1d/fno/128/RMSR/pearson\endcsname{0.920}
\expandafter\def\csname res@diffsorp1d/fno/128/RMSR/r2\endcsname{0.927}
\expandafter\def\csname res@diffsorp1d/fno/128/Reference/nrmse\endcsname{0.000}
\expandafter\def\csname res@diffsorp1d/fno/128/Reference/pearson\endcsname{1.000}
\expandafter\def\csname res@diffsorp1d/fno/128/Reference/r2\endcsname{0.997}
\expandafter\def\csname res@diffsorp1d/fno/256/IG/nrmse\endcsname{0.090}
\expandafter\def\csname res@diffsorp1d/fno/256/IG/pearson\endcsname{0.998}
\expandafter\def\csname res@diffsorp1d/fno/256/IG/r2\endcsname{0.998}
\expandafter\def\csname res@diffsorp1d/fno/256/KernelSHAP/nrmse\endcsname{0.092}
\expandafter\def\csname res@diffsorp1d/fno/256/KernelSHAP/pearson\endcsname{0.996}
\expandafter\def\csname res@diffsorp1d/fno/256/KernelSHAP/r2\endcsname{0.998}
\expandafter\def\csname res@diffsorp1d/fno/256/OperatorSHAP/nrmse\endcsname{0.444}
\expandafter\def\csname res@diffsorp1d/fno/256/OperatorSHAP/pearson\endcsname{0.913}
\expandafter\def\csname res@diffsorp1d/fno/256/OperatorSHAP/r2\endcsname{0.983}
\expandafter\def\csname res@diffsorp1d/fno/256/RMSR/nrmse\endcsname{0.707}
\expandafter\def\csname res@diffsorp1d/fno/256/RMSR/pearson\endcsname{0.769}
\expandafter\def\csname res@diffsorp1d/fno/256/RMSR/r2\endcsname{0.901}
\expandafter\def\csname res@diffsorp1d/fno/256/Reference/nrmse\endcsname{0.000}
\expandafter\def\csname res@diffsorp1d/fno/256/Reference/pearson\endcsname{1.000}
\expandafter\def\csname res@diffsorp1d/fno/256/Reference/r2\endcsname{0.998}
\expandafter\def\csname res@diffsorp1d/fno/32/IG/nrmse\endcsname{0.029}
\expandafter\def\csname res@diffsorp1d/fno/32/IG/pearson\endcsname{1.000}
\expandafter\def\csname res@diffsorp1d/fno/32/IG/r2\endcsname{0.985}
\expandafter\def\csname res@diffsorp1d/fno/32/KernelSHAP/nrmse\endcsname{0.005}
\expandafter\def\csname res@diffsorp1d/fno/32/KernelSHAP/pearson\endcsname{1.000}
\expandafter\def\csname res@diffsorp1d/fno/32/KernelSHAP/r2\endcsname{0.985}
\expandafter\def\csname res@diffsorp1d/fno/32/OperatorSHAP/nrmse\endcsname{0.278}
\expandafter\def\csname res@diffsorp1d/fno/32/OperatorSHAP/pearson\endcsname{0.960}
\expandafter\def\csname res@diffsorp1d/fno/32/OperatorSHAP/r2\endcsname{0.926}
\expandafter\def\csname res@diffsorp1d/fno/32/RMSR/nrmse\endcsname{0.086}
\expandafter\def\csname res@diffsorp1d/fno/32/RMSR/pearson\endcsname{0.998}
\expandafter\def\csname res@diffsorp1d/fno/32/RMSR/r2\endcsname{0.972}
\expandafter\def\csname res@diffsorp1d/fno/32/Reference/nrmse\endcsname{0.000}
\expandafter\def\csname res@diffsorp1d/fno/32/Reference/pearson\endcsname{1.000}
\expandafter\def\csname res@diffsorp1d/fno/32/Reference/r2\endcsname{0.985}
\expandafter\def\csname res@diffsorp1d/fno/512/IG/nrmse\endcsname{0.176}
\expandafter\def\csname res@diffsorp1d/fno/512/IG/pearson\endcsname{0.994}
\expandafter\def\csname res@diffsorp1d/fno/512/IG/r2\endcsname{0.998}
\expandafter\def\csname res@diffsorp1d/fno/512/KernelSHAP/nrmse\endcsname{1.856}
\expandafter\def\csname res@diffsorp1d/fno/512/KernelSHAP/pearson\endcsname{0.569}
\expandafter\def\csname res@diffsorp1d/fno/512/KernelSHAP/r2\endcsname{0.746}
\expandafter\def\csname res@diffsorp1d/fno/512/OperatorSHAP/nrmse\endcsname{0.582}
\expandafter\def\csname res@diffsorp1d/fno/512/OperatorSHAP/pearson\endcsname{0.822}
\expandafter\def\csname res@diffsorp1d/fno/512/OperatorSHAP/r2\endcsname{0.991}
\expandafter\def\csname res@diffsorp1d/fno/512/RMSR/nrmse\endcsname{0.875}
\expandafter\def\csname res@diffsorp1d/fno/512/RMSR/pearson\endcsname{0.521}
\expandafter\def\csname res@diffsorp1d/fno/512/RMSR/r2\endcsname{0.923}
\expandafter\def\csname res@diffsorp1d/fno/512/Reference/nrmse\endcsname{0.000}
\expandafter\def\csname res@diffsorp1d/fno/512/Reference/pearson\endcsname{1.000}
\expandafter\def\csname res@diffsorp1d/fno/512/Reference/r2\endcsname{0.996}
\expandafter\def\csname res@diffsorp1d/fno/64/IG/nrmse\endcsname{0.038}
\expandafter\def\csname res@diffsorp1d/fno/64/IG/pearson\endcsname{1.000}
\expandafter\def\csname res@diffsorp1d/fno/64/IG/r2\endcsname{0.990}
\expandafter\def\csname res@diffsorp1d/fno/64/KernelSHAP/nrmse\endcsname{0.014}
\expandafter\def\csname res@diffsorp1d/fno/64/KernelSHAP/pearson\endcsname{1.000}
\expandafter\def\csname res@diffsorp1d/fno/64/KernelSHAP/r2\endcsname{0.991}
\expandafter\def\csname res@diffsorp1d/fno/64/OperatorSHAP/nrmse\endcsname{0.327}
\expandafter\def\csname res@diffsorp1d/fno/64/OperatorSHAP/pearson\endcsname{0.945}
\expandafter\def\csname res@diffsorp1d/fno/64/OperatorSHAP/r2\endcsname{0.947}
\expandafter\def\csname res@diffsorp1d/fno/64/RMSR/nrmse\endcsname{0.270}
\expandafter\def\csname res@diffsorp1d/fno/64/RMSR/pearson\endcsname{0.979}
\expandafter\def\csname res@diffsorp1d/fno/64/RMSR/r2\endcsname{0.964}
\expandafter\def\csname res@diffsorp1d/fno/64/Reference/nrmse\endcsname{0.000}
\expandafter\def\csname res@diffsorp1d/fno/64/Reference/pearson\endcsname{1.000}
\expandafter\def\csname res@diffsorp1d/fno/64/Reference/r2\endcsname{0.991}
\expandafter\def\csname res@diffsorp1dmean/fno/1024/IG/nrmse\endcsname{0.102}
\expandafter\def\csname res@diffsorp1dmean/fno/1024/IG/pearson\endcsname{0.995}
\expandafter\def\csname res@diffsorp1dmean/fno/1024/IG/r2\endcsname{1.000}
\expandafter\def\csname res@diffsorp1dmean/fno/1024/KernelSHAP/nrmse\endcsname{6.925}
\expandafter\def\csname res@diffsorp1dmean/fno/1024/KernelSHAP/pearson\endcsname{0.131}
\expandafter\def\csname res@diffsorp1dmean/fno/1024/KernelSHAP/r2\endcsname{-3.924}
\expandafter\def\csname res@diffsorp1dmean/fno/1024/OperatorSHAP/nrmse\endcsname{0.838}
\expandafter\def\csname res@diffsorp1dmean/fno/1024/OperatorSHAP/pearson\endcsname{0.571}
\expandafter\def\csname res@diffsorp1dmean/fno/1024/OperatorSHAP/r2\endcsname{0.915}
\expandafter\def\csname res@diffsorp1dmean/fno/1024/RMSR/nrmse\endcsname{0.137}
\expandafter\def\csname res@diffsorp1dmean/fno/1024/RMSR/pearson\endcsname{0.991}
\expandafter\def\csname res@diffsorp1dmean/fno/1024/RMSR/r2\endcsname{0.995}
\expandafter\def\csname res@diffsorp1dmean/fno/1024/Reference/nrmse\endcsname{0.000}
\expandafter\def\csname res@diffsorp1dmean/fno/1024/Reference/pearson\endcsname{1.000}
\expandafter\def\csname res@diffsorp1dmean/fno/1024/Reference/r2\endcsname{0.999}
\expandafter\def\csname res@diffsorp1dmean/fno/64/IG/nrmse\endcsname{0.010}
\expandafter\def\csname res@diffsorp1dmean/fno/64/IG/pearson\endcsname{1.000}
\expandafter\def\csname res@diffsorp1dmean/fno/64/IG/r2\endcsname{0.999}
\expandafter\def\csname res@diffsorp1dmean/fno/64/KernelSHAP/nrmse\endcsname{0.009}
\expandafter\def\csname res@diffsorp1dmean/fno/64/KernelSHAP/pearson\endcsname{1.000}
\expandafter\def\csname res@diffsorp1dmean/fno/64/KernelSHAP/r2\endcsname{0.999}
\expandafter\def\csname res@diffsorp1dmean/fno/64/OperatorSHAP/nrmse\endcsname{0.089}
\expandafter\def\csname res@diffsorp1dmean/fno/64/OperatorSHAP/pearson\endcsname{0.996}
\expandafter\def\csname res@diffsorp1dmean/fno/64/OperatorSHAP/r2\endcsname{0.996}
\expandafter\def\csname res@diffsorp1dmean/fno/64/RMSR/nrmse\endcsname{0.271}
\expandafter\def\csname res@diffsorp1dmean/fno/64/RMSR/pearson\endcsname{0.974}
\expandafter\def\csname res@diffsorp1dmean/fno/64/RMSR/r2\endcsname{0.977}
\expandafter\def\csname res@diffsorp1dmean/fno/64/Reference/nrmse\endcsname{0.000}
\expandafter\def\csname res@diffsorp1dmean/fno/64/Reference/pearson\endcsname{1.000}
\expandafter\def\csname res@diffsorp1dmean/fno/64/Reference/r2\endcsname{0.999}
\expandafter\def\csname res@diffsorp1dmid/fno/1024/IG/nrmse\endcsname{1.000}
\expandafter\def\csname res@diffsorp1dmid/fno/1024/IG/r2\endcsname{-0.984}
\expandafter\def\csname res@diffsorp1dmid/fno/1024/KernelSHAP/nrmse\endcsname{78.799}
\expandafter\def\csname res@diffsorp1dmid/fno/1024/KernelSHAP/pearson\endcsname{0.004}
\expandafter\def\csname res@diffsorp1dmid/fno/1024/KernelSHAP/r2\endcsname{-105.891}
\expandafter\def\csname res@diffsorp1dmid/fno/1024/OperatorSHAP/nrmse\endcsname{0.340}
\expandafter\def\csname res@diffsorp1dmid/fno/1024/OperatorSHAP/pearson\endcsname{0.960}
\expandafter\def\csname res@diffsorp1dmid/fno/1024/OperatorSHAP/r2\endcsname{0.893}
\expandafter\def\csname res@diffsorp1dmid/fno/1024/RMSR/nrmse\endcsname{1.073}
\expandafter\def\csname res@diffsorp1dmid/fno/1024/RMSR/pearson\endcsname{0.209}
\expandafter\def\csname res@diffsorp1dmid/fno/1024/RMSR/r2\endcsname{0.867}
\expandafter\def\csname res@diffsorp1dmid/fno/1024/Reference/nrmse\endcsname{0.000}
\expandafter\def\csname res@diffsorp1dmid/fno/1024/Reference/pearson\endcsname{1.000}
\expandafter\def\csname res@diffsorp1dmid/fno/1024/Reference/r2\endcsname{0.894}
\expandafter\def\csname res@diffsorp1dmid/fno/16/IG/nrmse\endcsname{1.000}
\expandafter\def\csname res@diffsorp1dmid/fno/16/IG/r2\endcsname{-0.982}
\expandafter\def\csname res@diffsorp1dmid/fno/16/KernelSHAP/nrmse\endcsname{0.060}
\expandafter\def\csname res@diffsorp1dmid/fno/16/KernelSHAP/pearson\endcsname{0.998}
\expandafter\def\csname res@diffsorp1dmid/fno/16/KernelSHAP/r2\endcsname{0.799}
\expandafter\def\csname res@diffsorp1dmid/fno/16/OperatorSHAP/nrmse\endcsname{0.051}
\expandafter\def\csname res@diffsorp1dmid/fno/16/OperatorSHAP/pearson\endcsname{0.999}
\expandafter\def\csname res@diffsorp1dmid/fno/16/OperatorSHAP/r2\endcsname{0.799}
\expandafter\def\csname res@diffsorp1dmid/fno/16/RMSR/nrmse\endcsname{0.000}
\expandafter\def\csname res@diffsorp1dmid/fno/16/RMSR/pearson\endcsname{1.000}
\expandafter\def\csname res@diffsorp1dmid/fno/16/RMSR/r2\endcsname{0.800}
\expandafter\def\csname res@diffsorp1dmid/fno/16/Reference/nrmse\endcsname{0.000}
\expandafter\def\csname res@diffsorp1dmid/fno/16/Reference/pearson\endcsname{1.000}
\expandafter\def\csname res@diffsorp1dmid/fno/16/Reference/r2\endcsname{0.800}
\expandafter\def\csname res@diffsorp1dmid/fno/256/IG/nrmse\endcsname{1.000}
\expandafter\def\csname res@diffsorp1dmid/fno/256/IG/r2\endcsname{-0.983}
\expandafter\def\csname res@diffsorp1dmid/fno/256/KernelSHAP/nrmse\endcsname{1.867}
\expandafter\def\csname res@diffsorp1dmid/fno/256/KernelSHAP/pearson\endcsname{0.450}
\expandafter\def\csname res@diffsorp1dmid/fno/256/KernelSHAP/r2\endcsname{0.628}
\expandafter\def\csname res@diffsorp1dmid/fno/256/OperatorSHAP/nrmse\endcsname{0.156}
\expandafter\def\csname res@diffsorp1dmid/fno/256/OperatorSHAP/pearson\endcsname{0.990}
\expandafter\def\csname res@diffsorp1dmid/fno/256/OperatorSHAP/r2\endcsname{0.860}
\expandafter\def\csname res@diffsorp1dmid/fno/256/RMSR/nrmse\endcsname{0.816}
\expandafter\def\csname res@diffsorp1dmid/fno/256/RMSR/pearson\endcsname{0.582}
\expandafter\def\csname res@diffsorp1dmid/fno/256/RMSR/r2\endcsname{0.793}
\expandafter\def\csname res@diffsorp1dmid/fno/256/Reference/nrmse\endcsname{0.000}
\expandafter\def\csname res@diffsorp1dmid/fno/256/Reference/pearson\endcsname{1.000}
\expandafter\def\csname res@diffsorp1dmid/fno/256/Reference/r2\endcsname{0.862}
\expandafter\def\csname res@diffsorp1dmid/fno/64/IG/nrmse\endcsname{1.000}
\expandafter\def\csname res@diffsorp1dmid/fno/64/IG/r2\endcsname{-0.974}
\expandafter\def\csname res@diffsorp1dmid/fno/64/KernelSHAP/nrmse\endcsname{0.380}
\expandafter\def\csname res@diffsorp1dmid/fno/64/KernelSHAP/pearson\endcsname{0.931}
\expandafter\def\csname res@diffsorp1dmid/fno/64/KernelSHAP/r2\endcsname{0.773}
\expandafter\def\csname res@diffsorp1dmid/fno/64/OperatorSHAP/nrmse\endcsname{0.122}
\expandafter\def\csname res@diffsorp1dmid/fno/64/OperatorSHAP/pearson\endcsname{0.993}
\expandafter\def\csname res@diffsorp1dmid/fno/64/OperatorSHAP/r2\endcsname{0.804}
\expandafter\def\csname res@diffsorp1dmid/fno/64/RMSR/nrmse\endcsname{0.257}
\expandafter\def\csname res@diffsorp1dmid/fno/64/RMSR/pearson\endcsname{0.972}
\expandafter\def\csname res@diffsorp1dmid/fno/64/RMSR/r2\endcsname{0.783}
\expandafter\def\csname res@diffsorp1dmid/fno/64/Reference/nrmse\endcsname{0.000}
\expandafter\def\csname res@diffsorp1dmid/fno/64/Reference/pearson\endcsname{1.000}
\expandafter\def\csname res@diffsorp1dmid/fno/64/Reference/r2\endcsname{0.806}
\expandafter\def\csname res@heat1d/fno/1024/IG/nrmse\endcsname{0.373}
\expandafter\def\csname res@heat1d/fno/1024/IG/pearson\endcsname{0.945}
\expandafter\def\csname res@heat1d/fno/1024/IG/r2\endcsname{1.000}
\expandafter\def\csname res@heat1d/fno/1024/KernelSHAP/nrmse\endcsname{18.008}
\expandafter\def\csname res@heat1d/fno/1024/KernelSHAP/pearson\endcsname{0.067}
\expandafter\def\csname res@heat1d/fno/1024/KernelSHAP/r2\endcsname{-0.610}
\expandafter\def\csname res@heat1d/fno/1024/OperatorSHAP/nrmse\endcsname{0.417}
\expandafter\def\csname res@heat1d/fno/1024/OperatorSHAP/pearson\endcsname{0.925}
\expandafter\def\csname res@heat1d/fno/1024/OperatorSHAP/r2\endcsname{1.000}
\expandafter\def\csname res@heat1d/fno/1024/RMSR/nrmse\endcsname{1.485}
\expandafter\def\csname res@heat1d/fno/1024/RMSR/pearson\endcsname{0.184}
\expandafter\def\csname res@heat1d/fno/1024/RMSR/r2\endcsname{0.994}
\expandafter\def\csname res@heat1d/fno/1024/Reference/nrmse\endcsname{0.000}
\expandafter\def\csname res@heat1d/fno/1024/Reference/pearson\endcsname{1.000}
\expandafter\def\csname res@heat1d/fno/1024/Reference/r2\endcsname{1.000}
\expandafter\def\csname res@heat1d/fno/128/IG/nrmse\endcsname{0.030}
\expandafter\def\csname res@heat1d/fno/128/IG/pearson\endcsname{1.000}
\expandafter\def\csname res@heat1d/fno/128/IG/r2\endcsname{1.000}
\expandafter\def\csname res@heat1d/fno/128/KernelSHAP/nrmse\endcsname{0.024}
\expandafter\def\csname res@heat1d/fno/128/KernelSHAP/pearson\endcsname{1.000}
\expandafter\def\csname res@heat1d/fno/128/KernelSHAP/r2\endcsname{1.000}
\expandafter\def\csname res@heat1d/fno/128/OperatorSHAP/nrmse\endcsname{0.057}
\expandafter\def\csname res@heat1d/fno/128/OperatorSHAP/pearson\endcsname{0.998}
\expandafter\def\csname res@heat1d/fno/128/OperatorSHAP/r2\endcsname{1.000}
\expandafter\def\csname res@heat1d/fno/128/RMSR/nrmse\endcsname{0.490}
\expandafter\def\csname res@heat1d/fno/128/RMSR/pearson\endcsname{0.882}
\expandafter\def\csname res@heat1d/fno/128/RMSR/r2\endcsname{0.993}
\expandafter\def\csname res@heat1d/fno/128/Reference/nrmse\endcsname{0.000}
\expandafter\def\csname res@heat1d/fno/128/Reference/pearson\endcsname{1.000}
\expandafter\def\csname res@heat1d/fno/128/Reference/r2\endcsname{1.000}
\expandafter\def\csname res@heat1d/fno/256/IG/nrmse\endcsname{0.099}
\expandafter\def\csname res@heat1d/fno/256/IG/pearson\endcsname{0.997}
\expandafter\def\csname res@heat1d/fno/256/IG/r2\endcsname{1.000}
\expandafter\def\csname res@heat1d/fno/256/KernelSHAP/nrmse\endcsname{0.097}
\expandafter\def\csname res@heat1d/fno/256/KernelSHAP/pearson\endcsname{0.997}
\expandafter\def\csname res@heat1d/fno/256/KernelSHAP/r2\endcsname{1.000}
\expandafter\def\csname res@heat1d/fno/256/OperatorSHAP/nrmse\endcsname{0.118}
\expandafter\def\csname res@heat1d/fno/256/OperatorSHAP/pearson\endcsname{0.994}
\expandafter\def\csname res@heat1d/fno/256/OperatorSHAP/r2\endcsname{1.000}
\expandafter\def\csname res@heat1d/fno/256/RMSR/nrmse\endcsname{0.830}
\expandafter\def\csname res@heat1d/fno/256/RMSR/pearson\endcsname{0.545}
\expandafter\def\csname res@heat1d/fno/256/RMSR/r2\endcsname{0.990}
\expandafter\def\csname res@heat1d/fno/256/Reference/nrmse\endcsname{0.000}
\expandafter\def\csname res@heat1d/fno/256/Reference/pearson\endcsname{1.000}
\expandafter\def\csname res@heat1d/fno/256/Reference/r2\endcsname{1.000}
\expandafter\def\csname res@heat1d/fno/32/IG/nrmse\endcsname{0.011}
\expandafter\def\csname res@heat1d/fno/32/IG/pearson\endcsname{1.000}
\expandafter\def\csname res@heat1d/fno/32/IG/r2\endcsname{1.000}
\expandafter\def\csname res@heat1d/fno/32/KernelSHAP/nrmse\endcsname{0.002}
\expandafter\def\csname res@heat1d/fno/32/KernelSHAP/pearson\endcsname{1.000}
\expandafter\def\csname res@heat1d/fno/32/KernelSHAP/r2\endcsname{1.000}
\expandafter\def\csname res@heat1d/fno/32/OperatorSHAP/nrmse\endcsname{0.021}
\expandafter\def\csname res@heat1d/fno/32/OperatorSHAP/pearson\endcsname{1.000}
\expandafter\def\csname res@heat1d/fno/32/OperatorSHAP/r2\endcsname{1.000}
\expandafter\def\csname res@heat1d/fno/32/RMSR/nrmse\endcsname{0.082}
\expandafter\def\csname res@heat1d/fno/32/RMSR/pearson\endcsname{0.997}
\expandafter\def\csname res@heat1d/fno/32/RMSR/r2\endcsname{0.999}
\expandafter\def\csname res@heat1d/fno/32/Reference/nrmse\endcsname{0.000}
\expandafter\def\csname res@heat1d/fno/32/Reference/pearson\endcsname{1.000}
\expandafter\def\csname res@heat1d/fno/32/Reference/r2\endcsname{1.000}
\expandafter\def\csname res@heat1d/fno/512/IG/nrmse\endcsname{0.203}
\expandafter\def\csname res@heat1d/fno/512/IG/pearson\endcsname{0.986}
\expandafter\def\csname res@heat1d/fno/512/IG/r2\endcsname{1.000}
\expandafter\def\csname res@heat1d/fno/512/KernelSHAP/nrmse\endcsname{1.984}
\expandafter\def\csname res@heat1d/fno/512/KernelSHAP/pearson\endcsname{0.424}
\expandafter\def\csname res@heat1d/fno/512/KernelSHAP/r2\endcsname{0.966}
\expandafter\def\csname res@heat1d/fno/512/OperatorSHAP/nrmse\endcsname{0.234}
\expandafter\def\csname res@heat1d/fno/512/OperatorSHAP/pearson\endcsname{0.978}
\expandafter\def\csname res@heat1d/fno/512/OperatorSHAP/r2\endcsname{1.000}
\expandafter\def\csname res@heat1d/fno/512/RMSR/nrmse\endcsname{1.223}
\expandafter\def\csname res@heat1d/fno/512/RMSR/pearson\endcsname{0.234}
\expandafter\def\csname res@heat1d/fno/512/RMSR/r2\endcsname{0.990}
\expandafter\def\csname res@heat1d/fno/512/Reference/nrmse\endcsname{0.000}
\expandafter\def\csname res@heat1d/fno/512/Reference/pearson\endcsname{1.000}
\expandafter\def\csname res@heat1d/fno/512/Reference/r2\endcsname{1.000}
\expandafter\def\csname res@heat1d/fno/64/IG/nrmse\endcsname{0.017}
\expandafter\def\csname res@heat1d/fno/64/IG/pearson\endcsname{1.000}
\expandafter\def\csname res@heat1d/fno/64/IG/r2\endcsname{1.000}
\expandafter\def\csname res@heat1d/fno/64/KernelSHAP/nrmse\endcsname{0.006}
\expandafter\def\csname res@heat1d/fno/64/KernelSHAP/pearson\endcsname{1.000}
\expandafter\def\csname res@heat1d/fno/64/KernelSHAP/r2\endcsname{1.000}
\expandafter\def\csname res@heat1d/fno/64/OperatorSHAP/nrmse\endcsname{0.036}
\expandafter\def\csname res@heat1d/fno/64/OperatorSHAP/pearson\endcsname{0.999}
\expandafter\def\csname res@heat1d/fno/64/OperatorSHAP/r2\endcsname{1.000}
\expandafter\def\csname res@heat1d/fno/64/RMSR/nrmse\endcsname{0.254}
\expandafter\def\csname res@heat1d/fno/64/RMSR/pearson\endcsname{0.975}
\expandafter\def\csname res@heat1d/fno/64/RMSR/r2\endcsname{0.996}
\expandafter\def\csname res@heat1d/fno/64/Reference/nrmse\endcsname{0.000}
\expandafter\def\csname res@heat1d/fno/64/Reference/pearson\endcsname{1.000}
\expandafter\def\csname res@heat1d/fno/64/Reference/r2\endcsname{1.000}
\expandafter\def\csname res@heat1dmid/fno/1024/IG/nrmse\endcsname{1.000}
\expandafter\def\csname res@heat1dmid/fno/1024/IG/r2\endcsname{-1.000}
\expandafter\def\csname res@heat1dmid/fno/1024/KernelSHAP/nrmse\endcsname{165.848}
\expandafter\def\csname res@heat1dmid/fno/1024/KernelSHAP/pearson\endcsname{0.008}
\expandafter\def\csname res@heat1dmid/fno/1024/KernelSHAP/r2\endcsname{-98.764}
\expandafter\def\csname res@heat1dmid/fno/1024/OperatorSHAP/nrmse\endcsname{0.624}
\expandafter\def\csname res@heat1dmid/fno/1024/OperatorSHAP/pearson\endcsname{0.799}
\expandafter\def\csname res@heat1dmid/fno/1024/OperatorSHAP/r2\endcsname{0.899}
\expandafter\def\csname res@heat1dmid/fno/1024/RMSR/nrmse\endcsname{2.398}
\expandafter\def\csname res@heat1dmid/fno/1024/RMSR/pearson\endcsname{0.027}
\expandafter\def\csname res@heat1dmid/fno/1024/RMSR/r2\endcsname{0.886}
\expandafter\def\csname res@heat1dmid/fno/1024/Reference/nrmse\endcsname{0.000}
\expandafter\def\csname res@heat1dmid/fno/1024/Reference/pearson\endcsname{1.000}
\expandafter\def\csname res@heat1dmid/fno/1024/Reference/r2\endcsname{0.898}
\expandafter\def\csname res@heat1dmid/fno/16/IG/nrmse\endcsname{1.000}
\expandafter\def\csname res@heat1dmid/fno/16/IG/r2\endcsname{-1.000}
\expandafter\def\csname res@heat1dmid/fno/16/KernelSHAP/nrmse\endcsname{0.127}
\expandafter\def\csname res@heat1dmid/fno/16/KernelSHAP/pearson\endcsname{0.991}
\expandafter\def\csname res@heat1dmid/fno/16/KernelSHAP/r2\endcsname{0.778}
\expandafter\def\csname res@heat1dmid/fno/16/OperatorSHAP/nrmse\endcsname{0.170}
\expandafter\def\csname res@heat1dmid/fno/16/OperatorSHAP/pearson\endcsname{0.986}
\expandafter\def\csname res@heat1dmid/fno/16/OperatorSHAP/r2\endcsname{0.775}
\expandafter\def\csname res@heat1dmid/fno/16/RMSR/nrmse\endcsname{0.040}
\expandafter\def\csname res@heat1dmid/fno/16/RMSR/pearson\endcsname{0.999}
\expandafter\def\csname res@heat1dmid/fno/16/RMSR/r2\endcsname{0.781}
\expandafter\def\csname res@heat1dmid/fno/16/Reference/nrmse\endcsname{0.000}
\expandafter\def\csname res@heat1dmid/fno/16/Reference/pearson\endcsname{1.000}
\expandafter\def\csname res@heat1dmid/fno/16/Reference/r2\endcsname{0.781}
\expandafter\def\csname res@heat1dmid/fno/256/IG/nrmse\endcsname{1.000}
\expandafter\def\csname res@heat1dmid/fno/256/IG/r2\endcsname{-1.000}
\expandafter\def\csname res@heat1dmid/fno/256/KernelSHAP/nrmse\endcsname{4.663}
\expandafter\def\csname res@heat1dmid/fno/256/KernelSHAP/pearson\endcsname{0.121}
\expandafter\def\csname res@heat1dmid/fno/256/KernelSHAP/r2\endcsname{0.667}
\expandafter\def\csname res@heat1dmid/fno/256/OperatorSHAP/nrmse\endcsname{0.231}
\expandafter\def\csname res@heat1dmid/fno/256/OperatorSHAP/pearson\endcsname{0.975}
\expandafter\def\csname res@heat1dmid/fno/256/OperatorSHAP/r2\endcsname{0.875}
\expandafter\def\csname res@heat1dmid/fno/256/RMSR/nrmse\endcsname{1.522}
\expandafter\def\csname res@heat1dmid/fno/256/RMSR/pearson\endcsname{0.131}
\expandafter\def\csname res@heat1dmid/fno/256/RMSR/r2\endcsname{0.853}
\expandafter\def\csname res@heat1dmid/fno/256/Reference/nrmse\endcsname{0.000}
\expandafter\def\csname res@heat1dmid/fno/256/Reference/pearson\endcsname{1.000}
\expandafter\def\csname res@heat1dmid/fno/256/Reference/r2\endcsname{0.876}
\expandafter\def\csname res@heat1dmid/fno/64/IG/nrmse\endcsname{1.000}
\expandafter\def\csname res@heat1dmid/fno/64/IG/r2\endcsname{-0.999}
\expandafter\def\csname res@heat1dmid/fno/64/KernelSHAP/nrmse\endcsname{0.896}
\expandafter\def\csname res@heat1dmid/fno/64/KernelSHAP/pearson\endcsname{0.727}
\expandafter\def\csname res@heat1dmid/fno/64/KernelSHAP/r2\endcsname{0.799}
\expandafter\def\csname res@heat1dmid/fno/64/OperatorSHAP/nrmse\endcsname{0.141}
\expandafter\def\csname res@heat1dmid/fno/64/OperatorSHAP/pearson\endcsname{0.992}
\expandafter\def\csname res@heat1dmid/fno/64/OperatorSHAP/r2\endcsname{0.836}
\expandafter\def\csname res@heat1dmid/fno/64/RMSR/nrmse\endcsname{0.611}
\expandafter\def\csname res@heat1dmid/fno/64/RMSR/pearson\endcsname{0.757}
\expandafter\def\csname res@heat1dmid/fno/64/RMSR/r2\endcsname{0.818}
\expandafter\def\csname res@heat1dmid/fno/64/Reference/nrmse\endcsname{0.000}
\expandafter\def\csname res@heat1dmid/fno/64/Reference/pearson\endcsname{1.000}
\expandafter\def\csname res@heat1dmid/fno/64/Reference/r2\endcsname{0.837}
\expandafter\def\csname res@heat2d/localFNO/1024/IG/nrmse\endcsname{0.394}
\expandafter\def\csname res@heat2d/localFNO/1024/IG/pearson\endcsname{0.933}
\expandafter\def\csname res@heat2d/localFNO/1024/IG/r2\endcsname{1.000}
\expandafter\def\csname res@heat2d/localFNO/1024/KernelSHAP/nrmse\endcsname{13.740}
\expandafter\def\csname res@heat2d/localFNO/1024/KernelSHAP/pearson\endcsname{0.082}
\expandafter\def\csname res@heat2d/localFNO/1024/KernelSHAP/r2\endcsname{-0.023}
\expandafter\def\csname res@heat2d/localFNO/1024/OperatorSHAP/nrmse\endcsname{0.391}
\expandafter\def\csname res@heat2d/localFNO/1024/OperatorSHAP/pearson\endcsname{0.929}
\expandafter\def\csname res@heat2d/localFNO/1024/OperatorSHAP/r2\endcsname{1.000}
\expandafter\def\csname res@heat2d/localFNO/1024/RMSR/nrmse\endcsname{1.298}
\expandafter\def\csname res@heat2d/localFNO/1024/RMSR/pearson\endcsname{0.198}
\expandafter\def\csname res@heat2d/localFNO/1024/RMSR/r2\endcsname{0.993}
\expandafter\def\csname res@heat2d/localFNO/1024/Reference/nrmse\endcsname{0.000}
\expandafter\def\csname res@heat2d/localFNO/1024/Reference/pearson\endcsname{1.000}
\expandafter\def\csname res@heat2d/localFNO/1024/Reference/r2\endcsname{0.999}
\expandafter\def\csname res@heat2d/localFNO/256/IG/nrmse\endcsname{0.117}
\expandafter\def\csname res@heat2d/localFNO/256/IG/pearson\endcsname{0.994}
\expandafter\def\csname res@heat2d/localFNO/256/IG/r2\endcsname{1.000}
\expandafter\def\csname res@heat2d/localFNO/256/KernelSHAP/nrmse\endcsname{0.117}
\expandafter\def\csname res@heat2d/localFNO/256/KernelSHAP/pearson\endcsname{0.994}
\expandafter\def\csname res@heat2d/localFNO/256/KernelSHAP/r2\endcsname{1.000}
\expandafter\def\csname res@heat2d/localFNO/256/OperatorSHAP/nrmse\endcsname{0.129}
\expandafter\def\csname res@heat2d/localFNO/256/OperatorSHAP/pearson\endcsname{0.993}
\expandafter\def\csname res@heat2d/localFNO/256/OperatorSHAP/r2\endcsname{1.000}
\expandafter\def\csname res@heat2d/localFNO/256/RMSR/nrmse\endcsname{0.721}
\expandafter\def\csname res@heat2d/localFNO/256/RMSR/pearson\endcsname{0.666}
\expandafter\def\csname res@heat2d/localFNO/256/RMSR/r2\endcsname{0.988}
\expandafter\def\csname res@heat2d/localFNO/256/Reference/nrmse\endcsname{0.000}
\expandafter\def\csname res@heat2d/localFNO/256/Reference/pearson\endcsname{1.000}
\expandafter\def\csname res@heat2d/localFNO/256/Reference/r2\endcsname{1.000}
\expandafter\def\csname res@heat2d/localFNO/64/IG/nrmse\endcsname{0.010}
\expandafter\def\csname res@heat2d/localFNO/64/IG/pearson\endcsname{1.000}
\expandafter\def\csname res@heat2d/localFNO/64/IG/r2\endcsname{1.000}
\expandafter\def\csname res@heat2d/localFNO/64/KernelSHAP/nrmse\endcsname{0.009}
\expandafter\def\csname res@heat2d/localFNO/64/KernelSHAP/pearson\endcsname{1.000}
\expandafter\def\csname res@heat2d/localFNO/64/KernelSHAP/r2\endcsname{1.000}
\expandafter\def\csname res@heat2d/localFNO/64/OperatorSHAP/nrmse\endcsname{0.034}
\expandafter\def\csname res@heat2d/localFNO/64/OperatorSHAP/pearson\endcsname{0.999}
\expandafter\def\csname res@heat2d/localFNO/64/OperatorSHAP/r2\endcsname{1.000}
\expandafter\def\csname res@heat2d/localFNO/64/RMSR/nrmse\endcsname{0.314}
\expandafter\def\csname res@heat2d/localFNO/64/RMSR/pearson\endcsname{0.963}
\expandafter\def\csname res@heat2d/localFNO/64/RMSR/r2\endcsname{0.987}
\expandafter\def\csname res@heat2d/localFNO/64/Reference/nrmse\endcsname{0.000}
\expandafter\def\csname res@heat2d/localFNO/64/Reference/pearson\endcsname{1.000}
\expandafter\def\csname res@heat2d/localFNO/64/Reference/r2\endcsname{1.000}
\expandafter\def\csname res@heat2dho1024/localFNO/1024/IG/nrmse\endcsname{0.394}
\expandafter\def\csname res@heat2dho1024/localFNO/1024/IG/pearson\endcsname{0.933}
\expandafter\def\csname res@heat2dho1024/localFNO/1024/IG/r2\endcsname{1.000}
\expandafter\def\csname res@heat2dho1024/localFNO/1024/KernelSHAP/nrmse\endcsname{13.740}
\expandafter\def\csname res@heat2dho1024/localFNO/1024/KernelSHAP/pearson\endcsname{0.082}
\expandafter\def\csname res@heat2dho1024/localFNO/1024/KernelSHAP/r2\endcsname{-0.023}
\expandafter\def\csname res@heat2dho1024/localFNO/1024/OperatorSHAP/nrmse\endcsname{0.763}
\expandafter\def\csname res@heat2dho1024/localFNO/1024/OperatorSHAP/pearson\endcsname{0.793}
\expandafter\def\csname res@heat2dho1024/localFNO/1024/OperatorSHAP/r2\endcsname{0.998}
\expandafter\def\csname res@heat2dho1024/localFNO/1024/RMSR/nrmse\endcsname{1.298}
\expandafter\def\csname res@heat2dho1024/localFNO/1024/RMSR/pearson\endcsname{0.198}
\expandafter\def\csname res@heat2dho1024/localFNO/1024/RMSR/r2\endcsname{0.993}
\expandafter\def\csname res@heat2dho1024/localFNO/1024/Reference/nrmse\endcsname{0.000}
\expandafter\def\csname res@heat2dho1024/localFNO/1024/Reference/pearson\endcsname{1.000}
\expandafter\def\csname res@heat2dho1024/localFNO/1024/Reference/r2\endcsname{0.999}
\expandafter\def\csname res@heat2dho256/localFNO/256/IG/nrmse\endcsname{0.117}
\expandafter\def\csname res@heat2dho256/localFNO/256/IG/pearson\endcsname{0.994}
\expandafter\def\csname res@heat2dho256/localFNO/256/IG/r2\endcsname{1.000}
\expandafter\def\csname res@heat2dho256/localFNO/256/KernelSHAP/nrmse\endcsname{0.117}
\expandafter\def\csname res@heat2dho256/localFNO/256/KernelSHAP/pearson\endcsname{0.994}
\expandafter\def\csname res@heat2dho256/localFNO/256/KernelSHAP/r2\endcsname{1.000}
\expandafter\def\csname res@heat2dho256/localFNO/256/OperatorSHAP/nrmse\endcsname{0.203}
\expandafter\def\csname res@heat2dho256/localFNO/256/OperatorSHAP/pearson\endcsname{0.980}
\expandafter\def\csname res@heat2dho256/localFNO/256/OperatorSHAP/r2\endcsname{0.999}
\expandafter\def\csname res@heat2dho256/localFNO/256/RMSR/nrmse\endcsname{0.721}
\expandafter\def\csname res@heat2dho256/localFNO/256/RMSR/pearson\endcsname{0.666}
\expandafter\def\csname res@heat2dho256/localFNO/256/RMSR/r2\endcsname{0.988}
\expandafter\def\csname res@heat2dho256/localFNO/256/Reference/nrmse\endcsname{0.000}
\expandafter\def\csname res@heat2dho256/localFNO/256/Reference/pearson\endcsname{1.000}
\expandafter\def\csname res@heat2dho256/localFNO/256/Reference/r2\endcsname{1.000}
\expandafter\def\csname res@heat2dmid/localFNO/1024/IG/nrmse\endcsname{1.000}
\expandafter\def\csname res@heat2dmid/localFNO/1024/IG/r2\endcsname{-1.001}
\expandafter\def\csname res@heat2dmid/localFNO/1024/KernelSHAP/nrmse\endcsname{148.854}
\expandafter\def\csname res@heat2dmid/localFNO/1024/KernelSHAP/pearson\endcsname{-0.004}
\expandafter\def\csname res@heat2dmid/localFNO/1024/KernelSHAP/r2\endcsname{-87.362}
\expandafter\def\csname res@heat2dmid/localFNO/1024/OperatorSHAP/nrmse\endcsname{0.649}
\expandafter\def\csname res@heat2dmid/localFNO/1024/OperatorSHAP/pearson\endcsname{0.841}
\expandafter\def\csname res@heat2dmid/localFNO/1024/OperatorSHAP/r2\endcsname{0.899}
\expandafter\def\csname res@heat2dmid/localFNO/1024/RMSR/nrmse\endcsname{2.170}
\expandafter\def\csname res@heat2dmid/localFNO/1024/RMSR/pearson\endcsname{0.030}
\expandafter\def\csname res@heat2dmid/localFNO/1024/RMSR/r2\endcsname{0.883}
\expandafter\def\csname res@heat2dmid/localFNO/1024/Reference/nrmse\endcsname{0.000}
\expandafter\def\csname res@heat2dmid/localFNO/1024/Reference/pearson\endcsname{1.000}
\expandafter\def\csname res@heat2dmid/localFNO/1024/Reference/r2\endcsname{0.898}
\expandafter\def\csname res@heat2dmid/localFNO/16/IG/nrmse\endcsname{1.000}
\expandafter\def\csname res@heat2dmid/localFNO/16/IG/r2\endcsname{-1.003}
\expandafter\def\csname res@heat2dmid/localFNO/16/KernelSHAP/nrmse\endcsname{0.143}
\expandafter\def\csname res@heat2dmid/localFNO/16/KernelSHAP/pearson\endcsname{0.988}
\expandafter\def\csname res@heat2dmid/localFNO/16/KernelSHAP/r2\endcsname{0.763}
\expandafter\def\csname res@heat2dmid/localFNO/16/OperatorSHAP/nrmse\endcsname{0.118}
\expandafter\def\csname res@heat2dmid/localFNO/16/OperatorSHAP/pearson\endcsname{0.993}
\expandafter\def\csname res@heat2dmid/localFNO/16/OperatorSHAP/r2\endcsname{0.767}
\expandafter\def\csname res@heat2dmid/localFNO/16/RMSR/nrmse\endcsname{0.107}
\expandafter\def\csname res@heat2dmid/localFNO/16/RMSR/pearson\endcsname{0.993}
\expandafter\def\csname res@heat2dmid/localFNO/16/RMSR/r2\endcsname{0.766}
\expandafter\def\csname res@heat2dmid/localFNO/16/Reference/nrmse\endcsname{0.000}
\expandafter\def\csname res@heat2dmid/localFNO/16/Reference/pearson\endcsname{1.000}
\expandafter\def\csname res@heat2dmid/localFNO/16/Reference/r2\endcsname{0.769}
\expandafter\def\csname res@heat2dmid/localFNO/256/IG/nrmse\endcsname{1.000}
\expandafter\def\csname res@heat2dmid/localFNO/256/IG/r2\endcsname{-1.001}
\expandafter\def\csname res@heat2dmid/localFNO/256/KernelSHAP/nrmse\endcsname{3.800}
\expandafter\def\csname res@heat2dmid/localFNO/256/KernelSHAP/pearson\endcsname{0.180}
\expandafter\def\csname res@heat2dmid/localFNO/256/KernelSHAP/r2\endcsname{0.661}
\expandafter\def\csname res@heat2dmid/localFNO/256/OperatorSHAP/nrmse\endcsname{0.241}
\expandafter\def\csname res@heat2dmid/localFNO/256/OperatorSHAP/pearson\endcsname{0.974}
\expandafter\def\csname res@heat2dmid/localFNO/256/OperatorSHAP/r2\endcsname{0.875}
\expandafter\def\csname res@heat2dmid/localFNO/256/RMSR/nrmse\endcsname{1.315}
\expandafter\def\csname res@heat2dmid/localFNO/256/RMSR/pearson\endcsname{0.198}
\expandafter\def\csname res@heat2dmid/localFNO/256/RMSR/r2\endcsname{0.846}
\expandafter\def\csname res@heat2dmid/localFNO/256/Reference/nrmse\endcsname{0.000}
\expandafter\def\csname res@heat2dmid/localFNO/256/Reference/pearson\endcsname{1.000}
\expandafter\def\csname res@heat2dmid/localFNO/256/Reference/r2\endcsname{0.875}
\expandafter\def\csname res@heat2dmid/localFNO/64/IG/nrmse\endcsname{1.000}
\expandafter\def\csname res@heat2dmid/localFNO/64/IG/r2\endcsname{-1.001}
\expandafter\def\csname res@heat2dmid/localFNO/64/KernelSHAP/nrmse\endcsname{0.671}
\expandafter\def\csname res@heat2dmid/localFNO/64/KernelSHAP/pearson\endcsname{0.815}
\expandafter\def\csname res@heat2dmid/localFNO/64/KernelSHAP/r2\endcsname{0.796}
\expandafter\def\csname res@heat2dmid/localFNO/64/OperatorSHAP/nrmse\endcsname{0.135}
\expandafter\def\csname res@heat2dmid/localFNO/64/OperatorSHAP/pearson\endcsname{0.990}
\expandafter\def\csname res@heat2dmid/localFNO/64/OperatorSHAP/r2\endcsname{0.835}
\expandafter\def\csname res@heat2dmid/localFNO/64/RMSR/nrmse\endcsname{0.549}
\expandafter\def\csname res@heat2dmid/localFNO/64/RMSR/pearson\endcsname{0.838}
\expandafter\def\csname res@heat2dmid/localFNO/64/RMSR/r2\endcsname{0.812}
\expandafter\def\csname res@heat2dmid/localFNO/64/Reference/nrmse\endcsname{0.000}
\expandafter\def\csname res@heat2dmid/localFNO/64/Reference/pearson\endcsname{1.000}
\expandafter\def\csname res@heat2dmid/localFNO/64/Reference/r2\endcsname{0.836}
\expandafter\def\csname res@heat3d/fno/512/IG/nrmse\endcsname{0.255}
\expandafter\def\csname res@heat3d/fno/512/IG/pearson\endcsname{0.979}
\expandafter\def\csname res@heat3d/fno/512/IG/r2\endcsname{1.000}
\expandafter\def\csname res@heat3d/fno/512/KernelSHAP/nrmse\endcsname{1.036}
\expandafter\def\csname res@heat3d/fno/512/KernelSHAP/pearson\endcsname{0.726}
\expandafter\def\csname res@heat3d/fno/512/KernelSHAP/r2\endcsname{0.974}
\expandafter\def\csname res@heat3d/fno/512/OperatorSHAP/nrmse\endcsname{0.432}
\expandafter\def\csname res@heat3d/fno/512/OperatorSHAP/pearson\endcsname{0.906}
\expandafter\def\csname res@heat3d/fno/512/OperatorSHAP/r2\endcsname{0.996}
\expandafter\def\csname res@heat3d/fno/512/RMSR/nrmse\endcsname{0.854}
\expandafter\def\csname res@heat3d/fno/512/RMSR/pearson\endcsname{0.531}
\expandafter\def\csname res@heat3d/fno/512/RMSR/r2\endcsname{0.980}
\expandafter\def\csname res@heat3d/fno/512/Reference/nrmse\endcsname{0.000}
\expandafter\def\csname res@heat3d/fno/512/Reference/pearson\endcsname{1.000}
\expandafter\def\csname res@heat3d/fno/512/Reference/r2\endcsname{0.999}
\expandafter\def\csname res@heat3dmid/fno/512/IG/nrmse\endcsname{1.000}
\expandafter\def\csname res@heat3dmid/fno/512/IG/r2\endcsname{-1.000}
\expandafter\def\csname res@heat3dmid/fno/512/KernelSHAP/nrmse\endcsname{55.623}
\expandafter\def\csname res@heat3dmid/fno/512/KernelSHAP/pearson\endcsname{0.004}
\expandafter\def\csname res@heat3dmid/fno/512/KernelSHAP/r2\endcsname{-60.883}
\expandafter\def\csname res@heat3dmid/fno/512/OperatorSHAP/nrmse\endcsname{0.516}
\expandafter\def\csname res@heat3dmid/fno/512/OperatorSHAP/pearson\endcsname{0.884}
\expandafter\def\csname res@heat3dmid/fno/512/OperatorSHAP/r2\endcsname{0.884}
\expandafter\def\csname res@heat3dmid/fno/512/RMSR/nrmse\endcsname{1.360}
\expandafter\def\csname res@heat3dmid/fno/512/RMSR/pearson\endcsname{0.107}
\expandafter\def\csname res@heat3dmid/fno/512/RMSR/r2\endcsname{0.858}
\expandafter\def\csname res@heat3dmid/fno/512/Reference/nrmse\endcsname{0.000}
\expandafter\def\csname res@heat3dmid/fno/512/Reference/pearson\endcsname{1.000}
\expandafter\def\csname res@heat3dmid/fno/512/Reference/r2\endcsname{0.887}
\expandafter\def\csname res@mgn_air/gino/64/IG/nrmse\endcsname{1.004}
\expandafter\def\csname res@mgn_air/gino/64/IG/pearson\endcsname{0.727}
\expandafter\def\csname res@mgn_air/gino/64/IG/r2\endcsname{0.727}
\expandafter\def\csname res@mgn_air/gino/64/KernelSHAP/nrmse\endcsname{0.212}
\expandafter\def\csname res@mgn_air/gino/64/KernelSHAP/pearson\endcsname{0.977}
\expandafter\def\csname res@mgn_air/gino/64/KernelSHAP/r2\endcsname{0.822}
\expandafter\def\csname res@mgn_air/gino/64/OperatorSHAP/nrmse\endcsname{0.657}
\expandafter\def\csname res@mgn_air/gino/64/OperatorSHAP/pearson\endcsname{0.737}
\expandafter\def\csname res@mgn_air/gino/64/OperatorSHAP/r2\endcsname{0.751}
\expandafter\def\csname res@mgn_air/gino/64/RMSR/nrmse\endcsname{0.255}
\expandafter\def\csname res@mgn_air/gino/64/RMSR/pearson\endcsname{0.970}
\expandafter\def\csname res@mgn_air/gino/64/RMSR/r2\endcsname{0.822}
\expandafter\def\csname res@mgn_air/gino/64/Reference/nrmse\endcsname{0.000}
\expandafter\def\csname res@mgn_air/gino/64/Reference/pearson\endcsname{1.000}
\expandafter\def\csname res@mgn_air/gino/64/Reference/r2\endcsname{0.829}
\expandafter\def\csname res@mgn_airmid/gino/64/IG/nrmse\endcsname{1.000}
\expandafter\def\csname res@mgn_airmid/gino/64/IG/r2\endcsname{-1.210}
\expandafter\def\csname res@mgn_airmid/gino/64/KernelSHAP/nrmse\endcsname{0.634}
\expandafter\def\csname res@mgn_airmid/gino/64/KernelSHAP/pearson\endcsname{0.827}
\expandafter\def\csname res@mgn_airmid/gino/64/KernelSHAP/r2\endcsname{0.749}
\expandafter\def\csname res@mgn_airmid/gino/64/OperatorSHAP/nrmse\endcsname{0.805}
\expandafter\def\csname res@mgn_airmid/gino/64/OperatorSHAP/pearson\endcsname{0.630}
\expandafter\def\csname res@mgn_airmid/gino/64/OperatorSHAP/r2\endcsname{0.694}
\expandafter\def\csname res@mgn_airmid/gino/64/RMSR/nrmse\endcsname{0.392}
\expandafter\def\csname res@mgn_airmid/gino/64/RMSR/pearson\endcsname{0.911}
\expandafter\def\csname res@mgn_airmid/gino/64/RMSR/r2\endcsname{0.757}
\expandafter\def\csname res@mgn_airmid/gino/64/Reference/nrmse\endcsname{0.000}
\expandafter\def\csname res@mgn_airmid/gino/64/Reference/pearson\endcsname{1.000}
\expandafter\def\csname res@mgn_airmid/gino/64/Reference/r2\endcsname{0.766}
\expandafter\def\csname res@mgn_cyl/gino/64/IG/nrmse\endcsname{0.415}
\expandafter\def\csname res@mgn_cyl/gino/64/IG/pearson\endcsname{0.916}
\expandafter\def\csname res@mgn_cyl/gino/64/IG/r2\endcsname{0.900}
\expandafter\def\csname res@mgn_cyl/gino/64/KernelSHAP/nrmse\endcsname{0.130}
\expandafter\def\csname res@mgn_cyl/gino/64/KernelSHAP/pearson\endcsname{0.991}
\expandafter\def\csname res@mgn_cyl/gino/64/KernelSHAP/r2\endcsname{0.923}
\expandafter\def\csname res@mgn_cyl/gino/64/OperatorSHAP/nrmse\endcsname{0.315}
\expandafter\def\csname res@mgn_cyl/gino/64/OperatorSHAP/pearson\endcsname{0.946}
\expandafter\def\csname res@mgn_cyl/gino/64/OperatorSHAP/r2\endcsname{0.910}
\expandafter\def\csname res@mgn_cyl/gino/64/RMSR/nrmse\endcsname{0.181}
\expandafter\def\csname res@mgn_cyl/gino/64/RMSR/pearson\endcsname{0.985}
\expandafter\def\csname res@mgn_cyl/gino/64/RMSR/r2\endcsname{0.921}
\expandafter\def\csname res@mgn_cyl/gino/64/Reference/nrmse\endcsname{0.000}
\expandafter\def\csname res@mgn_cyl/gino/64/Reference/pearson\endcsname{1.000}
\expandafter\def\csname res@mgn_cyl/gino/64/Reference/r2\endcsname{0.925}
\expandafter\def\csname res@mgn_cylmid/gino/64/IG/nrmse\endcsname{1.000}
\expandafter\def\csname res@mgn_cylmid/gino/64/IG/r2\endcsname{-1.066}
\expandafter\def\csname res@mgn_cylmid/gino/64/KernelSHAP/nrmse\endcsname{0.465}
\expandafter\def\csname res@mgn_cylmid/gino/64/KernelSHAP/pearson\endcsname{0.900}
\expandafter\def\csname res@mgn_cylmid/gino/64/KernelSHAP/r2\endcsname{0.789}
\expandafter\def\csname res@mgn_cylmid/gino/64/OperatorSHAP/nrmse\endcsname{0.402}
\expandafter\def\csname res@mgn_cylmid/gino/64/OperatorSHAP/pearson\endcsname{0.928}
\expandafter\def\csname res@mgn_cylmid/gino/64/OperatorSHAP/r2\endcsname{0.794}
\expandafter\def\csname res@mgn_cylmid/gino/64/RMSR/nrmse\endcsname{0.289}
\expandafter\def\csname res@mgn_cylmid/gino/64/RMSR/pearson\endcsname{0.957}
\expandafter\def\csname res@mgn_cylmid/gino/64/RMSR/r2\endcsname{0.799}
\expandafter\def\csname res@mgn_cylmid/gino/64/Reference/nrmse\endcsname{0.000}
\expandafter\def\csname res@mgn_cylmid/gino/64/Reference/pearson\endcsname{1.000}
\expandafter\def\csname res@mgn_cylmid/gino/64/Reference/r2\endcsname{0.808}
\expandafter\def\csname res@rxndiff2d/localFNO/1024/IG/nrmse\endcsname{1.501}
\expandafter\def\csname res@rxndiff2d/localFNO/1024/IG/pearson\endcsname{0.789}
\expandafter\def\csname res@rxndiff2d/localFNO/1024/IG/r2\endcsname{0.350}
\expandafter\def\csname res@rxndiff2d/localFNO/1024/KernelSHAP/nrmse\endcsname{51.625}
\expandafter\def\csname res@rxndiff2d/localFNO/1024/KernelSHAP/pearson\endcsname{0.048}
\expandafter\def\csname res@rxndiff2d/localFNO/1024/KernelSHAP/r2\endcsname{-252.302}
\expandafter\def\csname res@rxndiff2d/localFNO/1024/OperatorSHAP/nrmse\endcsname{0.630}
\expandafter\def\csname res@rxndiff2d/localFNO/1024/OperatorSHAP/pearson\endcsname{0.847}
\expandafter\def\csname res@rxndiff2d/localFNO/1024/OperatorSHAP/r2\endcsname{0.527}
\expandafter\def\csname res@rxndiff2d/localFNO/1024/RMSR/nrmse\endcsname{0.628}
\expandafter\def\csname res@rxndiff2d/localFNO/1024/RMSR/pearson\endcsname{0.804}
\expandafter\def\csname res@rxndiff2d/localFNO/1024/RMSR/r2\endcsname{0.456}
\expandafter\def\csname res@rxndiff2d/localFNO/1024/Reference/nrmse\endcsname{0.000}
\expandafter\def\csname res@rxndiff2d/localFNO/1024/Reference/pearson\endcsname{1.000}
\expandafter\def\csname res@rxndiff2d/localFNO/1024/Reference/r2\endcsname{0.534}
\expandafter\def\csname res@rxndiff2d/localFNO/256/IG/nrmse\endcsname{1.457}
\expandafter\def\csname res@rxndiff2d/localFNO/256/IG/pearson\endcsname{0.875}
\expandafter\def\csname res@rxndiff2d/localFNO/256/IG/r2\endcsname{-0.101}
\expandafter\def\csname res@rxndiff2d/localFNO/256/KernelSHAP/nrmse\endcsname{1.392}
\expandafter\def\csname res@rxndiff2d/localFNO/256/KernelSHAP/pearson\endcsname{0.663}
\expandafter\def\csname res@rxndiff2d/localFNO/256/KernelSHAP/r2\endcsname{-0.083}
\expandafter\def\csname res@rxndiff2d/localFNO/256/OperatorSHAP/nrmse\endcsname{0.386}
\expandafter\def\csname res@rxndiff2d/localFNO/256/OperatorSHAP/pearson\endcsname{0.932}
\expandafter\def\csname res@rxndiff2d/localFNO/256/OperatorSHAP/r2\endcsname{0.292}
\expandafter\def\csname res@rxndiff2d/localFNO/256/RMSR/nrmse\endcsname{0.642}
\expandafter\def\csname res@rxndiff2d/localFNO/256/RMSR/pearson\endcsname{0.790}
\expandafter\def\csname res@rxndiff2d/localFNO/256/RMSR/r2\endcsname{0.210}
\expandafter\def\csname res@rxndiff2d/localFNO/256/Reference/nrmse\endcsname{0.000}
\expandafter\def\csname res@rxndiff2d/localFNO/256/Reference/pearson\endcsname{1.000}
\expandafter\def\csname res@rxndiff2d/localFNO/256/Reference/r2\endcsname{0.331}
\expandafter\def\csname res@rxndiff2d/localFNO/64/IG/nrmse\endcsname{2.564}
\expandafter\def\csname res@rxndiff2d/localFNO/64/IG/pearson\endcsname{0.733}
\expandafter\def\csname res@rxndiff2d/localFNO/64/IG/r2\endcsname{-1.252}
\expandafter\def\csname res@rxndiff2d/localFNO/64/KernelSHAP/nrmse\endcsname{0.422}
\expandafter\def\csname res@rxndiff2d/localFNO/64/KernelSHAP/pearson\endcsname{0.922}
\expandafter\def\csname res@rxndiff2d/localFNO/64/KernelSHAP/r2\endcsname{0.211}
\expandafter\def\csname res@rxndiff2d/localFNO/64/OperatorSHAP/nrmse\endcsname{0.208}
\expandafter\def\csname res@rxndiff2d/localFNO/64/OperatorSHAP/pearson\endcsname{0.980}
\expandafter\def\csname res@rxndiff2d/localFNO/64/OperatorSHAP/r2\endcsname{0.246}
\expandafter\def\csname res@rxndiff2d/localFNO/64/RMSR/nrmse\endcsname{0.364}
\expandafter\def\csname res@rxndiff2d/localFNO/64/RMSR/pearson\endcsname{0.944}
\expandafter\def\csname res@rxndiff2d/localFNO/64/RMSR/r2\endcsname{0.230}
\expandafter\def\csname res@rxndiff2d/localFNO/64/Reference/nrmse\endcsname{0.000}
\expandafter\def\csname res@rxndiff2d/localFNO/64/Reference/pearson\endcsname{1.000}
\expandafter\def\csname res@rxndiff2d/localFNO/64/Reference/r2\endcsname{0.251}
\expandafter\def\csname res@rxndiff2dmid/localFNO/1024/IG/nrmse\endcsname{1.000}
\expandafter\def\csname res@rxndiff2dmid/localFNO/1024/IG/r2\endcsname{-0.910}
\expandafter\def\csname res@rxndiff2dmid/localFNO/1024/KernelSHAP/nrmse\endcsname{114.896}
\expandafter\def\csname res@rxndiff2dmid/localFNO/1024/KernelSHAP/pearson\endcsname{0.022}
\expandafter\def\csname res@rxndiff2dmid/localFNO/1024/KernelSHAP/r2\endcsname{-659.842}
\expandafter\def\csname res@rxndiff2dmid/localFNO/1024/OperatorSHAP/nrmse\endcsname{0.548}
\expandafter\def\csname res@rxndiff2dmid/localFNO/1024/OperatorSHAP/pearson\endcsname{0.892}
\expandafter\def\csname res@rxndiff2dmid/localFNO/1024/OperatorSHAP/r2\endcsname{0.560}
\expandafter\def\csname res@rxndiff2dmid/localFNO/1024/RMSR/nrmse\endcsname{0.772}
\expandafter\def\csname res@rxndiff2dmid/localFNO/1024/RMSR/pearson\endcsname{0.657}
\expandafter\def\csname res@rxndiff2dmid/localFNO/1024/RMSR/r2\endcsname{0.517}
\expandafter\def\csname res@rxndiff2dmid/localFNO/1024/Reference/nrmse\endcsname{0.000}
\expandafter\def\csname res@rxndiff2dmid/localFNO/1024/Reference/pearson\endcsname{1.000}
\expandafter\def\csname res@rxndiff2dmid/localFNO/1024/Reference/r2\endcsname{0.561}
\expandafter\def\csname res@rxndiff2dmid/localFNO/16/IG/nrmse\endcsname{1.000}
\expandafter\def\csname res@rxndiff2dmid/localFNO/16/IG/r2\endcsname{-0.535}
\expandafter\def\csname res@rxndiff2dmid/localFNO/16/KernelSHAP/nrmse\endcsname{0.092}
\expandafter\def\csname res@rxndiff2dmid/localFNO/16/KernelSHAP/pearson\endcsname{0.996}
\expandafter\def\csname res@rxndiff2dmid/localFNO/16/KernelSHAP/r2\endcsname{0.295}
\expandafter\def\csname res@rxndiff2dmid/localFNO/16/OperatorSHAP/nrmse\endcsname{0.037}
\expandafter\def\csname res@rxndiff2dmid/localFNO/16/OperatorSHAP/pearson\endcsname{0.999}
\expandafter\def\csname res@rxndiff2dmid/localFNO/16/OperatorSHAP/r2\endcsname{0.298}
\expandafter\def\csname res@rxndiff2dmid/localFNO/16/RMSR/nrmse\endcsname{0.075}
\expandafter\def\csname res@rxndiff2dmid/localFNO/16/RMSR/pearson\endcsname{0.997}
\expandafter\def\csname res@rxndiff2dmid/localFNO/16/RMSR/r2\endcsname{0.296}
\expandafter\def\csname res@rxndiff2dmid/localFNO/16/Reference/nrmse\endcsname{0.000}
\expandafter\def\csname res@rxndiff2dmid/localFNO/16/Reference/pearson\endcsname{1.000}
\expandafter\def\csname res@rxndiff2dmid/localFNO/16/Reference/r2\endcsname{0.298}
\expandafter\def\csname res@rxndiff2dmid/localFNO/256/IG/nrmse\endcsname{1.000}
\expandafter\def\csname res@rxndiff2dmid/localFNO/256/IG/r2\endcsname{-1.076}
\expandafter\def\csname res@rxndiff2dmid/localFNO/256/KernelSHAP/nrmse\endcsname{2.935}
\expandafter\def\csname res@rxndiff2dmid/localFNO/256/KernelSHAP/pearson\endcsname{0.431}
\expandafter\def\csname res@rxndiff2dmid/localFNO/256/KernelSHAP/r2\endcsname{-0.504}
\expandafter\def\csname res@rxndiff2dmid/localFNO/256/OperatorSHAP/nrmse\endcsname{0.366}
\expandafter\def\csname res@rxndiff2dmid/localFNO/256/OperatorSHAP/pearson\endcsname{0.941}
\expandafter\def\csname res@rxndiff2dmid/localFNO/256/OperatorSHAP/r2\endcsname{0.166}
\expandafter\def\csname res@rxndiff2dmid/localFNO/256/RMSR/nrmse\endcsname{0.796}
\expandafter\def\csname res@rxndiff2dmid/localFNO/256/RMSR/pearson\endcsname{0.639}
\expandafter\def\csname res@rxndiff2dmid/localFNO/256/RMSR/r2\endcsname{0.155}
\expandafter\def\csname res@rxndiff2dmid/localFNO/256/Reference/nrmse\endcsname{0.000}
\expandafter\def\csname res@rxndiff2dmid/localFNO/256/Reference/pearson\endcsname{1.000}
\expandafter\def\csname res@rxndiff2dmid/localFNO/256/Reference/r2\endcsname{0.210}
\expandafter\def\csname res@rxndiff2dmid/localFNO/64/IG/nrmse\endcsname{1.000}
\expandafter\def\csname res@rxndiff2dmid/localFNO/64/IG/r2\endcsname{-0.953}
\expandafter\def\csname res@rxndiff2dmid/localFNO/64/KernelSHAP/nrmse\endcsname{0.595}
\expandafter\def\csname res@rxndiff2dmid/localFNO/64/KernelSHAP/pearson\endcsname{0.863}
\expandafter\def\csname res@rxndiff2dmid/localFNO/64/KernelSHAP/r2\endcsname{-0.122}
\expandafter\def\csname res@rxndiff2dmid/localFNO/64/OperatorSHAP/nrmse\endcsname{0.255}
\expandafter\def\csname res@rxndiff2dmid/localFNO/64/OperatorSHAP/pearson\endcsname{0.972}
\expandafter\def\csname res@rxndiff2dmid/localFNO/64/OperatorSHAP/r2\endcsname{-0.085}
\expandafter\def\csname res@rxndiff2dmid/localFNO/64/RMSR/nrmse\endcsname{0.384}
\expandafter\def\csname res@rxndiff2dmid/localFNO/64/RMSR/pearson\endcsname{0.922}
\expandafter\def\csname res@rxndiff2dmid/localFNO/64/RMSR/r2\endcsname{-0.092}
\expandafter\def\csname res@rxndiff2dmid/localFNO/64/Reference/nrmse\endcsname{0.000}
\expandafter\def\csname res@rxndiff2dmid/localFNO/64/Reference/pearson\endcsname{1.000}
\expandafter\def\csname res@rxndiff2dmid/localFNO/64/Reference/r2\endcsname{-0.076}
\expandafter\def\csname res@sw2d/localFNO/1024/IG/nrmse\endcsname{0.529}
\expandafter\def\csname res@sw2d/localFNO/1024/IG/pearson\endcsname{0.898}
\expandafter\def\csname res@sw2d/localFNO/1024/IG/r2\endcsname{0.980}
\expandafter\def\csname res@sw2d/localFNO/1024/KernelSHAP/nrmse\endcsname{9.197}
\expandafter\def\csname res@sw2d/localFNO/1024/KernelSHAP/pearson\endcsname{0.136}
\expandafter\def\csname res@sw2d/localFNO/1024/KernelSHAP/r2\endcsname{-3.861}
\expandafter\def\csname res@sw2d/localFNO/1024/OperatorSHAP/nrmse\endcsname{0.527}
\expandafter\def\csname res@sw2d/localFNO/1024/OperatorSHAP/pearson\endcsname{0.923}
\expandafter\def\csname res@sw2d/localFNO/1024/OperatorSHAP/r2\endcsname{0.974}
\expandafter\def\csname res@sw2d/localFNO/1024/RMSR/nrmse\endcsname{0.416}
\expandafter\def\csname res@sw2d/localFNO/1024/RMSR/pearson\endcsname{0.910}
\expandafter\def\csname res@sw2d/localFNO/1024/RMSR/r2\endcsname{0.951}
\expandafter\def\csname res@sw2d/localFNO/1024/Reference/nrmse\endcsname{0.000}
\expandafter\def\csname res@sw2d/localFNO/1024/Reference/pearson\endcsname{1.000}
\expandafter\def\csname res@sw2d/localFNO/1024/Reference/r2\endcsname{0.968}
\expandafter\def\csname res@sw2d/localFNO/256/IG/nrmse\endcsname{0.455}
\expandafter\def\csname res@sw2d/localFNO/256/IG/pearson\endcsname{0.915}
\expandafter\def\csname res@sw2d/localFNO/256/IG/r2\endcsname{0.869}
\expandafter\def\csname res@sw2d/localFNO/256/KernelSHAP/nrmse\endcsname{0.482}
\expandafter\def\csname res@sw2d/localFNO/256/KernelSHAP/pearson\endcsname{0.927}
\expandafter\def\csname res@sw2d/localFNO/256/KernelSHAP/r2\endcsname{0.864}
\expandafter\def\csname res@sw2d/localFNO/256/OperatorSHAP/nrmse\endcsname{0.311}
\expandafter\def\csname res@sw2d/localFNO/256/OperatorSHAP/pearson\endcsname{0.964}
\expandafter\def\csname res@sw2d/localFNO/256/OperatorSHAP/r2\endcsname{0.881}
\expandafter\def\csname res@sw2d/localFNO/256/RMSR/nrmse\endcsname{0.575}
\expandafter\def\csname res@sw2d/localFNO/256/RMSR/pearson\endcsname{0.850}
\expandafter\def\csname res@sw2d/localFNO/256/RMSR/r2\endcsname{0.803}
\expandafter\def\csname res@sw2d/localFNO/256/Reference/nrmse\endcsname{0.000}
\expandafter\def\csname res@sw2d/localFNO/256/Reference/pearson\endcsname{1.000}
\expandafter\def\csname res@sw2d/localFNO/256/Reference/r2\endcsname{0.886}
\expandafter\def\csname res@sw2d/localFNO/64/IG/nrmse\endcsname{0.731}
\expandafter\def\csname res@sw2d/localFNO/64/IG/pearson\endcsname{0.729}
\expandafter\def\csname res@sw2d/localFNO/64/IG/r2\endcsname{0.421}
\expandafter\def\csname res@sw2d/localFNO/64/KernelSHAP/nrmse\endcsname{0.258}
\expandafter\def\csname res@sw2d/localFNO/64/KernelSHAP/pearson\endcsname{0.970}
\expandafter\def\csname res@sw2d/localFNO/64/KernelSHAP/r2\endcsname{0.550}
\expandafter\def\csname res@sw2d/localFNO/64/OperatorSHAP/nrmse\endcsname{0.093}
\expandafter\def\csname res@sw2d/localFNO/64/OperatorSHAP/pearson\endcsname{0.996}
\expandafter\def\csname res@sw2d/localFNO/64/OperatorSHAP/r2\endcsname{0.570}
\expandafter\def\csname res@sw2d/localFNO/64/RMSR/nrmse\endcsname{0.416}
\expandafter\def\csname res@sw2d/localFNO/64/RMSR/pearson\endcsname{0.949}
\expandafter\def\csname res@sw2d/localFNO/64/RMSR/r2\endcsname{0.499}
\expandafter\def\csname res@sw2d/localFNO/64/Reference/nrmse\endcsname{0.000}
\expandafter\def\csname res@sw2d/localFNO/64/Reference/pearson\endcsname{1.000}
\expandafter\def\csname res@sw2d/localFNO/64/Reference/r2\endcsname{0.573}
\expandafter\def\csname res@sw2dmid/localFNO/1024/IG/nrmse\endcsname{1.000}
\expandafter\def\csname res@sw2dmid/localFNO/1024/IG/r2\endcsname{-0.949}
\expandafter\def\csname res@sw2dmid/localFNO/1024/KernelSHAP/nrmse\endcsname{34.301}
\expandafter\def\csname res@sw2dmid/localFNO/1024/KernelSHAP/pearson\endcsname{0.010}
\expandafter\def\csname res@sw2dmid/localFNO/1024/KernelSHAP/r2\endcsname{-66.535}
\expandafter\def\csname res@sw2dmid/localFNO/1024/OperatorSHAP/nrmse\endcsname{0.333}
\expandafter\def\csname res@sw2dmid/localFNO/1024/OperatorSHAP/pearson\endcsname{0.964}
\expandafter\def\csname res@sw2dmid/localFNO/1024/OperatorSHAP/r2\endcsname{0.887}
\expandafter\def\csname res@sw2dmid/localFNO/1024/RMSR/nrmse\endcsname{0.428}
\expandafter\def\csname res@sw2dmid/localFNO/1024/RMSR/pearson\endcsname{0.909}
\expandafter\def\csname res@sw2dmid/localFNO/1024/RMSR/r2\endcsname{0.871}
\expandafter\def\csname res@sw2dmid/localFNO/1024/Reference/nrmse\endcsname{0.000}
\expandafter\def\csname res@sw2dmid/localFNO/1024/Reference/pearson\endcsname{1.000}
\expandafter\def\csname res@sw2dmid/localFNO/1024/Reference/r2\endcsname{0.885}
\expandafter\def\csname res@sw2dmid/localFNO/16/IG/nrmse\endcsname{1.000}
\expandafter\def\csname res@sw2dmid/localFNO/16/IG/r2\endcsname{-1.590}
\expandafter\def\csname res@sw2dmid/localFNO/16/KernelSHAP/nrmse\endcsname{0.118}
\expandafter\def\csname res@sw2dmid/localFNO/16/KernelSHAP/pearson\endcsname{0.992}
\expandafter\def\csname res@sw2dmid/localFNO/16/KernelSHAP/r2\endcsname{0.264}
\expandafter\def\csname res@sw2dmid/localFNO/16/OperatorSHAP/nrmse\endcsname{0.056}
\expandafter\def\csname res@sw2dmid/localFNO/16/OperatorSHAP/pearson\endcsname{0.998}
\expandafter\def\csname res@sw2dmid/localFNO/16/OperatorSHAP/r2\endcsname{0.269}
\expandafter\def\csname res@sw2dmid/localFNO/16/RMSR/nrmse\endcsname{0.118}
\expandafter\def\csname res@sw2dmid/localFNO/16/RMSR/pearson\endcsname{0.993}
\expandafter\def\csname res@sw2dmid/localFNO/16/RMSR/r2\endcsname{0.264}
\expandafter\def\csname res@sw2dmid/localFNO/16/Reference/nrmse\endcsname{0.000}
\expandafter\def\csname res@sw2dmid/localFNO/16/Reference/pearson\endcsname{1.000}
\expandafter\def\csname res@sw2dmid/localFNO/16/Reference/r2\endcsname{0.270}
\expandafter\def\csname res@sw2dmid/localFNO/256/IG/nrmse\endcsname{1.000}
\expandafter\def\csname res@sw2dmid/localFNO/256/IG/r2\endcsname{-0.934}
\expandafter\def\csname res@sw2dmid/localFNO/256/KernelSHAP/nrmse\endcsname{1.394}
\expandafter\def\csname res@sw2dmid/localFNO/256/KernelSHAP/pearson\endcsname{0.539}
\expandafter\def\csname res@sw2dmid/localFNO/256/KernelSHAP/r2\endcsname{0.645}
\expandafter\def\csname res@sw2dmid/localFNO/256/OperatorSHAP/nrmse\endcsname{0.219}
\expandafter\def\csname res@sw2dmid/localFNO/256/OperatorSHAP/pearson\endcsname{0.978}
\expandafter\def\csname res@sw2dmid/localFNO/256/OperatorSHAP/r2\endcsname{0.811}
\expandafter\def\csname res@sw2dmid/localFNO/256/RMSR/nrmse\endcsname{0.694}
\expandafter\def\csname res@sw2dmid/localFNO/256/RMSR/pearson\endcsname{0.728}
\expandafter\def\csname res@sw2dmid/localFNO/256/RMSR/r2\endcsname{0.756}
\expandafter\def\csname res@sw2dmid/localFNO/256/Reference/nrmse\endcsname{0.000}
\expandafter\def\csname res@sw2dmid/localFNO/256/Reference/pearson\endcsname{1.000}
\expandafter\def\csname res@sw2dmid/localFNO/256/Reference/r2\endcsname{0.813}
\expandafter\def\csname res@sw2dmid/localFNO/64/IG/nrmse\endcsname{1.000}
\expandafter\def\csname res@sw2dmid/localFNO/64/IG/r2\endcsname{-0.838}
\expandafter\def\csname res@sw2dmid/localFNO/64/KernelSHAP/nrmse\endcsname{0.429}
\expandafter\def\csname res@sw2dmid/localFNO/64/KernelSHAP/pearson\endcsname{0.917}
\expandafter\def\csname res@sw2dmid/localFNO/64/KernelSHAP/r2\endcsname{0.525}
\expandafter\def\csname res@sw2dmid/localFNO/64/OperatorSHAP/nrmse\endcsname{0.125}
\expandafter\def\csname res@sw2dmid/localFNO/64/OperatorSHAP/pearson\endcsname{0.992}
\expandafter\def\csname res@sw2dmid/localFNO/64/OperatorSHAP/r2\endcsname{0.551}
\expandafter\def\csname res@sw2dmid/localFNO/64/RMSR/nrmse\endcsname{0.442}
\expandafter\def\csname res@sw2dmid/localFNO/64/RMSR/pearson\endcsname{0.912}
\expandafter\def\csname res@sw2dmid/localFNO/64/RMSR/r2\endcsname{0.518}
\expandafter\def\csname res@sw2dmid/localFNO/64/Reference/nrmse\endcsname{0.000}
\expandafter\def\csname res@sw2dmid/localFNO/64/Reference/pearson\endcsname{1.000}
\expandafter\def\csname res@sw2dmid/localFNO/64/Reference/r2\endcsname{0.554}

%% file: tables/numbers_analytic.tex
\expandafter\def\csname res@analytic/heat1d/fno/pearson\endcsname{1.000}
\expandafter\def\csname res@analytic/heat1d/mwtGupta/pearson\endcsname{0.977}

%% file: tables/numbers_baseline.tex
\expandafter\def\csname res@baseline/darcy2d/1024/reference_pearson\endcsname{-0.24}
\expandafter\def\csname res@baseline/darcy2d/64/reference_pearson\endcsname{-0.08}
\expandafter\def\csname res@baseline/diffsorp1d/1024/reference_pearson\endcsname{-0.02}
\expandafter\def\csname res@baseline/diffsorp1d/64/reference_pearson\endcsname{0.06}

%% file: tables/numbers_matched_cost.tex
\expandafter\def\csname res@matched/ig_ratio_max\endcsname{1.9}
\expandafter\def\csname res@matched/ig_ratio_min\endcsname{1.5}

%% file: tables/all_pde_metrics_short.tex
\small
\begin{tabular}{llcccc!{\vrule width 0.3pt}cccc!{\vrule width 0.3pt}cccc}
\toprule
 &  & \multicolumn{4}{c!{\vrule width 0.3pt}}{R$^2$ $\uparrow$} & \multicolumn{4}{c!{\vrule width 0.3pt}}{Pearson $\uparrow$} & \multicolumn{4}{c}{NRMSE $\downarrow$} \\
PDE & $n$ & RMSR & KS & IG & OS (ours) & RMSR & KS & IG & OS (ours) & RMSR & KS & IG & OS (ours) \\
\midrule
\multicolumn{14}{l}{\emph{Point readout}} \\
\multirow{6}{*}{Burgers 1D} & 64 & 0.924 & \textbf{0.957} & \textbf{0.956} & \textbf{0.952} & 0.979 & \textbf{0.999} & \textbf{0.999} & 0.995 & 0.256 & \textbf{0.049} & \textbf{0.043} & 0.100 \\
 &  & \tiny{[0.876, 0.978]} & \tiny{[0.900, 0.991]} & \tiny{[0.895, 0.991]} & \tiny{[0.887, 0.987]} & \tiny{[0.968, 0.985]} & \tiny{[0.995, 1.000]} & \tiny{[0.997, 1.000]} & \tiny{[0.991, 0.997]} & \tiny{[0.197, 0.314]} & \tiny{[0.016, 0.097]} & \tiny{[0.018, 0.090]} & \tiny{[0.073, 0.158]} \\
 & 256 & 0.944 & \textbf{0.971} & \textbf{0.971} & \textbf{0.966} & 0.596 & 0.946 & \textbf{0.997} & 0.985 & 0.805 & 0.317 & \textbf{0.094} & 0.195 \\
 &  & \tiny{[0.834, 0.979]} & \tiny{[0.936, 0.996]} & \tiny{[0.941, 0.997]} & \tiny{[0.919, 0.991]} & \tiny{[0.373, 0.814]} & \tiny{[0.847, 0.997]} & \tiny{[0.995, 0.998]} & \tiny{[0.971, 0.992]} & \tiny{[0.684, 0.964]} & \tiny{[0.122, 0.574]} & \tiny{[0.064, 0.122]} & \tiny{[0.126, 0.263]} \\
 & 1024 & 0.960 & -2.497 & \textbf{0.985} & 0.970 & 0.240 & 0.050 & \textbf{0.958} & 0.900 & 1.357 & 28.083 & \textbf{0.317} & 0.480 \\
 &  & \tiny{[0.891, 0.988]} & \tiny{[-13.453, -0.042]} & \tiny{[0.956, 0.998]} & \tiny{[0.937, 0.993]} & \tiny{[0.123, 0.554]} & \tiny{[0.023, 0.108]} & \tiny{[0.940, 0.968]} & \tiny{[0.848, 0.930]} & \tiny{[0.865, 1.642]} & \tiny{[12.454, 53.720]} & \tiny{[0.255, 0.481]} & \tiny{[0.363, 0.724]} \\
\cmidrule(l){1-14}
\multirow{6}{*}{Shallow Water 2D} & 64 & 0.499 & 0.550 & 0.421 & \textbf{0.570} & 0.949 & 0.970 & 0.729 & \textbf{0.996} & 0.416 & 0.258 & 0.731 & \textbf{0.093} \\
 &  & \tiny{[0.408, 0.600]} & \tiny{[0.456, 0.624]} & \tiny{[0.247, 0.476]} & \tiny{[0.479, 0.644]} & \tiny{[0.940, 0.957]} & \tiny{[0.961, 0.976]} & \tiny{[0.620, 0.810]} & \tiny{[0.995, 0.997]} & \tiny{[0.372, 0.440]} & \tiny{[0.226, 0.291]} & \tiny{[0.617, 0.855]} & \tiny{[0.085, 0.106]} \\
 & 256 & 0.803 & 0.864 & 0.869 & \textbf{0.881} & 0.850 & 0.927 & 0.915 & \textbf{0.964} & 0.575 & 0.482 & 0.455 & \textbf{0.311} \\
 &  & \tiny{[0.745, 0.874]} & \tiny{[0.824, 0.892]} & \tiny{[0.825, 0.897]} & \tiny{[0.841, 0.910]} & \tiny{[0.834, 0.899]} & \tiny{[0.906, 0.936]} & \tiny{[0.892, 0.927]} & \tiny{[0.951, 0.972]} & \tiny{[0.451, 0.606]} & \tiny{[0.451, 0.522]} & \tiny{[0.425, 0.497]} & \tiny{[0.288, 0.364]} \\
 & 1024 & 0.951 & -3.861 & \textbf{0.980} & 0.974 & 0.910 & 0.136 & 0.898 & \textbf{0.923} & \textbf{0.416} & 9.197 & 0.529 & 0.527 \\
 &  & \tiny{[0.929, 0.965]} & \tiny{[-23.325, -0.627]} & \tiny{[0.971, 0.986]} & \tiny{[0.964, 0.982]} & \tiny{[0.880, 0.932]} & \tiny{[0.050, 0.216]} & \tiny{[0.879, 0.921]} & \tiny{[0.898, 0.939]} & \tiny{[0.363, 0.477]} & \tiny{[5.292, 22.002]} & \tiny{[0.433, 0.637]} & \tiny{[0.454, 0.598]} \\
\cmidrule(l){1-14}
\multirow{2}{*}{MGN Cylinder} & 64 & \textbf{0.921} & \textbf{0.923} & 0.900 & \textbf{0.910} & 0.985 & \textbf{0.991} & 0.916 & 0.946 & 0.181 & \textbf{0.130} & 0.415 & 0.315 \\
 &  & \tiny{[0.824, 0.963]} & \tiny{[0.827, 0.963]} & \tiny{[0.775, 0.943]} & \tiny{[0.783, 0.955]} & \tiny{[0.977, 0.992]} & \tiny{[0.984, 0.993]} & \tiny{[0.830, 0.951]} & \tiny{[0.890, 0.983]} & \tiny{[0.138, 0.235]} & \tiny{[0.116, 0.170]} & \tiny{[0.338, 0.558]} & \tiny{[0.185, 0.472]} \\
\midrule
\multicolumn{14}{l}{\emph{Decision readout}} \\
\multirow{6}{*}{Burgers 1D} & 64 & \textbf{0.765} & \textbf{0.764} & -1.120 & \textbf{0.776} & 0.884 & 0.855 & -- & \textbf{0.990} & 0.449 & 0.535 & 1.000 & \textbf{0.160} \\
 &  & \tiny{[0.630, 0.811]} & \tiny{[0.627, 0.799]} & \tiny{[-1.393, -1.004]} & \tiny{[0.648, 0.824]} & \tiny{[0.775, 0.946]} & \tiny{[0.748, 0.929]} &  & \tiny{[0.983, 0.993]} & \tiny{[0.348, 0.580]} & \tiny{[0.412, 0.761]} & \tiny{[1.000, 1.000]} & \tiny{[0.126, 0.214]} \\
 & 256 & \textbf{0.810} & 0.702 & -1.113 & \textbf{0.832} & 0.178 & 0.193 & -- & \textbf{0.975} & 1.221 & 3.391 & 1.000 & \textbf{0.232} \\
 &  & \tiny{[0.721, 0.848]} & \tiny{[0.631, 0.728]} & \tiny{[-1.280, -1.012]} & \tiny{[0.745, 0.867]} & \tiny{[0.083, 0.446]} & \tiny{[0.093, 0.385]} &  & \tiny{[0.964, 0.983]} & \tiny{[0.889, 1.499]} & \tiny{[2.121, 4.410]} & \tiny{[1.000, 1.000]} & \tiny{[0.182, 0.342]} \\
 & 1024 & \textbf{0.858} & -86.894 & -1.096 & \textbf{0.865} & 0.038 & 0.006 & -- & \textbf{0.831} & 1.828 & 173.467 & 1.000 & \textbf{0.564} \\
 &  & \tiny{[0.797, 0.882]} & \tiny{[-310.480, -30.642]} & \tiny{[-1.221, -1.016]} & \tiny{[0.807, 0.893]} & \tiny{[0.009, 0.129]} & \tiny{[-0.018, 0.029]} &  & \tiny{[0.779, 0.869]} & \tiny{[1.275, 2.228]} & \tiny{[82.409, 316.716]} & \tiny{[1.000, 1.000]} & \tiny{[0.493, 0.630]} \\
\cmidrule(l){1-14}
\multirow{6}{*}{Shallow Water 2D} & 64 & 0.518 & 0.525 & -0.838 & \textbf{0.551} & 0.912 & 0.917 & -- & \textbf{0.992} & 0.442 & 0.429 & 1.000 & \textbf{0.125} \\
 &  & \tiny{[0.474, 0.668]} & \tiny{[0.479, 0.670]} & \tiny{[-0.938, -0.764]} & \tiny{[0.508, 0.694]} & \tiny{[0.892, 0.934]} & \tiny{[0.904, 0.932]} &  & \tiny{[0.990, 0.994]} & \tiny{[0.383, 0.480]} & \tiny{[0.396, 0.463]} & \tiny{[1.000, 1.000]} & \tiny{[0.107, 0.157]} \\
 & 256 & 0.756 & 0.645 & -0.934 & \textbf{0.811} & 0.728 & 0.539 & -- & \textbf{0.978} & 0.694 & 1.394 & 1.000 & \textbf{0.219} \\
 &  & \tiny{[0.717, 0.815]} & \tiny{[0.591, 0.699]} & \tiny{[-1.012, -0.870]} & \tiny{[0.785, 0.833]} & \tiny{[0.660, 0.869]} & \tiny{[0.493, 0.572]} &  & \tiny{[0.971, 0.984]} & \tiny{[0.505, 0.763]} & \tiny{[1.301, 1.509]} & \tiny{[1.000, 1.000]} & \tiny{[0.190, 0.266]} \\
 & 1024 & 0.871 & -66.535 & -0.949 & \textbf{0.887} & 0.909 & 0.010 & -- & \textbf{0.964} & 0.428 & 34.301 & 1.000 & \textbf{0.333} \\
 &  & \tiny{[0.855, 0.878]} & \tiny{[-298.438, -26.911]} & \tiny{[-0.998, -0.903]} & \tiny{[0.878, 0.892]} & \tiny{[0.875, 0.935]} & \tiny{[-0.006, 0.039]} &  & \tiny{[0.943, 0.974]} & \tiny{[0.367, 0.493]} & \tiny{[23.284, 82.147]} & \tiny{[1.000, 1.000]} & \tiny{[0.267, 0.389]} \\
\cmidrule(l){1-14}
\multirow{2}{*}{MGN Cylinder} & 64 & \textbf{0.799} & \textbf{0.789} & -1.066 & \textbf{0.794} & \textbf{0.957} & 0.900 & -- & 0.928 & \textbf{0.289} & 0.465 & 1.000 & 0.402 \\
 &  & \tiny{[0.735, 0.830]} & \tiny{[0.732, 0.812]} & \tiny{[-1.196, -0.948]} & \tiny{[0.705, 0.824]} & \tiny{[0.909, 0.976]} & \tiny{[0.818, 0.943]} &  & \tiny{[0.774, 0.977]} & \tiny{[0.218, 0.406]} & \tiny{[0.366, 0.619]} & \tiny{[1.000, 1.000]} & \tiny{[0.213, 0.645]} \\
\bottomrule
\end{tabular}

%% file: tables/all_pde_metrics_full_top.tex
\small
\begin{tabular}{llcccc!{\vrule width 0.3pt}cccc!{\vrule width 0.3pt}cccc}
\toprule
 &  & \multicolumn{4}{c!{\vrule width 0.3pt}}{R$^2$ $\uparrow$} & \multicolumn{4}{c!{\vrule width 0.3pt}}{Pearson $\uparrow$} & \multicolumn{4}{c}{NRMSE $\downarrow$} \\
PDE & $n$ & RMSR & KS & IG & OS (ours) & RMSR & KS & IG & OS (ours) & RMSR & KS & IG & OS (ours) \\
\midrule
\multirow{12}{*}{Heat 1D} & 32 & 0.999 & \textbf{1.000} & \textbf{1.000} & \textbf{1.000} & 0.997 & \textbf{1.000} & \textbf{1.000} & \textbf{1.000} & 0.082 & \textbf{0.002} & 0.011 & 0.021 \\
 &  & \tiny{[0.993, 0.999]} & \tiny{[1.000, 1.000]} & \tiny{[1.000, 1.000]} & \tiny{[0.999, 1.000]} & \tiny{[0.996, 0.998]} & \tiny{[1.000, 1.000]} & \tiny{[1.000, 1.000]} & \tiny{[1.000, 1.000]} & \tiny{[0.063, 0.103]} & \tiny{[0.001, 0.004]} & \tiny{[0.004, 0.022]} & \tiny{[0.016, 0.028]} \\
 & 64 & 0.996 & \textbf{1.000} & \textbf{1.000} & \textbf{1.000} & 0.975 & \textbf{1.000} & \textbf{1.000} & 0.999 & 0.254 & \textbf{0.006} & 0.017 & 0.036 \\
 &  & \tiny{[0.974, 0.997]} & \tiny{[1.000, 1.000]} & \tiny{[1.000, 1.000]} & \tiny{[0.999, 1.000]} & \tiny{[0.964, 0.980]} & \tiny{[1.000, 1.000]} & \tiny{[0.999, 1.000]} & \tiny{[0.998, 1.000]} & \tiny{[0.219, 0.307]} & \tiny{[0.004, 0.012]} & \tiny{[0.006, 0.031]} & \tiny{[0.027, 0.051]} \\
 & 128 & 0.993 & \textbf{1.000} & \textbf{1.000} & \textbf{1.000} & 0.882 & \textbf{1.000} & \textbf{1.000} & 0.998 & 0.490 & \textbf{0.024} & 0.030 & 0.057 \\
 &  & \tiny{[0.971, 0.995]} & \tiny{[1.000, 1.000]} & \tiny{[1.000, 1.000]} & \tiny{[0.999, 1.000]} & \tiny{[0.804, 0.932]} & \tiny{[0.999, 1.000]} & \tiny{[0.999, 1.000]} & \tiny{[0.995, 0.999]} & \tiny{[0.447, 0.549]} & \tiny{[0.020, 0.033]} & \tiny{[0.022, 0.051]} & \tiny{[0.042, 0.085]} \\
 & 256 & 0.990 & \textbf{1.000} & \textbf{1.000} & \textbf{1.000} & 0.545 & \textbf{0.997} & \textbf{0.997} & 0.994 & 0.830 & \textbf{0.097} & \textbf{0.099} & 0.118 \\
 &  & \tiny{[0.975, 0.992]} & \tiny{[1.000, 1.000]} & \tiny{[1.000, 1.000]} & \tiny{[0.999, 1.000]} & \tiny{[0.428, 0.717]} & \tiny{[0.995, 0.998]} & \tiny{[0.996, 0.998]} & \tiny{[0.989, 0.997]} & \tiny{[0.703, 0.949]} & \tiny{[0.063, 0.114]} & \tiny{[0.063, 0.113]} & \tiny{[0.089, 0.146]} \\
 & 512 & 0.990 & 0.966 & \textbf{1.000} & \textbf{1.000} & 0.234 & 0.424 & \textbf{0.986} & 0.978 & 1.223 & 1.984 & \textbf{0.203} & 0.234 \\
 &  & \tiny{[0.981, 0.992]} & \tiny{[0.817, 0.993]} & \tiny{[1.000, 1.000]} & \tiny{[0.999, 1.000]} & \tiny{[0.155, 0.431]} & \tiny{[0.228, 0.787]} & \tiny{[0.978, 0.990]} & \tiny{[0.965, 0.986]} & \tiny{[0.927, 1.480]} & \tiny{[0.851, 4.085]} & \tiny{[0.147, 0.275]} & \tiny{[0.177, 0.309]} \\
 & 1024 & 0.994 & -0.610 & \textbf{1.000} & \textbf{1.000} & 0.184 & 0.067 & \textbf{0.945} & 0.925 & 1.485 & 18.008 & \textbf{0.373} & 0.417 \\
 &  & \tiny{[0.990, 0.995]} & \tiny{[-7.243, 0.635]} & \tiny{[1.000, 1.000]} & \tiny{[0.999, 1.000]} & \tiny{[0.116, 0.379]} & \tiny{[0.028, 0.111]} & \tiny{[0.918, 0.962]} & \tiny{[0.890, 0.945]} & \tiny{[1.056, 1.708]} & \tiny{[8.654, 40.741]} & \tiny{[0.305, 0.485]} & \tiny{[0.341, 0.581]} \\
\cmidrule(l){1-14}
\multirow{12}{*}{Burgers 1D} & 32 & \textbf{0.940} & \textbf{0.946} & \textbf{0.941} & \textbf{0.939} & 0.998 & \textbf{1.000} & 0.999 & 0.997 & 0.081 & \textbf{0.025} & 0.050 & 0.078 \\
 &  & \tiny{[0.885, 0.982]} & \tiny{[0.886, 0.987]} & \tiny{[0.880, 0.987]} & \tiny{[0.873, 0.979]} & \tiny{[0.996, 0.999]} & \tiny{[0.999, 1.000]} & \tiny{[0.997, 1.000]} & \tiny{[0.995, 0.999]} & \tiny{[0.056, 0.105]} & \tiny{[0.008, 0.044]} & \tiny{[0.023, 0.095]} & \tiny{[0.061, 0.113]} \\
 & 64 & 0.924 & \textbf{0.957} & \textbf{0.956} & \textbf{0.952} & 0.979 & \textbf{0.999} & \textbf{0.999} & 0.995 & 0.256 & \textbf{0.049} & \textbf{0.043} & 0.100 \\
 &  & \tiny{[0.876, 0.978]} & \tiny{[0.900, 0.991]} & \tiny{[0.895, 0.991]} & \tiny{[0.887, 0.987]} & \tiny{[0.968, 0.985]} & \tiny{[0.995, 1.000]} & \tiny{[0.997, 1.000]} & \tiny{[0.991, 0.997]} & \tiny{[0.197, 0.314]} & \tiny{[0.016, 0.097]} & \tiny{[0.018, 0.090]} & \tiny{[0.073, 0.158]} \\
 & 128 & 0.932 & \textbf{0.966} & \textbf{0.965} & \textbf{0.960} & 0.900 & 0.992 & \textbf{0.999} & 0.992 & 0.488 & 0.118 & \textbf{0.046} & 0.129 \\
 &  & \tiny{[0.807, 0.980]} & \tiny{[0.928, 0.995]} & \tiny{[0.928, 0.995]} & \tiny{[0.906, 0.989]} & \tiny{[0.813, 0.940]} & \tiny{[0.973, 1.000]} & \tiny{[0.997, 1.000]} & \tiny{[0.985, 0.996]} & \tiny{[0.411, 0.554]} & \tiny{[0.043, 0.205]} & \tiny{[0.026, 0.081]} & \tiny{[0.094, 0.198]} \\
 & 256 & 0.944 & \textbf{0.971} & \textbf{0.971} & \textbf{0.966} & 0.596 & 0.946 & \textbf{0.997} & 0.985 & 0.805 & 0.317 & \textbf{0.094} & 0.195 \\
 &  & \tiny{[0.834, 0.979]} & \tiny{[0.936, 0.996]} & \tiny{[0.941, 0.997]} & \tiny{[0.919, 0.991]} & \tiny{[0.373, 0.814]} & \tiny{[0.847, 0.997]} & \tiny{[0.995, 0.998]} & \tiny{[0.971, 0.992]} & \tiny{[0.684, 0.964]} & \tiny{[0.122, 0.574]} & \tiny{[0.064, 0.122]} & \tiny{[0.126, 0.263]} \\
 & 512 & 0.951 & 0.343 & \textbf{0.980} & \textbf{0.969} & 0.268 & 0.119 & \textbf{0.988} & 0.961 & 1.169 & 7.788 & \textbf{0.168} & 0.310 \\
 &  & \tiny{[0.866, 0.982]} & \tiny{[-2.848, 0.875]} & \tiny{[0.949, 0.998]} & \tiny{[0.930, 0.992]} & \tiny{[0.161, 0.548]} & \tiny{[0.041, 0.374]} & \tiny{[0.985, 0.990]} & \tiny{[0.936, 0.976]} & \tiny{[0.848, 1.425]} & \tiny{[2.454, 21.043]} & \tiny{[0.133, 0.267]} & \tiny{[0.212, 0.449]} \\
 & 1024 & 0.960 & -2.497 & \textbf{0.985} & 0.970 & 0.240 & 0.050 & \textbf{0.958} & 0.900 & 1.357 & 28.083 & \textbf{0.317} & 0.480 \\
 &  & \tiny{[0.891, 0.988]} & \tiny{[-13.453, -0.042]} & \tiny{[0.956, 0.998]} & \tiny{[0.937, 0.993]} & \tiny{[0.123, 0.554]} & \tiny{[0.023, 0.108]} & \tiny{[0.940, 0.968]} & \tiny{[0.848, 0.930]} & \tiny{[0.865, 1.642]} & \tiny{[12.454, 53.720]} & \tiny{[0.255, 0.481]} & \tiny{[0.363, 0.724]} \\
\cmidrule(l){1-14}
\multirow{12}{*}{Allen--Cahn 1D} & 32 & \textbf{0.997} & \textbf{0.997} & \textbf{0.997} & \textbf{0.997} & 0.999 & \textbf{1.000} & 0.998 & 0.999 & 0.044 & \textbf{0.010} & 0.058 & 0.036 \\
 &  & \tiny{[0.996, 0.999]} & \tiny{[0.996, 0.999]} & \tiny{[0.996, 0.998]} & \tiny{[0.996, 0.998]} & \tiny{[0.998, 0.999]} & \tiny{[1.000, 1.000]} & \tiny{[0.997, 0.999]} & \tiny{[0.999, 1.000]} & \tiny{[0.033, 0.051]} & \tiny{[0.006, 0.012]} & \tiny{[0.040, 0.069]} & \tiny{[0.027, 0.051]} \\
 & 64 & 0.995 & \textbf{0.998} & 0.997 & 0.997 & 0.974 & \textbf{1.000} & 0.994 & 0.997 & 0.208 & \textbf{0.026} & 0.101 & 0.077 \\
 &  & \tiny{[0.995, 0.997]} & \tiny{[0.997, 0.999]} & \tiny{[0.996, 0.999]} & \tiny{[0.997, 0.998]} & \tiny{[0.966, 0.986]} & \tiny{[0.999, 1.000]} & \tiny{[0.992, 0.998]} & \tiny{[0.996, 0.999]} & \tiny{[0.159, 0.234]} & \tiny{[0.016, 0.032]} & \tiny{[0.065, 0.117]} & \tiny{[0.056, 0.096]} \\
 & 128 & 0.994 & \textbf{0.998} & 0.997 & \textbf{0.998} & 0.893 & \textbf{0.997} & 0.987 & 0.996 & 0.416 & \textbf{0.072} & 0.155 & 0.090 \\
 &  & \tiny{[0.993, 0.995]} & \tiny{[0.997, 0.999]} & \tiny{[0.997, 0.999]} & \tiny{[0.997, 0.999]} & \tiny{[0.871, 0.922]} & \tiny{[0.996, 0.999]} & \tiny{[0.982, 0.995]} & \tiny{[0.988, 0.998]} & \tiny{[0.370, 0.454]} & \tiny{[0.042, 0.083]} & \tiny{[0.100, 0.179]} & \tiny{[0.058, 0.145]} \\
 & 256 & 0.992 & \textbf{0.998} & \textbf{0.998} & 0.997 & 0.700 & \textbf{0.980} & 0.977 & 0.952 & 0.718 & \textbf{0.191} & 0.220 & 0.306 \\
 &  & \tiny{[0.991, 0.993]} & \tiny{[0.998, 0.999]} & \tiny{[0.997, 0.999]} & \tiny{[0.997, 0.998]} & \tiny{[0.638, 0.738]} & \tiny{[0.974, 0.991]} & \tiny{[0.967, 0.991]} & \tiny{[0.912, 0.966]} & \tiny{[0.678, 0.775]} & \tiny{[0.130, 0.214]} & \tiny{[0.148, 0.257]} & \tiny{[0.261, 0.384]} \\
 & 512 & 0.992 & 0.874 & \textbf{0.998} & 0.997 & 0.481 & 0.225 & \textbf{0.966} & 0.827 & 1.020 & 4.344 & \textbf{0.292} & 0.542 \\
 &  & \tiny{[0.991, 0.993]} & \tiny{[0.540, 0.936]} & \tiny{[0.998, 0.999]} & \tiny{[0.996, 0.998]} & \tiny{[0.368, 0.545]} & \tiny{[0.112, 0.307]} & \tiny{[0.949, 0.983]} & \tiny{[0.778, 0.900]} & \tiny{[0.957, 1.138]} & \tiny{[3.133, 8.623]} & \tiny{[0.216, 0.343]} & \tiny{[0.419, 0.592]} \\
 & 1024 & 0.995 & 0.108 & \textbf{0.998} & 0.997 & 0.621 & 0.071 & \textbf{0.953} & 0.654 & 0.944 & 15.427 & \textbf{0.375} & 0.731 \\
 &  & \tiny{[0.994, 0.995]} & \tiny{[-4.202, 0.672]} & \tiny{[0.998, 0.999]} & \tiny{[0.996, 0.998]} & \tiny{[0.488, 0.698]} & \tiny{[0.021, 0.109]} & \tiny{[0.930, 0.966]} & \tiny{[0.612, 0.768]} & \tiny{[0.891, 1.045]} & \tiny{[9.121, 33.762]} & \tiny{[0.305, 0.430]} & \tiny{[0.618, 0.760]} \\
\cmidrule(l){1-14}
\multirow{12}{*}{Diff.--Sorp. 1D} & 32 & 0.972 & \textbf{0.985} & \textbf{0.985} & 0.926 & 0.998 & \textbf{1.000} & \textbf{1.000} & 0.960 & 0.086 & \textbf{0.005} & 0.029 & 0.278 \\
 &  & \tiny{[0.933, 0.998]} & \tiny{[0.962, 0.999]} & \tiny{[0.961, 0.999]} & \tiny{[0.486, 0.991]} & \tiny{[0.991, 0.999]} & \tiny{[1.000, 1.000]} & \tiny{[1.000, 1.000]} & \tiny{[0.824, 0.995]} & \tiny{[0.062, 0.199]} & \tiny{[0.004, 0.012]} & \tiny{[0.015, 0.033]} & \tiny{[0.098, 0.608]} \\
 & 64 & 0.964 & \textbf{0.991} & \textbf{0.990} & 0.947 & 0.979 & \textbf{1.000} & \textbf{1.000} & 0.945 & 0.270 & \textbf{0.014} & 0.038 & 0.327 \\
 &  & \tiny{[0.822, 0.989]} & \tiny{[0.975, 0.999]} & \tiny{[0.975, 0.999]} & \tiny{[0.280, 0.989]} & \tiny{[0.965, 0.989]} & \tiny{[1.000, 1.000]} & \tiny{[1.000, 1.000]} & \tiny{[0.727, 0.988]} & \tiny{[0.199, 0.397]} & \tiny{[0.009, 0.032]} & \tiny{[0.014, 0.043]} & \tiny{[0.160, 0.773]} \\
 & 128 & 0.927 & \textbf{0.997} & \textbf{0.997} & 0.968 & 0.920 & \textbf{0.999} & \textbf{0.999} & 0.931 & 0.479 & \textbf{0.037} & 0.047 & 0.373 \\
 &  & \tiny{[0.708, 0.983]} & \tiny{[0.987, 1.000]} & \tiny{[0.987, 1.000]} & \tiny{[0.126, 0.993]} & \tiny{[0.852, 0.950]} & \tiny{[0.999, 1.000]} & \tiny{[0.999, 1.000]} & \tiny{[0.657, 0.991]} & \tiny{[0.393, 0.671]} & \tiny{[0.026, 0.078]} & \tiny{[0.019, 0.084]} & \tiny{[0.138, 0.896]} \\
 & 256 & 0.901 & \textbf{0.998} & \textbf{0.998} & 0.983 & 0.769 & 0.996 & \textbf{0.998} & 0.913 & 0.707 & \textbf{0.092} & \textbf{0.090} & 0.444 \\
 &  & \tiny{[0.737, 0.980]} & \tiny{[0.982, 1.000]} & \tiny{[0.983, 1.000]} & \tiny{[0.388, 0.996]} & \tiny{[0.672, 0.835]} & \tiny{[0.994, 0.999]} & \tiny{[0.995, 0.999]} & \tiny{[0.607, 0.963]} & \tiny{[0.578, 0.777]} & \tiny{[0.063, 0.187]} & \tiny{[0.041, 0.188]} & \tiny{[0.266, 1.024]} \\
 & 512 & 0.923 & 0.746 & \textbf{0.998} & 0.991 & 0.521 & 0.569 & \textbf{0.994} & 0.822 & 0.875 & 1.856 & \textbf{0.176} & 0.582 \\
 &  & \tiny{[0.801, 0.982]} & \tiny{[-0.381, 0.955]} & \tiny{[0.983, 1.000]} & \tiny{[0.645, 0.996]} & \tiny{[0.410, 0.594]} & \tiny{[0.263, 0.848]} & \tiny{[0.980, 0.997]} & \tiny{[0.552, 0.924]} & \tiny{[0.811, 0.929]} & \tiny{[0.809, 3.692]} & \tiny{[0.097, 0.437]} & \tiny{[0.457, 1.251]} \\
 & 1024 & 0.956 & -0.164 & \textbf{0.999} & 0.994 & 0.403 & 0.174 & \textbf{0.971} & 0.704 & 0.985 & 7.074 & \textbf{0.342} & 0.766 \\
 &  & \tiny{[0.862, 0.989]} & \tiny{[-5.381, 0.562]} & \tiny{[0.973, 1.000]} & \tiny{[0.782, 0.997]} & \tiny{[0.340, 0.578]} & \tiny{[0.086, 0.287]} & \tiny{[0.923, 0.982]} & \tiny{[0.446, 0.897]} & \tiny{[0.849, 1.093]} & \tiny{[4.726, 14.885]} & \tiny{[0.260, 1.076]} & \tiny{[0.552, 1.866]} \\
\cmidrule(l){1-14}
\multirow{12}{*}{Advection 1D} & 32 & \textbf{0.912} & \textbf{0.912} & \textbf{0.895} & \textbf{0.898} & \textbf{0.995} & \textbf{0.995} & 0.954 & 0.968 & 0.112 & \textbf{0.095} & 0.308 & 0.251 \\
 &  & \tiny{[0.866, 0.947]} & \tiny{[0.865, 0.947]} & \tiny{[0.764, 0.926]} & \tiny{[0.797, 0.931]} & \tiny{[0.991, 0.997]} & \tiny{[0.992, 0.998]} & \tiny{[0.879, 0.980]} & \tiny{[0.930, 0.984]} & \tiny{[0.083, 0.139]} & \tiny{[0.069, 0.126]} & \tiny{[0.222, 0.488]} & \tiny{[0.179, 0.389]} \\
 & 64 & \textbf{0.937} & \textbf{0.944} & \textbf{0.926} & \textbf{0.928} & 0.974 & \textbf{0.982} & 0.934 & 0.953 & 0.253 & \textbf{0.182} & 0.381 & 0.308 \\
 &  & \tiny{[0.880, 0.958]} & \tiny{[0.887, 0.966]} & \tiny{[0.800, 0.952]} & \tiny{[0.821, 0.951]} & \tiny{[0.958, 0.981]} & \tiny{[0.972, 0.991]} & \tiny{[0.828, 0.977]} & \tiny{[0.890, 0.972]} & \tiny{[0.216, 0.316]} & \tiny{[0.129, 0.228]} & \tiny{[0.242, 0.610]} & \tiny{[0.229, 0.514]} \\
 & 128 & \textbf{0.948} & \textbf{0.957} & \textbf{0.946} & \textbf{0.953} & 0.927 & \textbf{0.943} & 0.914 & \textbf{0.940} & 0.423 & \textbf{0.331} & 0.439 & 0.375 \\
 &  & \tiny{[0.893, 0.966]} & \tiny{[0.898, 0.972]} & \tiny{[0.846, 0.970]} & \tiny{[0.868, 0.969]} & \tiny{[0.901, 0.955]} & \tiny{[0.912, 0.971]} & \tiny{[0.769, 0.971]} & \tiny{[0.886, 0.968]} & \tiny{[0.349, 0.468]} & \tiny{[0.234, 0.444]} & \tiny{[0.269, 0.805]} & \tiny{[0.258, 0.556]} \\
 & 256 & \textbf{0.955} & \textbf{0.955} & \textbf{0.961} & \textbf{0.962} & 0.795 & 0.814 & \textbf{0.903} & \textbf{0.931} & 0.591 & 0.668 & 0.506 & \textbf{0.397} \\
 &  & \tiny{[0.914, 0.974]} & \tiny{[0.895, 0.977]} & \tiny{[0.871, 0.979]} & \tiny{[0.890, 0.979]} & \tiny{[0.688, 0.926]} & \tiny{[0.712, 0.901]} & \tiny{[0.693, 0.966]} & \tiny{[0.865, 0.966]} & \tiny{[0.436, 0.693]} & \tiny{[0.455, 0.903]} & \tiny{[0.285, 0.927]} & \tiny{[0.295, 0.554]} \\
 & 512 & \textbf{0.962} & -2.277 & \textbf{0.966} & \textbf{0.960} & 0.663 & 0.077 & \textbf{0.894} & \textbf{0.841} & 0.773 & 10.774 & \textbf{0.540} & \textbf{0.600} \\
 &  & \tiny{[0.926, 0.977]} & \tiny{[-16.345, -0.098]} & \tiny{[0.886, 0.985]} & \tiny{[0.886, 0.982]} & \tiny{[0.370, 0.911]} & \tiny{[0.032, 0.129]} & \tiny{[0.678, 0.964]} & \tiny{[0.748, 0.906]} & \tiny{[0.458, 0.990]} & \tiny{[6.883, 24.086]} & \tiny{[0.311, 0.965]} & \tiny{[0.461, 0.710]} \\
 & 1024 & \textbf{0.971} & -8.525 & \textbf{0.970} & \textbf{0.963} & 0.705 & 0.045 & \textbf{0.874} & 0.641 & 0.758 & 24.505 & \textbf{0.604} & 0.799 \\
 &  & \tiny{[0.943, 0.983]} & \tiny{[-39.101, -1.702]} & \tiny{[0.902, 0.987]} & \tiny{[0.892, 0.984]} & \tiny{[0.346, 0.946]} & \tiny{[0.011, 0.080]} & \tiny{[0.610, 0.974]} & \tiny{[0.575, 0.780]} & \tiny{[0.370, 1.056]} & \tiny{[12.215, 46.594]} & \tiny{[0.340, 1.053]} & \tiny{[0.634, 0.859]} \\
\bottomrule
\end{tabular}

%% file: tables/all_pde_metrics_full_bot.tex
\small
\begin{tabular}{llcccc!{\vrule width 0.3pt}cccc!{\vrule width 0.3pt}cccc}
\toprule
 &  & \multicolumn{4}{c!{\vrule width 0.3pt}}{R$^2$ $\uparrow$} & \multicolumn{4}{c!{\vrule width 0.3pt}}{Pearson $\uparrow$} & \multicolumn{4}{c}{NRMSE $\downarrow$} \\
PDE & $n$ & RMSR & KS & IG & OS (ours) & RMSR & KS & IG & OS (ours) & RMSR & KS & IG & OS (ours) \\
\midrule
\multirow{6}{*}{Heat 2D} & 64 & 0.987 & \textbf{1.000} & \textbf{1.000} & \textbf{1.000} & 0.963 & \textbf{1.000} & \textbf{1.000} & 0.999 & 0.314 & \textbf{0.009} & \textbf{0.010} & 0.034 \\
 &  & \tiny{[0.952, 0.995]} & \tiny{[1.000, 1.000]} & \tiny{[1.000, 1.000]} & \tiny{[0.999, 1.000]} & \tiny{[0.957, 0.973]} & \tiny{[1.000, 1.000]} & \tiny{[1.000, 1.000]} & \tiny{[0.999, 1.000]} & \tiny{[0.269, 0.346]} & \tiny{[0.007, 0.011]} & \tiny{[0.008, 0.013]} & \tiny{[0.028, 0.040]} \\
 & 256 & 0.988 & \textbf{1.000} & \textbf{1.000} & \textbf{1.000} & 0.666 & \textbf{0.994} & \textbf{0.994} & \textbf{0.993} & 0.721 & \textbf{0.117} & \textbf{0.117} & 0.129 \\
 &  & \tiny{[0.958, 0.992]} & \tiny{[1.000, 1.000]} & \tiny{[1.000, 1.000]} & \tiny{[1.000, 1.000]} & \tiny{[0.522, 0.783]} & \tiny{[0.992, 0.996]} & \tiny{[0.992, 0.996]} & \tiny{[0.990, 0.995]} & \tiny{[0.664, 0.851]} & \tiny{[0.092, 0.189]} & \tiny{[0.091, 0.190]} & \tiny{[0.102, 0.197]} \\
 & 1024 & 0.993 & -0.023 & \textbf{1.000} & \textbf{1.000} & 0.198 & 0.082 & \textbf{0.933} & 0.929 & 1.298 & 13.740 & \textbf{0.394} & \textbf{0.391} \\
 &  & \tiny{[0.984, 0.995]} & \tiny{[-6.136, 0.773]} & \tiny{[1.000, 1.000]} & \tiny{[1.000, 1.000]} & \tiny{[0.136, 0.377]} & \tiny{[0.038, 0.153]} & \tiny{[0.922, 0.943]} & \tiny{[0.919, 0.943]} & \tiny{[0.966, 1.558]} & \tiny{[6.950, 31.544]} & \tiny{[0.336, 0.553]} & \tiny{[0.337, 0.539]} \\
\cmidrule(l){1-14}
\multirow{6}{*}{Darcy 2D} & 64 & 0.854 & \textbf{0.915} & 0.904 & 0.904 & 0.977 & \textbf{0.999} & 0.993 & 0.993 & 0.304 & \textbf{0.041} & 0.155 & 0.124 \\
 &  & \tiny{[0.771, 0.912]} & \tiny{[0.848, 0.943]} & \tiny{[0.835, 0.935]} & \tiny{[0.831, 0.934]} & \tiny{[0.967, 0.984]} & \tiny{[0.999, 1.000]} & \tiny{[0.986, 0.996]} & \tiny{[0.988, 0.995]} & \tiny{[0.222, 0.366]} & \tiny{[0.029, 0.055]} & \tiny{[0.125, 0.179]} & \tiny{[0.099, 0.154]} \\
 & 256 & 0.895 & \textbf{0.942} & \textbf{0.941} & \textbf{0.941} & 0.876 & \textbf{0.986} & \textbf{0.987} & 0.982 & 0.561 & \textbf{0.216} & 0.246 & \textbf{0.216} \\
 &  & \tiny{[0.726, 0.932]} & \tiny{[0.892, 0.969]} & \tiny{[0.887, 0.968]} & \tiny{[0.886, 0.968]} & \tiny{[0.788, 0.902]} & \tiny{[0.975, 0.990]} & \tiny{[0.984, 0.991]} & \tiny{[0.974, 0.986]} & \tiny{[0.497, 0.700]} & \tiny{[0.180, 0.281]} & \tiny{[0.201, 0.294]} & \tiny{[0.181, 0.255]} \\
 & 1024 & 0.931 & -6.410 & \textbf{0.951} & \textbf{0.950} & 0.736 & 0.084 & \textbf{0.957} & 0.946 & 0.694 & 13.905 & 0.485 & \textbf{0.444} \\
 &  & \tiny{[0.871, 0.960]} & \tiny{[-23.831, -1.521]} & \tiny{[0.897, 0.980]} & \tiny{[0.901, 0.980]} & \tiny{[0.567, 0.815]} & \tiny{[0.045, 0.166]} & \tiny{[0.935, 0.968]} & \tiny{[0.923, 0.959]} & \tiny{[0.611, 0.830]} & \tiny{[7.124, 31.089]} & \tiny{[0.366, 0.678]} & \tiny{[0.321, 0.552]} \\
\cmidrule(l){1-14}
\multirow{6}{*}{Reac.--Diff. 2D} & 64 & \textbf{0.230} & 0.211 & -1.252 & \textbf{0.246} & 0.944 & 0.922 & 0.733 & \textbf{0.980} & 0.364 & 0.422 & 2.564 & \textbf{0.208} \\
 &  & \tiny{[-0.003, 0.323]} & \tiny{[-0.008, 0.330]} & \tiny{[-3.417, 0.138]} & \tiny{[0.023, 0.355]} & \tiny{[0.927, 0.954]} & \tiny{[0.902, 0.938]} & \tiny{[0.680, 0.794]} & \tiny{[0.969, 0.990]} & \tiny{[0.320, 0.407]} & \tiny{[0.378, 0.483]} & \tiny{[0.881, 4.700]} & \tiny{[0.154, 0.254]} \\
 & 256 & 0.210 & -0.083 & -0.101 & \textbf{0.292} & 0.790 & 0.663 & 0.875 & \textbf{0.932} & 0.642 & 1.392 & 1.457 & \textbf{0.386} \\
 &  & \tiny{[0.104, 0.329]} & \tiny{[-0.355, 0.168]} & \tiny{[-0.651, 0.256]} & \tiny{[0.174, 0.426]} & \tiny{[0.746, 0.841]} & \tiny{[0.585, 0.741]} & \tiny{[0.817, 0.907]} & \tiny{[0.910, 0.959]} & \tiny{[0.574, 0.692]} & \tiny{[1.120, 1.973]} & \tiny{[0.915, 2.375]} & \tiny{[0.313, 0.459]} \\
 & 1024 & 0.456 & -252.302 & 0.350 & \textbf{0.527} & 0.804 & 0.048 & 0.789 & \textbf{0.847} & \textbf{0.628} & 51.625 & 1.501 & \textbf{0.630} \\
 &  & \tiny{[0.240, 0.576]} & \tiny{[-1119.431, -69.746]} & \tiny{[0.013, 0.530]} & \tiny{[0.301, 0.626]} & \tiny{[0.719, 0.853]} & \tiny{[0.012, 0.083]} & \tiny{[0.683, 0.859]} & \tiny{[0.804, 0.900]} & \tiny{[0.546, 0.714]} & \tiny{[25.210, 103.550]} & \tiny{[1.184, 1.907]} & \tiny{[0.523, 0.740]} \\
\cmidrule(l){1-14}
\multirow{6}{*}{Shallow Water 2D} & 64 & 0.499 & 0.550 & 0.421 & \textbf{0.570} & 0.949 & 0.970 & 0.729 & \textbf{0.996} & 0.416 & 0.258 & 0.731 & \textbf{0.093} \\
 &  & \tiny{[0.408, 0.600]} & \tiny{[0.456, 0.624]} & \tiny{[0.247, 0.476]} & \tiny{[0.479, 0.644]} & \tiny{[0.940, 0.957]} & \tiny{[0.961, 0.976]} & \tiny{[0.620, 0.810]} & \tiny{[0.995, 0.997]} & \tiny{[0.372, 0.440]} & \tiny{[0.226, 0.291]} & \tiny{[0.617, 0.855]} & \tiny{[0.085, 0.106]} \\
 & 256 & 0.803 & 0.864 & 0.869 & \textbf{0.881} & 0.850 & 0.927 & 0.915 & \textbf{0.964} & 0.575 & 0.482 & 0.455 & \textbf{0.311} \\
 &  & \tiny{[0.745, 0.874]} & \tiny{[0.824, 0.892]} & \tiny{[0.825, 0.897]} & \tiny{[0.841, 0.910]} & \tiny{[0.834, 0.899]} & \tiny{[0.906, 0.936]} & \tiny{[0.892, 0.927]} & \tiny{[0.951, 0.972]} & \tiny{[0.451, 0.606]} & \tiny{[0.451, 0.522]} & \tiny{[0.425, 0.497]} & \tiny{[0.288, 0.364]} \\
 & 1024 & 0.951 & -3.861 & \textbf{0.980} & 0.974 & 0.910 & 0.136 & 0.898 & \textbf{0.923} & \textbf{0.416} & 9.197 & 0.529 & 0.527 \\
 &  & \tiny{[0.929, 0.965]} & \tiny{[-23.325, -0.627]} & \tiny{[0.971, 0.986]} & \tiny{[0.964, 0.982]} & \tiny{[0.880, 0.932]} & \tiny{[0.050, 0.216]} & \tiny{[0.879, 0.921]} & \tiny{[0.898, 0.939]} & \tiny{[0.363, 0.477]} & \tiny{[5.292, 22.002]} & \tiny{[0.433, 0.637]} & \tiny{[0.454, 0.598]} \\
\cmidrule(l){1-14}
\multirow{2}{*}{Heat 3D} & 512 & 0.980 & 0.974 & \textbf{1.000} & 0.996 & 0.531 & 0.726 & \textbf{0.979} & 0.906 & 0.854 & 1.036 & \textbf{0.255} & 0.432 \\
 &  & \tiny{[0.940, 0.988]} & \tiny{[0.659, 0.995]} & \tiny{[1.000, 1.000]} & \tiny{[0.986, 0.998]} & \tiny{[0.378, 0.679]} & \tiny{[0.329, 0.926]} & \tiny{[0.975, 0.983]} & \tiny{[0.896, 0.924]} & \tiny{[0.747, 0.994]} & \tiny{[0.516, 2.981]} & \tiny{[0.204, 0.328]} & \tiny{[0.403, 0.484]} \\
\cmidrule(l){1-14}
\multirow{2}{*}{MGN Cylinder} & 64 & \textbf{0.921} & \textbf{0.923} & 0.900 & \textbf{0.910} & 0.985 & \textbf{0.991} & 0.916 & 0.946 & 0.181 & \textbf{0.130} & 0.415 & 0.315 \\
 &  & \tiny{[0.824, 0.963]} & \tiny{[0.827, 0.963]} & \tiny{[0.775, 0.943]} & \tiny{[0.783, 0.955]} & \tiny{[0.977, 0.992]} & \tiny{[0.984, 0.993]} & \tiny{[0.830, 0.951]} & \tiny{[0.890, 0.983]} & \tiny{[0.138, 0.235]} & \tiny{[0.116, 0.170]} & \tiny{[0.338, 0.558]} & \tiny{[0.185, 0.472]} \\
\cmidrule(l){1-14}
\multirow{2}{*}{MGN Airfoil} & 64 & \textbf{0.822} & \textbf{0.822} & 0.727 & 0.751 & \textbf{0.970} & \textbf{0.977} & 0.727 & 0.737 & 0.255 & \textbf{0.212} & 1.004 & 0.657 \\
 &  & \tiny{[0.478, 0.879]} & \tiny{[0.470, 0.886]} & \tiny{[0.232, 0.817]} & \tiny{[0.444, 0.869]} & \tiny{[0.956, 0.981]} & \tiny{[0.954, 0.984]} & \tiny{[0.593, 0.848]} & \tiny{[0.580, 0.875]} & \tiny{[0.209, 0.317]} & \tiny{[0.182, 0.290]} & \tiny{[0.811, 1.363]} & \tiny{[0.465, 0.759]} \\
\bottomrule
\end{tabular}

%% file: tables/holdout.tex
\small
\begin{tabular}{llcccc!{\vrule width 0.3pt}cccc!{\vrule width 0.3pt}cccc}
\toprule
 &  & \multicolumn{4}{c!{\vrule width 0.3pt}}{R$^2$ $\uparrow$} & \multicolumn{4}{c!{\vrule width 0.3pt}}{Pearson $\uparrow$} & \multicolumn{4}{c}{NRMSE $\downarrow$} \\
Setting & $n$ & RMSR & KS & IG & OS (ours) & RMSR & KS & IG & OS (ours) & RMSR & KS & IG & OS (ours) \\
\midrule
\multirow{4}{*}{Burgers 1D, full ladder} & 256 & 0.944 & \textbf{0.971} & \textbf{0.971} & \textbf{0.966} & 0.596 & 0.946 & \textbf{0.997} & 0.985 & 0.805 & 0.317 & \textbf{0.094} & 0.195 \\
 &  & \tiny{[0.834, 0.979]} & \tiny{[0.936, 0.996]} & \tiny{[0.941, 0.997]} & \tiny{[0.919, 0.991]} & \tiny{[0.373, 0.814]} & \tiny{[0.847, 0.997]} & \tiny{[0.995, 0.998]} & \tiny{[0.971, 0.992]} & \tiny{[0.684, 0.964]} & \tiny{[0.122, 0.574]} & \tiny{[0.064, 0.122]} & \tiny{[0.126, 0.263]} \\
 & 1024 & 0.960 & -2.497 & \textbf{0.985} & 0.970 & 0.240 & 0.050 & \textbf{0.958} & 0.900 & 1.357 & 28.083 & \textbf{0.317} & 0.480 \\
 &  & \tiny{[0.891, 0.988]} & \tiny{[-13.453, -0.042]} & \tiny{[0.956, 0.998]} & \tiny{[0.937, 0.993]} & \tiny{[0.123, 0.554]} & \tiny{[0.023, 0.108]} & \tiny{[0.940, 0.968]} & \tiny{[0.848, 0.930]} & \tiny{[0.865, 1.642]} & \tiny{[12.454, 53.720]} & \tiny{[0.255, 0.481]} & \tiny{[0.363, 0.724]} \\
\cmidrule(l){1-14}
\multirow{2}{*}{Burgers 1D, $n=256$ held out} & 256 & 0.944 & \textbf{0.971} & \textbf{0.971} & \textbf{0.965} & 0.596 & 0.946 & \textbf{0.997} & 0.985 & 0.805 & 0.317 & \textbf{0.094} & 0.195 \\
 &  & \tiny{[0.834, 0.979]} & \tiny{[0.936, 0.996]} & \tiny{[0.941, 0.997]} & \tiny{[0.919, 0.991]} & \tiny{[0.373, 0.814]} & \tiny{[0.847, 0.997]} & \tiny{[0.995, 0.998]} & \tiny{[0.973, 0.991]} & \tiny{[0.684, 0.964]} & \tiny{[0.122, 0.574]} & \tiny{[0.064, 0.122]} & \tiny{[0.129, 0.260]} \\
\cmidrule(l){1-14}
\multirow{2}{*}{Burgers 1D, $n=1024$ held out} & 1024 & 0.960 & -2.497 & \textbf{0.985} & 0.970 & 0.240 & 0.050 & \textbf{0.958} & 0.900 & 1.357 & 28.083 & \textbf{0.317} & 0.481 \\
 &  & \tiny{[0.891, 0.988]} & \tiny{[-13.453, -0.042]} & \tiny{[0.956, 0.998]} & \tiny{[0.937, 0.993]} & \tiny{[0.123, 0.554]} & \tiny{[0.023, 0.108]} & \tiny{[0.940, 0.968]} & \tiny{[0.845, 0.930]} & \tiny{[0.865, 1.642]} & \tiny{[12.454, 53.720]} & \tiny{[0.255, 0.481]} & \tiny{[0.361, 0.723]} \\
\cmidrule(l){1-14}
\multirow{4}{*}{Heat 2D, full ladder} & 256 & 0.988 & \textbf{1.000} & \textbf{1.000} & \textbf{1.000} & 0.666 & \textbf{0.994} & \textbf{0.994} & \textbf{0.993} & 0.721 & \textbf{0.117} & \textbf{0.117} & 0.129 \\
 &  & \tiny{[0.958, 0.992]} & \tiny{[1.000, 1.000]} & \tiny{[1.000, 1.000]} & \tiny{[1.000, 1.000]} & \tiny{[0.522, 0.783]} & \tiny{[0.992, 0.996]} & \tiny{[0.992, 0.996]} & \tiny{[0.990, 0.995]} & \tiny{[0.664, 0.851]} & \tiny{[0.092, 0.189]} & \tiny{[0.091, 0.190]} & \tiny{[0.102, 0.197]} \\
 & 1024 & 0.993 & -0.023 & \textbf{1.000} & \textbf{1.000} & 0.198 & 0.082 & \textbf{0.933} & 0.929 & 1.298 & 13.740 & \textbf{0.394} & \textbf{0.391} \\
 &  & \tiny{[0.984, 0.995]} & \tiny{[-6.136, 0.773]} & \tiny{[1.000, 1.000]} & \tiny{[1.000, 1.000]} & \tiny{[0.136, 0.377]} & \tiny{[0.038, 0.153]} & \tiny{[0.922, 0.943]} & \tiny{[0.919, 0.943]} & \tiny{[0.966, 1.558]} & \tiny{[6.950, 31.544]} & \tiny{[0.336, 0.553]} & \tiny{[0.337, 0.539]} \\
\cmidrule(l){1-14}
\multirow{2}{*}{Heat 2D, $n=256$ held out} & 256 & 0.988 & \textbf{1.000} & \textbf{1.000} & 0.999 & 0.666 & \textbf{0.994} & \textbf{0.994} & 0.980 & 0.721 & \textbf{0.117} & \textbf{0.117} & 0.203 \\
 &  & \tiny{[0.958, 0.992]} & \tiny{[1.000, 1.000]} & \tiny{[1.000, 1.000]} & \tiny{[0.996, 1.000]} & \tiny{[0.522, 0.783]} & \tiny{[0.992, 0.996]} & \tiny{[0.992, 0.996]} & \tiny{[0.971, 0.983]} & \tiny{[0.664, 0.851]} & \tiny{[0.092, 0.189]} & \tiny{[0.091, 0.190]} & \tiny{[0.173, 0.239]} \\
\cmidrule(l){1-14}
\multirow{2}{*}{Heat 2D, $n=1024$ held out} & 1024 & 0.993 & -0.023 & \textbf{1.000} & 0.998 & 0.198 & 0.082 & \textbf{0.933} & 0.793 & 1.298 & 13.740 & \textbf{0.394} & 0.763 \\
 &  & \tiny{[0.984, 0.995]} & \tiny{[-6.136, 0.773]} & \tiny{[1.000, 1.000]} & \tiny{[0.991, 0.999]} & \tiny{[0.136, 0.377]} & \tiny{[0.038, 0.153]} & \tiny{[0.922, 0.943]} & \tiny{[0.749, 0.825]} & \tiny{[0.966, 1.558]} & \tiny{[6.950, 31.544]} & \tiny{[0.336, 0.553]} & \tiny{[0.603, 1.030]} \\
\bottomrule
\end{tabular}

%% file: tables/mean_baseline.tex
\small
\begin{tabular}{llcccc!{\vrule width 0.3pt}cccc!{\vrule width 0.3pt}cccc}
\toprule
 &  & \multicolumn{4}{c!{\vrule width 0.3pt}}{R$^2$ $\uparrow$} & \multicolumn{4}{c!{\vrule width 0.3pt}}{Pearson $\uparrow$} & \multicolumn{4}{c}{NRMSE $\downarrow$} \\
Baseline & $n$ & RMSR & KS & IG & OS (ours) & RMSR & KS & IG & OS (ours) & RMSR & KS & IG & OS (ours) \\
\midrule
\multirow{4}{*}{Darcy 2D, zero} & 64 & 0.854 & \textbf{0.915} & 0.904 & 0.904 & 0.977 & \textbf{0.999} & 0.993 & 0.993 & 0.304 & \textbf{0.041} & 0.155 & 0.124 \\
 &  & \tiny{[0.771, 0.912]} & \tiny{[0.848, 0.943]} & \tiny{[0.835, 0.935]} & \tiny{[0.831, 0.934]} & \tiny{[0.967, 0.984]} & \tiny{[0.999, 1.000]} & \tiny{[0.986, 0.996]} & \tiny{[0.988, 0.995]} & \tiny{[0.222, 0.366]} & \tiny{[0.029, 0.055]} & \tiny{[0.125, 0.179]} & \tiny{[0.099, 0.154]} \\
 & 1024 & 0.931 & -6.410 & \textbf{0.951} & \textbf{0.950} & 0.736 & 0.084 & \textbf{0.957} & 0.946 & 0.694 & 13.905 & 0.485 & \textbf{0.444} \\
 &  & \tiny{[0.871, 0.960]} & \tiny{[-23.831, -1.521]} & \tiny{[0.897, 0.980]} & \tiny{[0.901, 0.980]} & \tiny{[0.567, 0.815]} & \tiny{[0.045, 0.166]} & \tiny{[0.935, 0.968]} & \tiny{[0.923, 0.959]} & \tiny{[0.611, 0.830]} & \tiny{[7.124, 31.089]} & \tiny{[0.366, 0.678]} & \tiny{[0.321, 0.552]} \\
\cmidrule(l){1-14}
\multirow{4}{*}{Darcy 2D, mean} & 64 & 0.962 & \textbf{0.975} & \textbf{0.973} & 0.963 & 0.966 & \textbf{1.000} & 0.998 & 0.974 & 0.282 & \textbf{0.016} & 0.059 & 0.229 \\
 &  & \tiny{[0.893, 0.978]} & \tiny{[0.941, 0.990]} & \tiny{[0.938, 0.990]} & \tiny{[0.907, 0.981]} & \tiny{[0.951, 0.977]} & \tiny{[1.000, 1.000]} & \tiny{[0.996, 0.999]} & \tiny{[0.932, 0.984]} & \tiny{[0.232, 0.380]} & \tiny{[0.012, 0.022]} & \tiny{[0.048, 0.084]} & \tiny{[0.181, 0.356]} \\
 & 1024 & 0.976 & -1.143 & \textbf{0.987} & \textbf{0.985} & 0.364 & 0.075 & \textbf{0.955} & 0.855 & 1.045 & 14.531 & \textbf{0.345} & 0.549 \\
 &  & \tiny{[0.926, 0.988]} & \tiny{[-15.560, 0.544]} & \tiny{[0.948, 0.996]} & \tiny{[0.940, 0.995]} & \tiny{[0.215, 0.537]} & \tiny{[0.036, 0.153]} & \tiny{[0.933, 0.968]} & \tiny{[0.787, 0.928]} & \tiny{[0.842, 1.204]} & \tiny{[6.885, 39.518]} & \tiny{[0.279, 0.478]} & \tiny{[0.395, 0.668]} \\
\cmidrule(l){1-14}
\multirow{4}{*}{Diff.--Sorp. 1D, zero} & 64 & 0.964 & \textbf{0.991} & \textbf{0.990} & 0.947 & 0.979 & \textbf{1.000} & \textbf{1.000} & 0.945 & 0.270 & \textbf{0.014} & 0.038 & 0.327 \\
 &  & \tiny{[0.822, 0.989]} & \tiny{[0.975, 0.999]} & \tiny{[0.975, 0.999]} & \tiny{[0.280, 0.989]} & \tiny{[0.965, 0.989]} & \tiny{[1.000, 1.000]} & \tiny{[1.000, 1.000]} & \tiny{[0.727, 0.988]} & \tiny{[0.199, 0.397]} & \tiny{[0.009, 0.032]} & \tiny{[0.014, 0.043]} & \tiny{[0.160, 0.773]} \\
 & 1024 & 0.956 & -0.164 & \textbf{0.999} & 0.994 & 0.403 & 0.174 & \textbf{0.971} & 0.704 & 0.985 & 7.074 & \textbf{0.342} & 0.766 \\
 &  & \tiny{[0.862, 0.989]} & \tiny{[-5.381, 0.562]} & \tiny{[0.973, 1.000]} & \tiny{[0.782, 0.997]} & \tiny{[0.340, 0.578]} & \tiny{[0.086, 0.287]} & \tiny{[0.923, 0.982]} & \tiny{[0.446, 0.897]} & \tiny{[0.849, 1.093]} & \tiny{[4.726, 14.885]} & \tiny{[0.260, 1.076]} & \tiny{[0.552, 1.866]} \\
\cmidrule(l){1-14}
\multirow{4}{*}{Diff.--Sorp. 1D, mean} & 64 & 0.977 & \textbf{0.999} & \textbf{0.999} & 0.996 & 0.974 & \textbf{1.000} & \textbf{1.000} & 0.996 & 0.271 & \textbf{0.009} & 0.010 & 0.089 \\
 &  & \tiny{[0.973, 0.987]} & \tiny{[0.997, 1.000]} & \tiny{[0.997, 1.000]} & \tiny{[0.992, 0.997]} & \tiny{[0.968, 0.982]} & \tiny{[1.000, 1.000]} & \tiny{[1.000, 1.000]} & \tiny{[0.994, 0.998]} & \tiny{[0.204, 0.320]} & \tiny{[0.006, 0.010]} & \tiny{[0.006, 0.012]} & \tiny{[0.076, 0.117]} \\
 & 1024 & 0.995 & -3.924 & \textbf{1.000} & 0.915 & 0.991 & 0.131 & \textbf{0.995} & 0.571 & 0.137 & 6.925 & \textbf{0.102} & 0.838 \\
 &  & \tiny{[0.992, 0.997]} & \tiny{[-19.506, -0.645]} & \tiny{[0.999, 1.000]} & \tiny{[0.904, 0.932]} & \tiny{[0.947, 0.994]} & \tiny{[0.067, 0.267]} & \tiny{[0.985, 0.998]} & \tiny{[0.536, 0.595]} & \tiny{[0.106, 0.339]} & \tiny{[4.154, 15.157]} & \tiny{[0.063, 0.185]} & \tiny{[0.824, 0.858]} \\
\bottomrule
\end{tabular}

%% file: tables/relu.tex
\small
\begin{tabular}{llcccc!{\vrule width 0.3pt}cccc!{\vrule width 0.3pt}cccc}
\toprule
 &  & \multicolumn{4}{c!{\vrule width 0.3pt}}{R$^2$ $\uparrow$} & \multicolumn{4}{c!{\vrule width 0.3pt}}{Pearson $\uparrow$} & \multicolumn{4}{c}{NRMSE $\downarrow$} \\
Activation & $n$ & RMSR & KS & IG & OS (ours) & RMSR & KS & IG & OS (ours) & RMSR & KS & IG & OS (ours) \\
\midrule
\multirow{4}{*}{GELU} & 64 & 0.924 & \textbf{0.957} & \textbf{0.956} & \textbf{0.952} & 0.979 & \textbf{0.999} & \textbf{0.999} & 0.995 & 0.256 & \textbf{0.049} & \textbf{0.043} & 0.100 \\
 &  & \tiny{[0.876, 0.978]} & \tiny{[0.900, 0.991]} & \tiny{[0.895, 0.991]} & \tiny{[0.887, 0.987]} & \tiny{[0.968, 0.985]} & \tiny{[0.995, 1.000]} & \tiny{[0.997, 1.000]} & \tiny{[0.991, 0.997]} & \tiny{[0.197, 0.314]} & \tiny{[0.016, 0.097]} & \tiny{[0.018, 0.090]} & \tiny{[0.073, 0.158]} \\
 & 256 & 0.944 & \textbf{0.971} & \textbf{0.971} & \textbf{0.966} & 0.596 & 0.946 & \textbf{0.997} & 0.985 & 0.805 & 0.317 & \textbf{0.094} & 0.195 \\
 &  & \tiny{[0.834, 0.979]} & \tiny{[0.936, 0.996]} & \tiny{[0.941, 0.997]} & \tiny{[0.919, 0.991]} & \tiny{[0.373, 0.814]} & \tiny{[0.847, 0.997]} & \tiny{[0.995, 0.998]} & \tiny{[0.971, 0.992]} & \tiny{[0.684, 0.964]} & \tiny{[0.122, 0.574]} & \tiny{[0.064, 0.122]} & \tiny{[0.126, 0.263]} \\
\cmidrule(l){1-14}
\multirow{4}{*}{ReLU} & 64 & 0.918 & \textbf{0.963} & \textbf{0.964} & \textbf{0.948} & 0.977 & \textbf{0.998} & \textbf{0.998} & \textbf{0.995} & 0.262 & \textbf{0.055} & \textbf{0.071} & 0.099 \\
 &  & \tiny{[0.871, 0.974]} & \tiny{[0.910, 0.988]} & \tiny{[0.901, 0.988]} & \tiny{[0.899, 0.984]} & \tiny{[0.966, 0.983]} & \tiny{[0.992, 1.000]} & \tiny{[0.995, 1.000]} & \tiny{[0.991, 0.997]} & \tiny{[0.204, 0.317]} & \tiny{[0.020, 0.119]} & \tiny{[0.036, 0.115]} & \tiny{[0.084, 0.173]} \\
 & 256 & 0.937 & \textbf{0.976} & \textbf{0.978} & 0.965 & 0.561 & 0.936 & \textbf{0.994} & 0.990 & 0.820 & 0.329 & \textbf{0.121} & 0.174 \\
 &  & \tiny{[0.848, 0.978]} & \tiny{[0.933, 0.992]} & \tiny{[0.935, 0.993]} & \tiny{[0.922, 0.991]} & \tiny{[0.396, 0.785]} & \tiny{[0.760, 0.995]} & \tiny{[0.990, 0.997]} & \tiny{[0.978, 0.995]} & \tiny{[0.689, 0.982]} & \tiny{[0.136, 0.833]} & \tiny{[0.094, 0.165]} & \tiny{[0.114, 0.265]} \\
\bottomrule
\end{tabular}

%% file: tables/midpoint.tex
\small
\begin{tabular}{llcccc!{\vrule width 0.3pt}cccc!{\vrule width 0.3pt}cccc}
\toprule
 &  & \multicolumn{4}{c!{\vrule width 0.3pt}}{R$^2$ $\uparrow$} & \multicolumn{4}{c!{\vrule width 0.3pt}}{Pearson $\uparrow$} & \multicolumn{4}{c}{NRMSE $\downarrow$} \\
PDE & $n$ & RMSR & KS & IG & OS (ours) & RMSR & KS & IG & OS (ours) & RMSR & KS & IG & OS (ours) \\
\midrule
\multirow{8}{*}{Heat 1D} & 16 & \textbf{0.781} & \textbf{0.778} & -1.000 & \textbf{0.775} & \textbf{0.999} & 0.991 & -- & 0.986 & \textbf{0.040} & 0.127 & 1.000 & 0.170 \\
 &  & \tiny{[0.774, 0.795]} & \tiny{[0.767, 0.790]} & \tiny{[-1.003, -0.994]} & \tiny{[0.757, 0.785]} & \tiny{[0.999, 1.000]} & \tiny{[0.984, 0.995]} &  & \tiny{[0.970, 0.991]} & \tiny{[0.028, 0.050]} & \tiny{[0.098, 0.165]} & \tiny{[1.000, 1.000]} & \tiny{[0.136, 0.247]} \\
 & 64 & \textbf{0.818} & \textbf{0.799} & -0.999 & \textbf{0.836} & 0.757 & 0.727 & -- & \textbf{0.992} & 0.611 & 0.896 & 1.000 & \textbf{0.141} \\
 &  & \tiny{[0.794, 0.824]} & \tiny{[0.789, 0.805]} & \tiny{[-1.004, -0.991]} & \tiny{[0.825, 0.839]} & \tiny{[0.599, 0.914]} & \tiny{[0.595, 0.880]} &  & \tiny{[0.988, 0.995]} & \tiny{[0.432, 0.750]} & \tiny{[0.507, 1.139]} & \tiny{[1.000, 1.000]} & \tiny{[0.103, 0.199]} \\
 & 256 & \textbf{0.853} & 0.667 & -1.000 & \textbf{0.875} & 0.131 & 0.121 & -- & \textbf{0.975} & 1.522 & 4.663 & 1.000 & \textbf{0.231} \\
 &  & \tiny{[0.835, 0.858]} & \tiny{[0.627, 0.689]} & \tiny{[-1.003, -0.993]} & \tiny{[0.872, 0.876]} & \tiny{[0.036, 0.305]} & \tiny{[0.066, 0.235]} &  & \tiny{[0.967, 0.982]} & \tiny{[1.028, 1.837]} & \tiny{[2.710, 5.898]} & \tiny{[1.000, 1.000]} & \tiny{[0.204, 0.308]} \\
 & 1024 & \textbf{0.886} & -98.764 & -1.000 & \textbf{0.899} & 0.027 & 0.008 & -- & \textbf{0.799} & 2.398 & 165.848 & 1.000 & \textbf{0.624} \\
 &  & \tiny{[0.880, 0.889]} & \tiny{[-471.047, -25.488]} & \tiny{[-1.003, -0.993]} & \tiny{[0.898, 0.899]} & \tiny{[-0.003, 0.083]} & \tiny{[-0.015, 0.024]} &  & \tiny{[0.677, 0.912]} & \tiny{[1.537, 2.743]} & \tiny{[101.183, 363.842]} & \tiny{[1.000, 1.000]} & \tiny{[0.544, 0.701]} \\
\cmidrule(l){1-14}
\multirow{8}{*}{Burgers 1D} & 16 & \textbf{0.703} & \textbf{0.701} & -1.195 & \textbf{0.693} & \textbf{0.999} & 0.996 & -- & 0.983 & \textbf{0.036} & 0.088 & 1.000 & 0.191 \\
 &  & \tiny{[0.640, 0.770]} & \tiny{[0.636, 0.763]} & \tiny{[-1.567, -1.001]} & \tiny{[0.611, 0.753]} & \tiny{[0.999, 1.000]} & \tiny{[0.992, 0.998]} &  & \tiny{[0.969, 0.993]} & \tiny{[0.023, 0.049]} & \tiny{[0.067, 0.120]} & \tiny{[1.000, 1.000]} & \tiny{[0.131, 0.265]} \\
 & 64 & \textbf{0.765} & \textbf{0.764} & -1.120 & \textbf{0.776} & 0.884 & 0.855 & -- & \textbf{0.990} & 0.449 & 0.535 & 1.000 & \textbf{0.160} \\
 &  & \tiny{[0.630, 0.811]} & \tiny{[0.627, 0.799]} & \tiny{[-1.393, -1.004]} & \tiny{[0.648, 0.824]} & \tiny{[0.775, 0.946]} & \tiny{[0.748, 0.929]} &  & \tiny{[0.983, 0.993]} & \tiny{[0.348, 0.580]} & \tiny{[0.412, 0.761]} & \tiny{[1.000, 1.000]} & \tiny{[0.126, 0.214]} \\
 & 256 & \textbf{0.810} & 0.702 & -1.113 & \textbf{0.832} & 0.178 & 0.193 & -- & \textbf{0.975} & 1.221 & 3.391 & 1.000 & \textbf{0.232} \\
 &  & \tiny{[0.721, 0.848]} & \tiny{[0.631, 0.728]} & \tiny{[-1.280, -1.012]} & \tiny{[0.745, 0.867]} & \tiny{[0.083, 0.446]} & \tiny{[0.093, 0.385]} &  & \tiny{[0.964, 0.983]} & \tiny{[0.889, 1.499]} & \tiny{[2.121, 4.410]} & \tiny{[1.000, 1.000]} & \tiny{[0.182, 0.342]} \\
 & 1024 & \textbf{0.858} & -86.894 & -1.096 & \textbf{0.865} & 0.038 & 0.006 & -- & \textbf{0.831} & 1.828 & 173.467 & 1.000 & \textbf{0.564} \\
 &  & \tiny{[0.797, 0.882]} & \tiny{[-310.480, -30.642]} & \tiny{[-1.221, -1.016]} & \tiny{[0.807, 0.893]} & \tiny{[0.009, 0.129]} & \tiny{[-0.018, 0.029]} &  & \tiny{[0.779, 0.869]} & \tiny{[1.275, 2.228]} & \tiny{[82.409, 316.716]} & \tiny{[1.000, 1.000]} & \tiny{[0.493, 0.630]} \\
\cmidrule(l){1-14}
\multirow{8}{*}{Allen--Cahn 1D} & 16 & \textbf{0.788} & \textbf{0.783} & -1.014 & \textbf{0.787} & \textbf{1.000} & 0.989 & -- & 0.998 & \textbf{0.024} & 0.142 & 1.000 & 0.075 \\
 &  & \tiny{[0.783, 0.811]} & \tiny{[0.778, 0.808]} & \tiny{[-1.037, -1.002]} & \tiny{[0.783, 0.803]} & \tiny{[0.999, 1.000]} & \tiny{[0.984, 0.995]} &  & \tiny{[0.993, 0.999]} & \tiny{[0.017, 0.033]} & \tiny{[0.094, 0.161]} & \tiny{[1.000, 1.000]} & \tiny{[0.049, 0.115]} \\
 & 64 & 0.825 & 0.804 & -1.044 & \textbf{0.838} & 0.762 & 0.672 & -- & \textbf{0.984} & 0.614 & 0.981 & 1.000 & \textbf{0.159} \\
 &  & \tiny{[0.821, 0.827]} & \tiny{[0.800, 0.808]} & \tiny{[-1.054, -1.024]} & \tiny{[0.832, 0.839]} & \tiny{[0.688, 0.839]} & \tiny{[0.585, 0.719]} &  & \tiny{[0.980, 0.991]} & \tiny{[0.506, 0.676]} & \tiny{[0.893, 1.076]} & \tiny{[1.000, 1.000]} & \tiny{[0.119, 0.181]} \\
 & 256 & 0.858 & 0.687 & -1.032 & \textbf{0.874} & 0.122 & 0.103 & -- & \textbf{0.899} & 1.528 & 5.011 & 1.000 & \textbf{0.399} \\
 &  & \tiny{[0.853, 0.860]} & \tiny{[0.674, 0.705]} & \tiny{[-1.041, -1.017]} & \tiny{[0.874, 0.875]} & \tiny{[0.042, 0.201]} & \tiny{[0.051, 0.151]} &  & \tiny{[0.872, 0.928]} & \tiny{[1.378, 1.683]} & \tiny{[4.729, 5.458]} & \tiny{[1.000, 1.000]} & \tiny{[0.341, 0.442]} \\
 & 1024 & 0.888 & -69.217 & -1.027 & \textbf{0.897} & 0.018 & 0.002 & -- & \textbf{0.577} & 2.058 & 178.767 & 1.000 & \textbf{0.776} \\
 &  & \tiny{[0.886, 0.890]} & \tiny{[-314.641, -25.083]} & \tiny{[-1.031, -1.013]} & \tiny{[0.897, 0.898]} & \tiny{[-0.004, 0.065]} & \tiny{[-0.017, 0.022]} &  & \tiny{[0.541, 0.686]} & \tiny{[1.848, 2.269]} & \tiny{[100.427, 357.943]} & \tiny{[1.000, 1.000]} & \tiny{[0.712, 0.799]} \\
\cmidrule(l){1-14}
\multirow{8}{*}{Diff.--Sorp. 1D} & 16 & \textbf{0.800} & \textbf{0.799} & -0.982 & \textbf{0.799} & \textbf{1.000} & 0.998 & -- & 0.999 & \textbf{0.000} & 0.060 & 1.000 & 0.051 \\
 &  & \tiny{[0.661, 0.845]} & \tiny{[0.659, 0.843]} & \tiny{[-1.000, -0.746]} & \tiny{[0.655, 0.843]} & \tiny{[1.000, 1.000]} & \tiny{[0.996, 1.000]} &  & \tiny{[0.996, 0.999]} & \tiny{[0.000, 0.023]} & \tiny{[0.031, 0.091]} & \tiny{[1.000, 1.000]} & \tiny{[0.037, 0.086]} \\
 & 64 & \textbf{0.783} & \textbf{0.773} & -0.974 & \textbf{0.804} & 0.972 & 0.931 & -- & \textbf{0.993} & 0.257 & 0.380 & 1.000 & \textbf{0.122} \\
 &  & \tiny{[0.717, 0.820]} & \tiny{[0.710, 0.802]} & \tiny{[-0.989, -0.913]} & \tiny{[0.725, 0.833]} & \tiny{[0.940, 0.984]} & \tiny{[0.902, 0.955]} &  & \tiny{[0.990, 0.996]} & \tiny{[0.194, 0.403]} & \tiny{[0.319, 0.457]} & \tiny{[1.000, 1.000]} & \tiny{[0.100, 0.162]} \\
 & 256 & 0.793 & 0.628 & -0.983 & \textbf{0.860} & 0.582 & 0.450 & -- & \textbf{0.990} & 0.816 & 1.867 & 1.000 & \textbf{0.156} \\
 &  & \tiny{[0.713, 0.844]} & \tiny{[0.547, 0.665]} & \tiny{[-0.995, -0.924]} & \tiny{[0.802, 0.874]} & \tiny{[0.399, 0.707]} & \tiny{[0.229, 0.563]} &  & \tiny{[0.981, 0.994]} & \tiny{[0.733, 0.987]} & \tiny{[1.374, 2.627]} & \tiny{[1.000, 1.000]} & \tiny{[0.120, 0.268]} \\
 & 1024 & 0.867 & -105.891 & -0.984 & \textbf{0.893} & 0.209 & 0.004 & -- & \textbf{0.960} & 1.073 & 78.799 & 1.000 & \textbf{0.340} \\
 &  & \tiny{[0.824, 0.884]} & \tiny{[-447.042, -27.001]} & \tiny{[-0.995, -0.913]} & \tiny{[0.860, 0.896]} & \tiny{[0.119, 0.348]} & \tiny{[-0.016, 0.030]} &  & \tiny{[0.892, 0.978]} & \tiny{[0.975, 1.385]} & \tiny{[46.774, 155.232]} & \tiny{[1.000, 1.000]} & \tiny{[0.234, 0.484]} \\
\cmidrule(l){1-14}
\multirow{8}{*}{Advection 1D} & 16 & \textbf{0.897} & \textbf{0.896} & -0.991 & \textbf{0.890} & \textbf{0.999} & \textbf{0.999} & -- & 0.995 & \textbf{0.040} & 0.055 & 1.000 & 0.120 \\
 &  & \tiny{[0.760, 0.956]} & \tiny{[0.758, 0.955]} & \tiny{[-1.010, -0.937]} & \tiny{[0.723, 0.954]} & \tiny{[0.998, 1.000]} & \tiny{[0.996, 0.999]} &  & \tiny{[0.971, 0.998]} & \tiny{[0.024, 0.057]} & \tiny{[0.033, 0.084]} & \tiny{[1.000, 1.000]} & \tiny{[0.069, 0.266]} \\
 & 64 & \textbf{0.800} & \textbf{0.791} & -0.973 & \textbf{0.799} & \textbf{0.922} & 0.882 & -- & \textbf{0.936} & 0.408 & 0.499 & 1.000 & \textbf{0.366} \\
 &  & \tiny{[0.760, 0.815]} & \tiny{[0.757, 0.806]} & \tiny{[-1.038, -0.892]} & \tiny{[0.735, 0.816]} & \tiny{[0.882, 0.939]} & \tiny{[0.858, 0.907]} &  & \tiny{[0.880, 0.957]} & \tiny{[0.360, 0.467]} & \tiny{[0.452, 0.579]} & \tiny{[1.000, 1.000]} & \tiny{[0.308, 0.503]} \\
 & 256 & \textbf{0.843} & 0.708 & -0.981 & \textbf{0.835} & 0.396 & 0.287 & -- & \textbf{0.672} & 0.943 & 2.382 & 1.000 & \textbf{0.755} \\
 &  & \tiny{[0.808, 0.853]} & \tiny{[0.662, 0.727]} & \tiny{[-1.052, -0.893]} & \tiny{[0.788, 0.860]} & \tiny{[0.220, 0.806]} & \tiny{[0.217, 0.392]} &  & \tiny{[0.465, 0.859]} & \tiny{[0.607, 1.090]} & \tiny{[1.835, 2.913]} & \tiny{[1.000, 1.000]} & \tiny{[0.524, 0.898]} \\
 & 1024 & \textbf{0.879} & -81.406 & -0.988 & 0.861 & 0.238 & 0.011 & -- & \textbf{0.325} & 1.044 & 85.246 & 1.000 & \textbf{0.948} \\
 &  & \tiny{[0.859, 0.889]} & \tiny{[-370.086, -20.223]} & \tiny{[-1.046, -0.905]} & \tiny{[0.811, 0.886]} & \tiny{[0.052, 0.934]} & \tiny{[-0.013, 0.030]} &  & \tiny{[0.227, 0.572]} & \tiny{[0.383, 1.485]} & \tiny{[35.886, 190.258]} & \tiny{[1.000, 1.000]} & \tiny{[0.811, 0.978]} \\
\cmidrule(l){1-14}
\multirow{8}{*}{Heat 2D} & 16 & \textbf{0.766} & \textbf{0.763} & -1.003 & \textbf{0.767} & \textbf{0.993} & 0.988 & -- & \textbf{0.993} & \textbf{0.107} & 0.143 & 1.000 & 0.118 \\
 &  & \tiny{[0.685, 0.773]} & \tiny{[0.684, 0.771]} & \tiny{[-1.008, -1.000]} & \tiny{[0.680, 0.774]} & \tiny{[0.987, 0.996]} & \tiny{[0.974, 0.994]} &  & \tiny{[0.989, 0.996]} & \tiny{[0.083, 0.146]} & \tiny{[0.108, 0.199]} & \tiny{[1.000, 1.000]} & \tiny{[0.086, 0.147]} \\
 & 64 & 0.812 & 0.796 & -1.001 & \textbf{0.835} & 0.838 & 0.815 & -- & \textbf{0.990} & 0.549 & 0.671 & 1.000 & \textbf{0.135} \\
 &  & \tiny{[0.757, 0.822]} & \tiny{[0.767, 0.802]} & \tiny{[-1.003, -0.999]} & \tiny{[0.806, 0.837]} & \tiny{[0.703, 0.901]} & \tiny{[0.652, 0.912]} &  & \tiny{[0.988, 0.994]} & \tiny{[0.442, 0.682]} & \tiny{[0.434, 0.995]} & \tiny{[1.000, 1.000]} & \tiny{[0.119, 0.159]} \\
 & 256 & 0.846 & 0.661 & -1.001 & \textbf{0.875} & 0.198 & 0.180 & -- & \textbf{0.974} & 1.315 & 3.800 & 1.000 & \textbf{0.241} \\
 &  & \tiny{[0.817, 0.855]} & \tiny{[0.625, 0.683]} & \tiny{[-1.002, -1.000]} & \tiny{[0.871, 0.876]} & \tiny{[0.103, 0.334]} & \tiny{[0.105, 0.304]} &  & \tiny{[0.966, 0.979]} & \tiny{[1.002, 1.642]} & \tiny{[2.172, 5.259]} & \tiny{[1.000, 1.000]} & \tiny{[0.216, 0.269]} \\
 & 1024 & 0.883 & -87.362 & -1.001 & \textbf{0.899} & 0.030 & -0.004 & -- & \textbf{0.841} & 2.170 & 148.854 & 1.000 & \textbf{0.649} \\
 &  & \tiny{[0.878, 0.887]} & \tiny{[-352.627, -32.800]} & \tiny{[-1.002, -1.000]} & \tiny{[0.897, 0.900]} & \tiny{[0.004, 0.079]} & \tiny{[-0.017, 0.013]} &  & \tiny{[0.748, 0.901]} & \tiny{[1.350, 2.468]} & \tiny{[87.984, 325.610]} & \tiny{[1.000, 1.000]} & \tiny{[0.565, 0.713]} \\
\cmidrule(l){1-14}
\multirow{8}{*}{Darcy 2D} & 16 & \textbf{0.666} & \textbf{0.664} & -1.011 & \textbf{0.664} & \textbf{0.999} & 0.997 & -- & 0.991 & \textbf{0.054} & 0.081 & 1.000 & 0.134 \\
 &  & \tiny{[0.481, 0.770]} & \tiny{[0.480, 0.769]} & \tiny{[-1.115, -0.947]} & \tiny{[0.469, 0.769]} & \tiny{[0.998, 0.999]} & \tiny{[0.995, 0.998]} &  & \tiny{[0.984, 0.997]} & \tiny{[0.043, 0.071]} & \tiny{[0.061, 0.099]} & \tiny{[1.000, 1.000]} & \tiny{[0.093, 0.189]} \\
 & 64 & \textbf{0.728} & \textbf{0.730} & -1.125 & \textbf{0.745} & 0.965 & 0.955 & -- & \textbf{0.981} & 0.286 & 0.315 & 1.000 & \textbf{0.208} \\
 &  & \tiny{[0.543, 0.785]} & \tiny{[0.546, 0.784]} & \tiny{[-1.270, -0.987]} & \tiny{[0.564, 0.798]} & \tiny{[0.945, 0.977]} & \tiny{[0.926, 0.967]} &  & \tiny{[0.964, 0.990]} & \tiny{[0.235, 0.358]} & \tiny{[0.266, 0.402]} & \tiny{[1.000, 1.000]} & \tiny{[0.153, 0.287]} \\
 & 256 & 0.782 & 0.682 & -1.110 & \textbf{0.816} & 0.619 & 0.480 & -- & \textbf{0.968} & 0.776 & 1.592 & 1.000 & \textbf{0.268} \\
 &  & \tiny{[0.683, 0.808]} & \tiny{[0.564, 0.707]} & \tiny{[-1.311, -0.987]} & \tiny{[0.755, 0.848]} & \tiny{[0.461, 0.776]} & \tiny{[0.395, 0.583]} &  & \tiny{[0.955, 0.976]} & \tiny{[0.646, 0.900]} & \tiny{[1.367, 1.849]} & \tiny{[1.000, 1.000]} & \tiny{[0.238, 0.393]} \\
 & 1024 & \textbf{0.831} & -99.623 & -1.146 & \textbf{0.846} & 0.301 & 0.018 & -- & \textbf{0.905} & 1.024 & 74.211 & 1.000 & \textbf{0.507} \\
 &  & \tiny{[0.773, 0.861]} & \tiny{[-400.029, -23.036]} & \tiny{[-1.295, -0.993]} & \tiny{[0.793, 0.880]} & \tiny{[0.098, 0.593]} & \tiny{[-0.005, 0.036]} &  & \tiny{[0.865, 0.932]} & \tiny{[0.844, 1.183]} & \tiny{[41.893, 126.786]} & \tiny{[1.000, 1.000]} & \tiny{[0.439, 0.659]} \\
\cmidrule(l){1-14}
\multirow{8}{*}{Reac.--Diff. 2D} & 16 & \textbf{0.296} & \textbf{0.295} & -0.535 & \textbf{0.298} & \textbf{0.997} & 0.996 & -- & \textbf{0.999} & 0.075 & 0.092 & 1.000 & \textbf{0.037} \\
 &  & \tiny{[0.236, 0.485]} & \tiny{[0.232, 0.485]} & \tiny{[-0.769, -0.416]} & \tiny{[0.240, 0.488]} & \tiny{[0.993, 0.998]} & \tiny{[0.991, 0.997]} &  & \tiny{[0.998, 1.000]} & \tiny{[0.063, 0.108]} & \tiny{[0.072, 0.123]} & \tiny{[1.000, 1.000]} & \tiny{[0.009, 0.066]} \\
 & 64 & \textbf{-0.092} & \textbf{-0.122} & -0.953 & \textbf{-0.085} & \textbf{0.922} & 0.863 & -- & \textbf{0.972} & 0.384 & 0.595 & 1.000 & \textbf{0.255} \\
 &  & \tiny{[-0.268, 0.232]} & \tiny{[-0.309, 0.212]} & \tiny{[-1.208, -0.311]} & \tiny{[-0.268, 0.250]} & \tiny{[0.895, 0.942]} & \tiny{[0.820, 0.887]} &  & \tiny{[0.926, 0.986]} & \tiny{[0.328, 0.443]} & \tiny{[0.523, 0.713]} & \tiny{[1.000, 1.000]} & \tiny{[0.176, 0.407]} \\
 & 256 & \textbf{0.155} & -0.504 & -1.076 & \textbf{0.166} & 0.639 & 0.431 & -- & \textbf{0.941} & 0.796 & 2.935 & 1.000 & \textbf{0.366} \\
 &  & \tiny{[-0.058, 0.375]} & \tiny{[-1.008, -0.138]} & \tiny{[-1.438, -0.652]} & \tiny{[-0.028, 0.450]} & \tiny{[0.565, 0.726]} & \tiny{[0.379, 0.479]} &  & \tiny{[0.913, 0.958]} & \tiny{[0.722, 0.848]} & \tiny{[2.379, 3.508]} & \tiny{[1.000, 1.000]} & \tiny{[0.322, 0.452]} \\
 & 1024 & 0.517 & -659.842 & -0.910 & \textbf{0.560} & 0.657 & 0.022 & -- & \textbf{0.892} & 0.772 & 114.896 & 1.000 & \textbf{0.548} \\
 &  & \tiny{[0.336, 0.631]} & \tiny{[-2504.930, -179.158]} & \tiny{[-1.312, -0.587]} & \tiny{[0.357, 0.672]} & \tiny{[0.513, 0.722]} & \tiny{[-0.003, 0.044]} &  & \tiny{[0.817, 0.922]} & \tiny{[0.711, 0.938]} & \tiny{[63.903, 207.205]} & \tiny{[1.000, 1.000]} & \tiny{[0.437, 0.692]} \\
\cmidrule(l){1-14}
\multirow{8}{*}{Shallow Water 2D} & 16 & \textbf{0.264} & \textbf{0.264} & -1.590 & \textbf{0.269} & 0.993 & 0.992 & -- & \textbf{0.998} & 0.118 & 0.118 & 1.000 & \textbf{0.056} \\
 &  & \tiny{[0.229, 0.318]} & \tiny{[0.226, 0.317]} & \tiny{[-1.944, -1.167]} & \tiny{[0.230, 0.320]} & \tiny{[0.991, 0.995]} & \tiny{[0.990, 0.994]} &  & \tiny{[0.997, 0.999]} & \tiny{[0.103, 0.132]} & \tiny{[0.109, 0.135]} & \tiny{[1.000, 1.000]} & \tiny{[0.042, 0.078]} \\
 & 64 & 0.518 & 0.525 & -0.838 & \textbf{0.551} & 0.912 & 0.917 & -- & \textbf{0.992} & 0.442 & 0.429 & 1.000 & \textbf{0.125} \\
 &  & \tiny{[0.474, 0.668]} & \tiny{[0.479, 0.670]} & \tiny{[-0.938, -0.764]} & \tiny{[0.508, 0.694]} & \tiny{[0.892, 0.934]} & \tiny{[0.904, 0.932]} &  & \tiny{[0.990, 0.994]} & \tiny{[0.383, 0.480]} & \tiny{[0.396, 0.463]} & \tiny{[1.000, 1.000]} & \tiny{[0.107, 0.157]} \\
 & 256 & 0.756 & 0.645 & -0.934 & \textbf{0.811} & 0.728 & 0.539 & -- & \textbf{0.978} & 0.694 & 1.394 & 1.000 & \textbf{0.219} \\
 &  & \tiny{[0.717, 0.815]} & \tiny{[0.591, 0.699]} & \tiny{[-1.012, -0.870]} & \tiny{[0.785, 0.833]} & \tiny{[0.660, 0.869]} & \tiny{[0.493, 0.572]} &  & \tiny{[0.971, 0.984]} & \tiny{[0.505, 0.763]} & \tiny{[1.301, 1.509]} & \tiny{[1.000, 1.000]} & \tiny{[0.190, 0.266]} \\
 & 1024 & 0.871 & -66.535 & -0.949 & \textbf{0.887} & 0.909 & 0.010 & -- & \textbf{0.964} & 0.428 & 34.301 & 1.000 & \textbf{0.333} \\
 &  & \tiny{[0.855, 0.878]} & \tiny{[-298.438, -26.911]} & \tiny{[-0.998, -0.903]} & \tiny{[0.878, 0.892]} & \tiny{[0.875, 0.935]} & \tiny{[-0.006, 0.039]} &  & \tiny{[0.943, 0.974]} & \tiny{[0.367, 0.493]} & \tiny{[23.284, 82.147]} & \tiny{[1.000, 1.000]} & \tiny{[0.267, 0.389]} \\
\cmidrule(l){1-14}
\multirow{2}{*}{Heat 3D} & 512 & 0.858 & -60.883 & -1.000 & \textbf{0.884} & 0.107 & 0.004 & -- & \textbf{0.884} & 1.360 & 55.623 & 1.000 & \textbf{0.516} \\
 &  & \tiny{[0.833, 0.866]} & \tiny{[-280.804, -13.804]} & \tiny{[-1.002, -0.996]} & \tiny{[0.872, 0.887]} & \tiny{[0.037, 0.309]} & \tiny{[-0.024, 0.030]} &  & \tiny{[0.855, 0.905]} & \tiny{[0.996, 1.670]} & \tiny{[35.638, 111.006]} & \tiny{[1.000, 1.000]} & \tiny{[0.465, 0.604]} \\
\cmidrule(l){1-14}
\multirow{2}{*}{MGN Cylinder} & 64 & \textbf{0.799} & \textbf{0.789} & -1.066 & \textbf{0.794} & \textbf{0.957} & 0.900 & -- & 0.928 & \textbf{0.289} & 0.465 & 1.000 & 0.402 \\
 &  & \tiny{[0.735, 0.830]} & \tiny{[0.732, 0.812]} & \tiny{[-1.196, -0.948]} & \tiny{[0.705, 0.824]} & \tiny{[0.909, 0.976]} & \tiny{[0.818, 0.943]} &  & \tiny{[0.774, 0.977]} & \tiny{[0.218, 0.406]} & \tiny{[0.366, 0.619]} & \tiny{[1.000, 1.000]} & \tiny{[0.213, 0.645]} \\
\cmidrule(l){1-14}
\multirow{2}{*}{MGN Airfoil} & 64 & \textbf{0.757} & \textbf{0.749} & -1.210 & 0.694 & \textbf{0.911} & 0.827 & -- & 0.630 & \textbf{0.392} & 0.634 & 1.000 & 0.805 \\
 &  & \tiny{[0.564, 0.792]} & \tiny{[0.550, 0.785]} & \tiny{[-1.624, -0.815]} & \tiny{[0.517, 0.775]} & \tiny{[0.848, 0.945]} & \tiny{[0.751, 0.894]} &  & \tiny{[0.269, 0.794]} & \tiny{[0.314, 0.506]} & \tiny{[0.516, 0.762]} & \tiny{[1.000, 1.000]} & \tiny{[0.639, 1.008]} \\
\bottomrule
\end{tabular}

%% file: tables/ssim.tex
\begin{tabular}{llccc}
\toprule
PDE & players & OperatorSHAP & Reference ($3\!\cdot\!10^5$) & RegressionMSR ($\max(2n, 1024)$) \\
\midrule
\multirow{6}{*}{Heat 2D} & 64--256 & 0.360 & 0.352 & 0.255 \\
 &  & \tiny{[0.318, 0.401]} & \tiny{[0.334, 0.413]} & \tiny{[0.192, 0.317]} \\
 & 256--1024 & 0.678 & 0.606 & 0.138 \\
 &  & \tiny{[0.656, 0.701]} & \tiny{[0.575, 0.624]} & \tiny{[0.104, 0.301]} \\
 & 64--1024 & 0.140 & 0.148 & 0.077 \\
 &  & \tiny{[0.115, 0.168]} & \tiny{[0.123, 0.183]} & \tiny{[0.043, 0.117]} \\
\midrule
\multirow{6}{*}{Shallow Water 2D} & 64--256 & 0.271 & 0.292 & 0.364 \\
 &  & \tiny{[0.228, 0.357]} & \tiny{[0.245, 0.346]} & \tiny{[0.301, 0.402]} \\
 & 256--1024 & 0.834 & 0.783 & 0.663 \\
 &  & \tiny{[0.824, 0.854]} & \tiny{[0.761, 0.798]} & \tiny{[0.632, 0.706]} \\
 & 64--1024 & 0.192 & 0.174 & 0.271 \\
 &  & \tiny{[0.147, 0.274]} & \tiny{[0.132, 0.245]} & \tiny{[0.230, 0.340]} \\
\bottomrule
\end{tabular}

%% file: tables/backbone_error.tex
\begin{tabular}{lccccccccc}
\toprule
PDE $\backslash\ n_{\mathrm{1D}}$ & 4 & 8 & 16 & 32 & 64 & 128 & 256 & 512 & 1024 \\
\midrule
Heat 1D & -- & 0.049 & 0.017 & 0.005 & 0.005 & 0.005 & 0.005 & 0.005 & 0.005 \\
Burgers 1D & -- & 0.115 & 0.053 & 0.047 & 0.046 & 0.051 & 0.051 & 0.051 & 0.050 \\
Burgers 1D (ReLU) & -- & 0.088 & 0.034 & 0.026 & 0.022 & 0.022 & 0.021 & 0.021 & 0.021 \\
Allen--Cahn 1D & -- & 0.245 & 0.008 & 0.004 & 0.003 & 0.003 & 0.007 & 0.009 & 0.010 \\
Diff.--Sorp. 1D & -- & 0.056 & 0.001 & 0.002 & 0.002 & 0.003 & 0.004 & 0.004 & 0.005 \\
Advection 1D & -- & 0.129 & 0.025 & 0.029 & 0.035 & 0.037 & 0.037 & 0.037 & 0.037 \\
Heat 2D & 0.364 & 0.049 & 0.002 & 0.001 & -- & -- & -- & -- & -- \\
Darcy 2D & 0.373 & 0.101 & 0.042 & 0.037 & -- & -- & -- & -- & -- \\
Reac.--Diff. 2D & 0.364 & 0.305 & 0.220 & 0.212 & -- & -- & -- & -- & -- \\
Shallow Water 2D & 0.085 & 0.035 & 0.006 & 0.005 & -- & -- & -- & -- & -- \\
Heat 3D & -- & 0.032 & -- & -- & -- & -- & -- & -- & -- \\
\bottomrule
\end{tabular}